\documentclass[nohyperref]{article}

\usepackage{microtype}
\usepackage{graphicx}
\usepackage{subfigure}
\usepackage{booktabs} % for professional tables
\usepackage{url}
\usepackage{float}

\usepackage{amsmath,amsfonts,bm}

\def\eqref#1{equation~\ref{#1}}
\def\1{\bm{1}}

\DeclareMathAlphabet{\mathsfit}{\encodingdefault}{\sfdefault}{m}{sl}
\SetMathAlphabet{\mathsfit}{bold}{\encodingdefault}{\sfdefault}{bx}{n}

\newcommand{\R}{\mathbb{R}}

\DeclareMathOperator*{\argmax}{arg\,max}

\usepackage{hyperref}

\usepackage[accepted]{icml2022}

\usepackage{amsmath}
\usepackage{amssymb}
\usepackage{mathtools}
\usepackage{amsthm}
\usepackage{algorithm}
\usepackage[noend]{algpseudocode}

\usepackage[capitalize,noabbrev]{cleveref}

\theoremstyle{plain}
\newtheorem{theorem}{Theorem}[section]

\newtheorem{lemma}[theorem]{Lemma}
\newtheorem{corollary}[theorem]{Corollary}
\theoremstyle{definition}

\theoremstyle{remark}
\newtheorem{remark}[theorem]{Remark}

\usepackage[textsize=tiny]{todonotes}

\DeclarePairedDelimiter{\norm}{\lVert}{\rVert}

\icmltitlerunning{Intriguing Properties of Input-Dependent Randomized Smoothing}

\begin{document}

\twocolumn[
\icmltitle{Intriguing Properties of Input-Dependent Randomized Smoothing}



\icmlsetsymbol{equal}{*}

\begin{icmlauthorlist}
\icmlauthor{Peter Súkeník}{ist}
\icmlauthor{Aleksei Kuvshinov}{tum}
\icmlauthor{Stephan Günnemann}{tum,mdsi}
\end{icmlauthorlist}

\icmlaffiliation{ist}{Institute of Science and Technology Austria, Klosterneuburg, Austria}
\icmlaffiliation{tum}{Technical University of Munich, School of Computation, Information and Technology, Munich, Germany}
\icmlaffiliation{mdsi}{Munich Data Science Institute, Munich, Germany}

\icmlcorrespondingauthor{Peter Súkeník}{peter.sukenik@ista.ac.at}

\icmlkeywords{Machine Learning, ICML}

\vskip 0.3in
]



\printAffiliationsAndNotice{}  % leave blank if no need to mention equal contribution

\begin{abstract}
Randomized smoothing is currently considered the state-of-the-art method to obtain certifiably robust classifiers. Despite its remarkable performance, the method is associated with various serious problems such as ``certified accuracy waterfalls'', certification vs.\ accuracy trade-off, or even fairness issues. Input-dependent smoothing approaches have been proposed with intention of overcoming these flaws. However, we demonstrate that these methods lack formal guarantees and so the resulting certificates are not justified. We show that in general, the input-dependent smoothing suffers from the curse of dimensionality, forcing the variance function to have low semi-elasticity. On the other hand, we provide a theoretical and practical framework that enables the usage of input-dependent smoothing even in the presence of the curse of dimensionality, under strict restrictions. We present one concrete design of the smoothing variance function and test it on CIFAR10 and MNIST. Our design mitigates some of the problems of classical smoothing and is formally underlined, yet further improvement of the design is still necessary.

\iffalse
Currently, randomized smoothing is the state-of-the-art method to obtain certifiably robust classifiers.
Despite its remarkable success, the method suffers from various serious problems such as the certified accuracy waterfalls, certification-accuracy trade-off, and even fairness issues.
Recently, input-dependent smoothing approaches were proposed to overcome these flaws.
However, we demonstrate that these methods lack formal guarantees and therefore the resulting certificates are not justified.
In this work, we develop the theoretical and practical framework that enables the usage of input-dependent variance functions.
We prove that the input-dependent smoothing suffers from the curse of dimensionality restricting the choice of the smoothing variance to the functions with a low elasticity.
We present one design of the variance function allowing to obtain theoretically proven robustness guarantees and test it on CIFAR10 and MNIST.
We show that it solves some of the problems of the classical smoothing and discuss possible improvements.
\fi
\end{abstract}

\section{Introduction} \label{intro}
Deep neural networks are one of the dominating recently used machine learning methods. They achieve state-of-the-art performance in a variety of applications like computer vision, natural language processing, and many others. The key property that makes neural networks so powerful is their expressivity \citep{guhring2020expressivity}. However, as a price, they possess a weakness - a vulnerability against \textit{adversarial attacks} \citep{szegedy2013intriguing, 10.1007/978-3-642-40994-3_25}. The adversarial attack on a sample $x_0$ is a point $x_1$ such that the distance $d(x_0, x_1)$ is small, yet the predictions of model $f$ on $x_0$ and $x_1$ differ. Such examples are often easy to construct, for example by optimizing for a change in prediction $f(x)$ \citep{10.1007/978-3-642-40994-3_25}. Even worse, these attacks are present even if the model's prediction on $x$ is unequivocal. 

This property is highly undesirable because in several sensitive applications, misclassifying a sample just because it does not follow the natural distribution might lead to serious and harmful consequences. A well-known example is a sticker placed on a traffic sign, which could confuse the self-driving car and cause an accident \citep{eykholt2018robust}. As a result, the robustness of classifiers against adversarial examples has begun to be a strongly discussed topic. Though many methods claim to provide robust classifiers, just some of them are \textit{certifiably} robust, i.e. the robustness is mathematically guaranteed. The certifiability turns out to be essential since more sophisticated attacks can break empirical defenses \citep{carlini2017adversarial}.

Currently, the dominant method to achieve the certifiable robustness is \textit{randomized smoothing} (RS). This clever idea to get rid of adversarial examples using randomization of input was introduced by \citet{lecuyer2019certified} and \citet{li2018certified} and fully formalized and improved by \citet{cohen2019certified}. Let $f$ be a classifier assigning inputs $x \in \mathbb{R}^N$ to one of the classes $C \in \mathcal{C}$. Given a random deviation $\epsilon \sim \mathcal{N}(0, \sigma^2 I)$, the \textit{smoothed classifier} $g$, made of $f$, is defined as: $g(x) = \argmax_C\mathbb{P}(f(x+\epsilon)=C),$ for $C \in \mathcal{C}$. In other words, the smoothed classifier classifies a class that has the highest probability under the sampling of $f(x+\epsilon)$. Consequently, an adversarial attack $x^\prime$ on $f$ is less dangerous for $g$, because $g$ does not look directly at $x^\prime$, but rather at its whole neighborhood, in a weighted manner. This way we can get rid of local artifacts that $f$ possesses -- thus the name ``smoothing''. It turns out, that $g$ enjoys strong robustness properties against attacks bounded by a specifically computed $l_2$-norm threshold, especially if $f$ is trained under a Gaussian noise augmentation \citep{cohen2019certified}. 

Unfortunately, since the introduction of the RS, several serious problems were reported to be connected to the technique. \citet{cohen2019certified} mention two of them. First is the usage of lower confidence bounds to estimate the leading class's probability. With a high probability, this leads to smaller reported certified radiuses in comparison with the true ones. Moreover, it yields a theoretical threshold, which upper-bounds the maximal possible certified radius and causes the ``certified accuracy waterfalls'', which significantly decreases the certified accuracy. This problem is particularly pronounced for small levels of the used smoothing variance $\sigma^2$, which motivates to use larger variance. Second, RS possesses a robustness vs.\ accuracy trade-off problem. The bigger $\sigma$ we use as the smoothing variance, the smaller \textit{clean accuracy} will the smoothed classifier have. This motivates to use rather smaller levels of $\sigma$. Third, as pointed out by \citet{mohapatra2020rethinking}, RS smooths the decision boundary of $f$ in such a way that bounded or convex regions begin to \textit{shrink} as $\sigma$ increases, while the unbounded and anti-convex regions expand. This, as the authors empirically demonstrate, creates a imbalance in class-wise accuracies (accuracies computed per each class separately) of $g$ and causes serious fairness issues. Therefore the smaller values of $\sigma$ are again more preferable. See Appendix~\ref{appA: the motivation} for a detailed discussion. 

Clearly, the usage of a global, constant $\sigma$ is suboptimal. For the samples close to the decision boundary, we want to use small $\sigma$, so that $f$ and $g$ have similar decision boundaries and the expressivity of $f$ is not lost (where not necessary). On the other hand, far from the decision boundary of $f$, where the probability of the dominant class is close to 1, we need bigger $\sigma$ to avoid the severe under-certification (see Appendix~\ref{appA: the motivation}). All together, using a non-constant $\sigma(x)$ rather than constant $\sigma$, a suitable smoothing variance could be used to achieve optimal robustness. 

To support this reasoning, we present a toy example. We train a network on a 2D dataset of a circular shape with the classes being two complementary sectors, one of which is of a very small angle. In Figure~\ref{fig: toy experiment main text} we show the difference between constant and input-dependent $\sigma$. Using the non-constant $\sigma(x)$ defined in Equation~\ref{eq: the sigma fcn}, we obtain an improvement both in terms of the certified radiuses as well as clean accuracy. For more details see Appendix~\ref{appA: the motivation}. Even though there are some works introducing this concept (see Appendix~\ref{appB: concurrent work}), most of them lack mathematical reasoning about the correctness of their method or fairness of their comparison, which, as we show, turns out to be critical. 

\begin{figure*}[t!]
    \centering
    \begin{minipage}[b]{0.32\linewidth}
        \includegraphics[width=\textwidth]{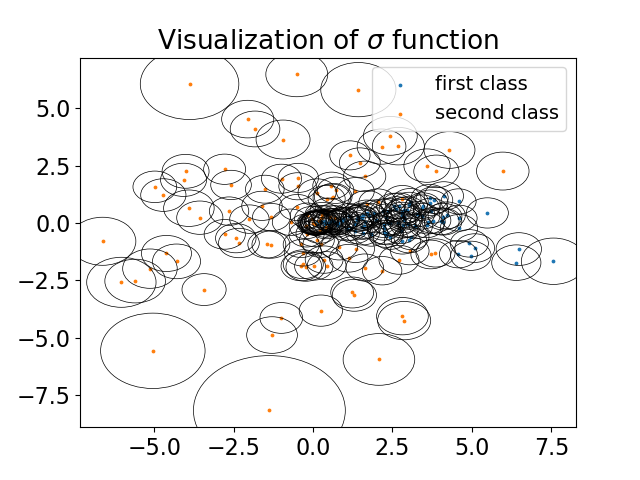}
    \end{minipage}
    \begin{minipage}[b]{0.32\linewidth}
        \includegraphics[width=\textwidth]{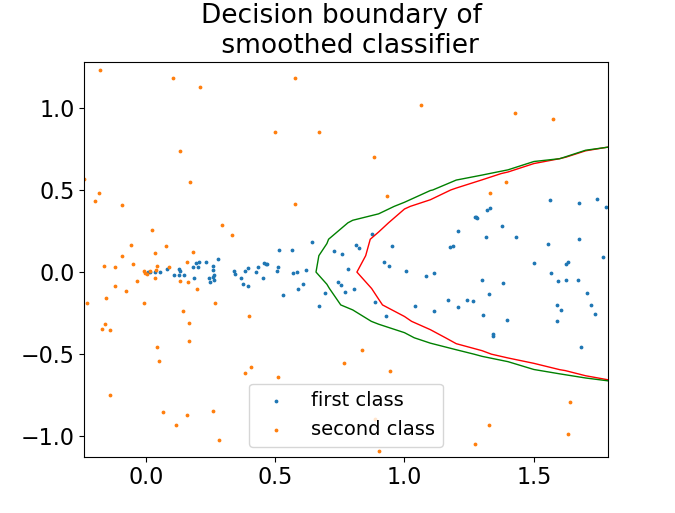}
    \end{minipage}
    \begin{minipage}[b]{0.32\linewidth}
        \includegraphics[width=\textwidth]{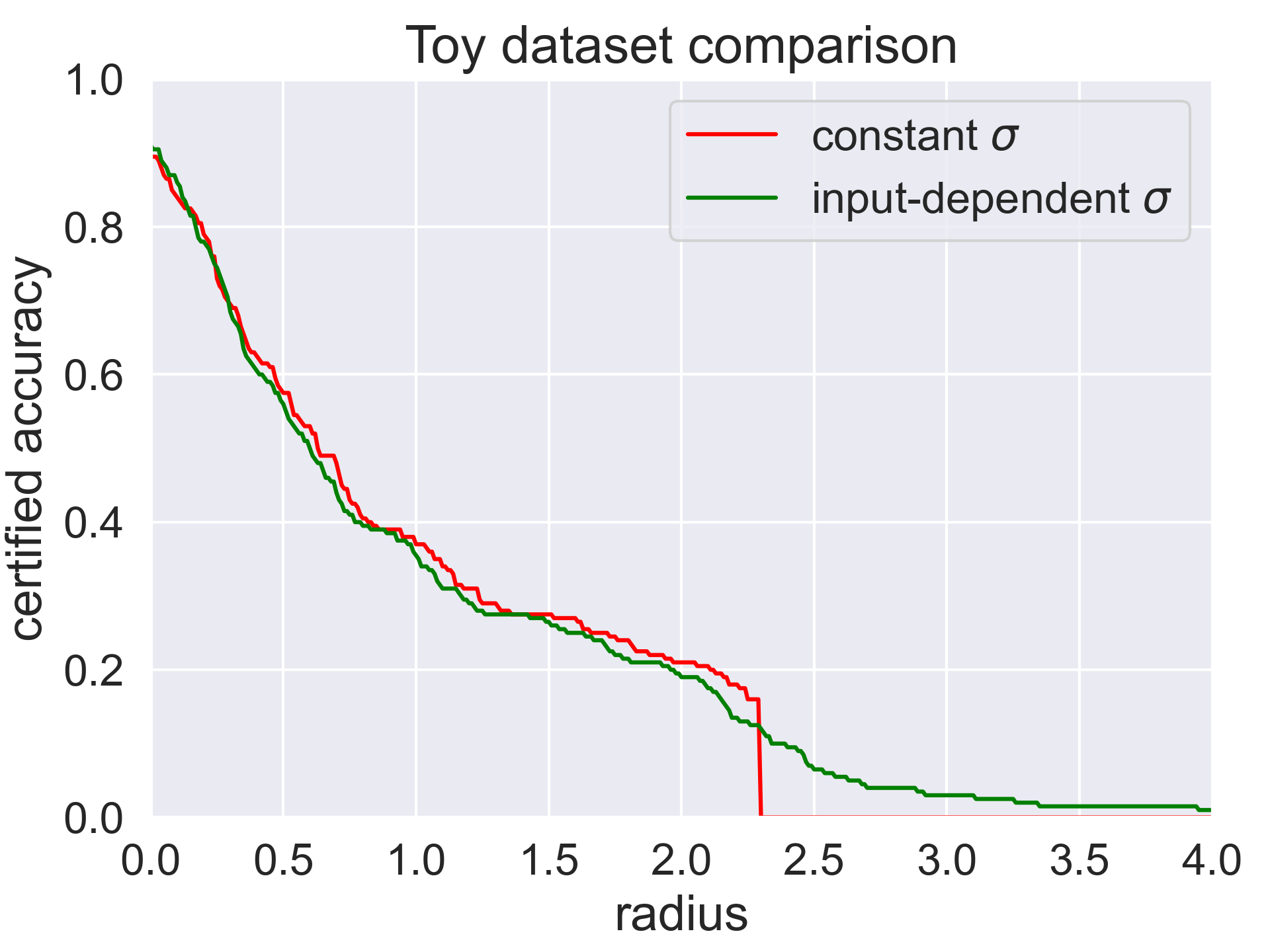}
    \end{minipage}
    \caption{Motivating toy experiment. We use constant $\sigma=0.6$ and input-dependent $\sigma(x)$ equal in average to the constant $\sigma$. \textbf{Left:} Dataset and the variance function depicted as circles with the radius equal to $\sigma(x)$ and centers at the data points. \textbf{Middle:} Zoomed in part of the dataset and decision boundaries of the smoothed classifiers with constant $\sigma$ (red) and input-dependent $\sigma(x)$ (green). Note that we recover a part of the misclassified data points by using a more appropriate smoothing strength close to the decision boundary. \textbf{Right:} Certified accuracy plot. The waterfall effect vanishes since the points far from the decision boundary are certified with a correspondingly large $\sigma(x)$.}
    \label{fig: toy experiment main text}
\end{figure*}

The main contributions of this work are fourfold. First, we generalize the methodology of \citet{cohen2019certified} for the case of the input-dependent RS (IDRS), obtaining useful and important insights about how to use the Neyman-Pearson lemma in this general case. Second and most importantly, we show that the IDRS suffers from the curse of dimensionality in the sense that the semi-elasticity coefficient $r$ of $\sigma(x)$ (a positive number such that \mbox{$|\log(\sigma(x_0))-\log(\sigma(x_1))|\le r\norm{x_0-x_1} \hspace{1mm} \forall x_0, x_1 \in \mathbb{R}^N$}) in a high-dimensional setting is restricted to be very small. This means, that even if we wanted to vary $\sigma(x)$ significantly with varying $x$, we cannot. The maximal reasonable speed of change of $\sigma(x)$ turns out to be almost too small to handle, especially in high dimensions. Third, in contrast, we also study the conditions on $\sigma(x)$ under which it is applicable in high-dimensional regime and prepare a theoretical framework necessary to build an efficient certification algorithm. We are the first to do so for $\sigma(x)$ functions, which are not locally constant (as in \cite{wang2021pretraintofinetune}). Finally, we provide a concrete design of the $\sigma(x)$ function, test it extensively and compare it to the classical RS on the CIFAR10 and MNIST datasets. We discuss to what extent the method treats the issues mentioned above. 

%Third, parallel to this theoretical insight, we point out to the scientific %community the dangerous mistake that could be and was made in the literature %on IDRS. The usage of the formula for certified radius from %\cite{cohen2019certified} in the context of IDRS is %mathematically unjustified and yields falsely large certified radiuses.

\section{IDRS and the Curse of Dimensionality}\label{sec:theory}
Let $\mathcal{C}$ be the set of classes, $f: \R^N \xrightarrow[]{} \mathcal{C}$ a classifier (referred to as the \textit{base} classifier), $\sigma: \mathbb{R}^N \xrightarrow[]{} \mathbb{R}$ a non-negative function and $\mathcal{P}(\mathcal{C})$ a set of distributions over $\mathcal{C}$. Then we call $G_f: \mathbb{R}^N \xrightarrow[]{} \mathcal{P}(\mathcal{C})$ the \textit{smoothed class probability predictor}, if ${G_f(x)}_C=\mathbb{P}(f(x+\epsilon)=C),$ where $\epsilon \sim \mathcal{N}(0, \sigma(x)^2 I)$ and $g_f: \mathbb{R}^N \xrightarrow[]{} \mathcal{C}$ is called \textit{smoothed classifier} if \mbox{$g_f(x)=\argmax_C {G_f(x)}_C$}, for $C \in \mathcal{C}$. We will omit the subscript $f$ in $g_f$ often, since it is usually clear from the context to which base classifier the $g$ corresponds. Furthermore, let $A := g(x)$ refer to the most likely class under the random variable $f(\mathcal{N}(x, \sigma^2 I))$, and $B$ denote the second most likely class. Define $p_A={G_f(x)}_A$ and $p_B={G_f(x)}_B$ as the respective probabilities. It is important to note that in practice, it is impossible to estimate $p_A$ and $p_B$ precisely. Instead, $p_A$ is estimated as a lower confidence bound (LCB) of the relative occurence of class $A$ in the predictions of $f$ given certain number of Monte-Carlo samples $n$ and a confidence level $\alpha$. The estimate is denoted as $\underline{p_A}$. Similarly to \citet{cohen2019certified} we use the exact Clopper-Pearson interval for estimation of the LCB. The same applies for $p_B$. We work with $l_2$-norms denoted as $\norm{x}$. When we speak of certified radius at sample $x_0$, we always mean the biggest $R \ge 0$ for which the underlying theory provides the guarantee $g(x_0)=g(x_1) \hspace{2mm} \forall x_1: \norm{x_0-x_1}\le R$. 

First of all, we summarize the main steps in the derivation of certified radius around $x_0$ using any method that relies on the Neyman-Pearson lemma (e.g. by \citet{cohen2019certified}). 
\begin{enumerate}
\item For a potential adversary $x_1$ specify the \textit{worst-case} classifier $f^*$, such that $\mathbb{P}(f^*(\mathcal{N}(x_0, \sigma^2 I))=A)=p_A$, while $\mathbb{P}(f^*(\mathcal{N}(x_1, \sigma^2 I))=B)$ is maximized.
\item Express %the probability
${G_{f^*}(x_1)}_B=\mathbb{P}(f^*(\mathcal{N}(x_1, \sigma^2 I))=B)$ as a function depending on $x_1$. %Now we know, what is the worst-case class $B$ probability at any adversary sample $x_1$, and use it to derive the certified radius.
\item Determine the conditions on $x_1$ (possibly related to $\norm{x_0-x_1}$) for which this probability is $\le 0.5$. From these conditions, derive the certified radius. 
\end{enumerate}

\iffalse
\cite{cohen2019certified} use Neyman-Pearson lemma to obtain a tight certified radius $R=\sigma/2(\Phi^{-1}(\underline{p_A})-\Phi^{-1}(\overline{p_B}))$. The main idea of Neyman-Pearson lemma is that it allows to specify the \textit{worst-case} classifier $f^*$, such that $\mathbb{P}(f^*(\mathcal{N}(x, \sigma^2 I))=A)=p_A$, while $\mathbb{P}(f^*(\mathcal{N}(x^\prime, \sigma^2 I))=B)$ is maximized, where $x^\prime$ is a potential adversary sample. In this way, we know, what is the worst-case class $B$ probability at any adversary sample $x^\prime$, and use it to derive the certified radius.
\fi

\citet{cohen2019certified} proceeds in this way to obtain a tight certified radius $R=\frac{\sigma}{2}(\Phi^{-1}(\underline{p_A})-\Phi^{-1}(\overline{p_B}))$. Unfortunately, their result is not directly applicable to the input-dependent case. Constant $\sigma$ simplifies the derivation of $f^*$ that turns out to be a linear classifier. This is not the case for non-constant $\sigma(x)$ anymore. Therefore, we generalize the methodology of \citet{cohen2019certified}. We put $p_B=1-p_A$ for simplicity (yet it is not necessary to assume this, see Appendix~\ref{appC: multi-class regime}). Let $x_0$ be the point to certify, $x_1$ the potential adversary point, $\delta=x_1-x_0$ the shift and $\sigma_0=\sigma(x_0)$, $\sigma_1=\sigma(x_1)$ the standard deviations used in $x_0$ and $x_1$, respectively. Furthermore, let $q_i$ be a density and $\mathbb{P}_i$ a probability measure corresponding to $\mathcal{N}(x_i, \sigma_i^2 I)$, $i \in \{0, 1\}$. 

\begin{lemma} \label{np lemma}
Out of all possible classifiers $f$ such that ${G_f(x_0)}_B \le p_B = 1-p_A$, the one, for which ${G_{f}(x_0+\delta)}_B$ is maximized predicts class $B$ in a region determined by the likelihood ratio: \begin{displaymath} B=\left\{x \in \mathbb{R}^N: \frac{q_1(x)}{q_0(x)} \ge \frac{1}{r}\right\}, \end{displaymath} where $r$ is fixed, such that $\mathbb{P}_0(B)=p_B$. Note that we use $B$ to denote both the class and the region of that class.
\end{lemma}

We use this lemma to compute the decision boundary of the worst-case classifier $f^*$. 

\begin{theorem} \label{lrt set}
If $\sigma_0 > \sigma_1$, then set $B$ is an $N$-dimensional ball with the center at $S_>$ and radius \mbox{$R_> := R_>(\sigma_0, \sigma_1, \delta, N, r)$}, defined in \mbox{Appendix~\ref{appF: proofs}}: $$S_>=x_0+\frac{\sigma_0^2}{\sigma_0^2 - \sigma_1^2}\delta.$$ 
If $\sigma_0 < \sigma_1$, then set $B$ is the complement of an $N$-dimensional ball with the center at $S_<$ and radius \mbox{$R_< := R_<(\sigma_0, \sigma_1, \delta, N, r)$}, expressed in Appendix~\ref{appF: proofs}: $$S_<=x_0-\frac{\sigma_0^2}{\sigma_1^2 - \sigma_0^2}\delta.$$
\end{theorem}

\begin{figure}[t!]
    \centering
    \begin{minipage}[b]{0.48\linewidth}
        \includegraphics[width=\textwidth]{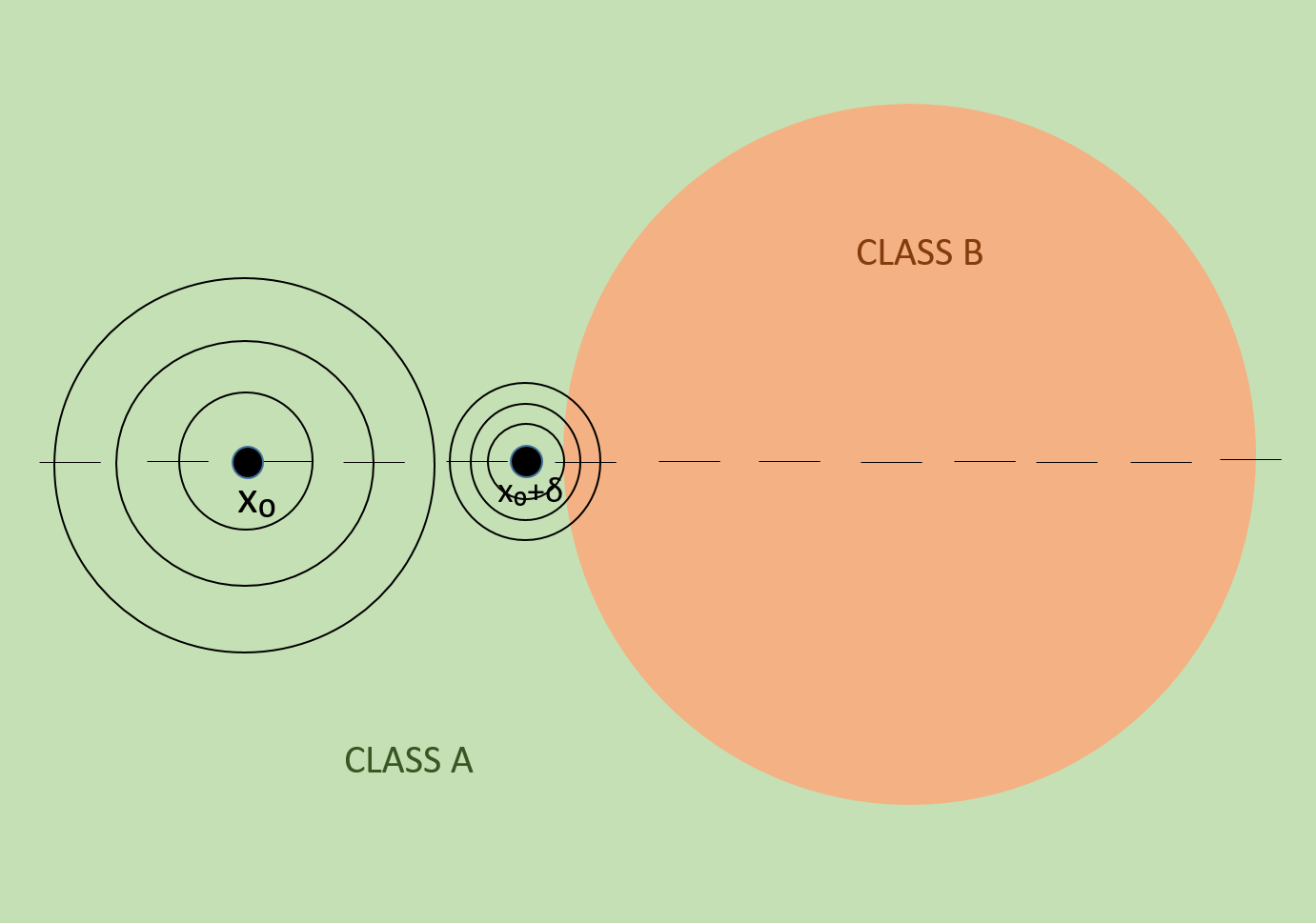}
    \end{minipage}
    \begin{minipage}[b]{0.48\linewidth}
        \includegraphics[width=\textwidth]{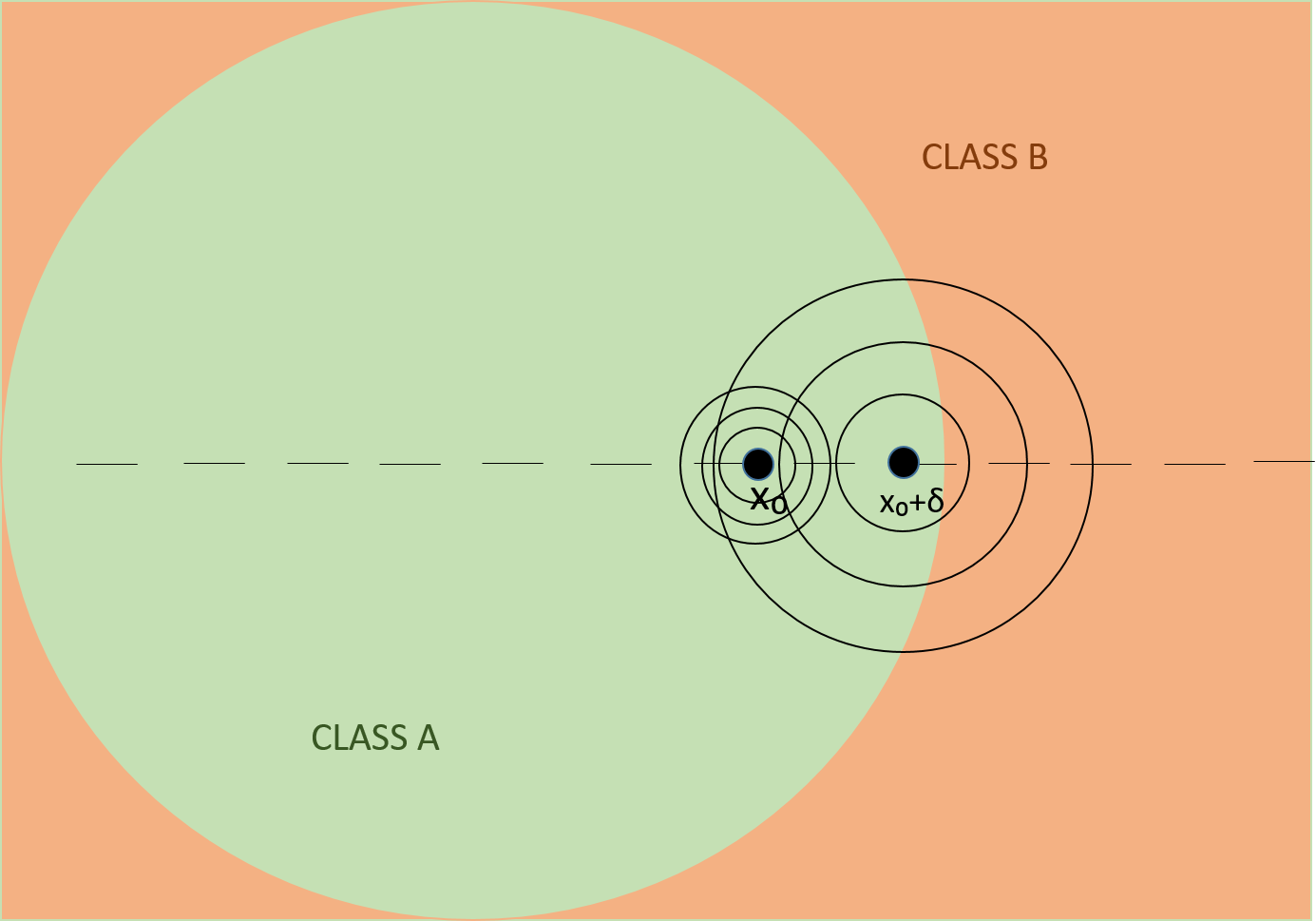}
    \end{minipage}
    \caption{Decision regions of the worst-case classifier $f^*$. \textbf{Left:} $\sigma_0>\sigma_1$ \textbf{Right:} $\sigma_0<\sigma_1$.}
    \label{fig: worst case classifiers main text}
\end{figure} 

As we depict in Figure~\ref{fig: worst case classifiers main text}, both resulting balls are centered on the line connecting $x_0, x_1$. Moreover, the centers of the balls are always further from $x_0$, than $x_1$ is from $x_0$ (even in the case $\sigma_0<\sigma_1$). In both cases, it depends on $p_A$ (since $r$ is fixed such that $\mathbb{P}_0(B)=p_B$) and the ball can, but might not cover $x_0$ and/or $x_1$. Note that if $\sigma_0=\sigma_1$, which can happen even in input-dependent regime, the worst-case classifier is the half-space described by \citet{cohen2019certified}.

To compute the probability of a ball under an isotropic Gaussian probability measure is more challenging than the probability of a half-space. In fact, there is no closed-form expression for it. However, this probability is connected to the non-central chi-squared distribution (NCCHSQ). More precisely, the probability of an $N$-dimensional ball centered at $z$ with radius $r$ under $\mathcal{N}(0, I)$ can be expressed as a cumulative distribution fucntion (cdf) of NCCHSQ with $N$ degrees of freedom, non-centrality parameter $\norm{z}^2$ and argument $r^2$. With this knowledge, we can express $\mathbb{P}_0(B)$ and $\mathbb{P}_1(B)$ in terms of the cdf of NCCHSQ as follows. 

\begin{theorem} \label{thm ncchsq}
$$\mathbb{P}_0(B)=\chi^2_N\left(\frac{\sigma_0^2}{(\sigma_0^2 - \sigma_1^2)^2}\norm{\delta}^2, \frac{R_{<,>}^2}{\sigma_0^2}\right),$$
$$\mathbb{P}_1(B)=\chi^2_N\left(\frac{\sigma_1^2}{(\sigma_0^2 - \sigma_1^2)^2}\norm{\delta}^2, \frac{R_{<,>}^2}{\sigma_1^2}\right),$$
where the sign $<$ or $>$ is chosen according to the inequality between $\sigma_0$ and $\sigma_1$, and $\chi^2_N$ is the cdf of NCCHSQ with $N$ degrees of freedom. 
\end{theorem}

Note that both Theorem~\ref{lrt set} and Theorem~\ref{thm ncchsq} work well also for $\delta=0$ (see the proofs in Appendix~\ref{appF: proofs}). In this case, we encounter a ball centered at $x_0=x_1$ and all the cdf functions become cdf functions of \textit{central} chi-squared.  

We expressed the probabilities of the decision region of the worst-case class $B$  using the cdf of NCCHCSQ. Now, how do we do the certification? We start with the certification just for two points, $x_0$ for which the certified radius is in question and its potential adversary $x_1$. We ask, under which circumstances can $x_1$ be certified from the point of view of $x_0$. To obtain $\mathbb{P}_1(B)$ as a function of $x_1$ as in step 2 of the certified radius derivation scheme from page 2, we first need to fix $\mathbb{P}_0(B)=1-p_A=p_B$. Having $x_0, x_1$, $p_A$ and $\sigma_0>\sigma_1$, we obtain such $R$, that \mbox{$\mathbb{P}_0(B) = \chi^2_N\left( \norm{\delta}^2 \sigma_0^2 / (\sigma_0^2 - \sigma_1^2)^2, R^2\right)=1-p_A=p_B$} simply by putting it into the quantile function. This way we get \mbox{$R^2 = \chi^2_{N,\text{qf}}\left(\norm{\delta}^2 \sigma_0^2 / (\sigma_0^2 - \sigma_1^2)^2, 1-p_A\right)$}. In the next step we substitute it directly into \mbox{$\mathbb{P}_1(B) = \chi^2_N\left(\norm{\delta}^2 \sigma_1^2 / (\sigma_0^2 - \sigma_1^2)^2, R^2 \sigma_0^2 / \sigma_1^2\right)$}. This way, we obtain $\mathbb{P}_1(B)$ and can judge, whether $\mathbb{P}_1(B)<1/2$ or not. Similar computation can be done if $\sigma_0<\sigma_1$. Denote $a:=\norm{\delta}$. We express $\mathbb{P}_1(B)$ more simply as a function of $a$ for $\sigma_0>\sigma_1$ as $$\xi_>(a):=\mathbb{P}_1(B)=$$
$$\chi^2_N\left(\frac{\sigma_1^2}{(\sigma_0^2 - \sigma_1^2)^2}a^2, \frac{\sigma_0^2}{\sigma_1^2}\chi^2_{N,qf}\left(\frac{\sigma_0^2}{(\sigma_0^2 - \sigma_1^2)^2}a^2, 1-p_A\right)\right)$$ and for $\sigma_0<\sigma_1$ as $$\xi_<(a):=\mathbb{P}_1(B)=$$
$$1-\chi^2_N\left(\frac{\sigma_1^2}{(\sigma_1^2 - \sigma_0^2)^2}a^2, \frac{\sigma_0^2}{\sigma_1^2}\chi^2_{N,qf}\left(\frac{\sigma_0^2}{(\sigma_1^2 - \sigma_0^2)^2}a^2, p_A\right)\right)$$ 

With this in mind, if we have $x_0, x_1, p_A, \sigma_0, \sigma_1$, then we can certify $x_1$ w.r.t $x_0$ simply by choosing the correct sign ($<, >$), computing $\xi_<(\norm{x_0-x_1})$ or $\xi_>(\norm{x_0-x_1})$ and comparing it with $0.5$. The sample plots of these $\xi$ functions can be found in Appendix~\ref{appC: more on theory}. 

Now, we are ready to discuss the curse of dimensionality. The problem that arises is that having a high dimension $N$ and $\sigma_0, \sigma_1$ differing a lot from each other, $\xi$ functions are already big at 0, even for considerably small $p_B$. For fixed ratio $\sigma_0 / \sigma_1$ and probability $p_B$, $\xi(0)$ increases with growing dimension $N$ and soon becomes bigger than 0.5. This, together with monotonicity of the $\xi$ function yields that no $x_1$ can be certified w.r.t. $x_0$, if $\sigma_0, \sigma_1$ are used. The more dissimilar the $\sigma_0$ and $\sigma_1$ are, the smaller the dimension $N$ needs to be for this situation to occur. If we want to certify $x_1$ in a reasonable distance from $x_0$, we need to use similar $\sigma_0, \sigma_1$. This restricts the variability of the $\sigma(x)$ function. We formalize the curse of dimensionality in the following theorems. We discuss more why the curse of dimensionality is present in Appendix~\ref{appB: the curse part}.

\begin{theorem}[curse of dimensionality] \label{main theorem}
Let $x_0$, $x_1$, $p_A$, $\sigma_0$, $\sigma_1$ and $N$ be as usual. The following two implications hold: 
\begin{enumerate}
    \item If $\sigma_0>\sigma_1$ and $$\log\left(\frac{\sigma_1^2}{\sigma_0^2}\right)+1-\frac{\sigma_1^2}{\sigma_0^2} < \frac{2\log(1-p_A)}{N},$$
    then $x_1$ is not certified w.r.t. $x_0$. 
    \item If $\sigma_0<\sigma_1$ and 
    $$\log\left(\frac{\sigma_1^2}{\sigma_0^2}\frac{N-1}{N}\right)+1-\frac{\sigma_1^2}{\sigma_0^2}\frac{N-1}{N} < \frac{2\log(1-p_A)}{N},$$
    then $x_1$ is not certified w.r.t. $x_0$. 
\end{enumerate}
\end{theorem}

\begin{corollary}[one-sided simpler bound] \label{main thm corollary}
Let $x_0$, $x_1$, $p_A$, $\sigma_0$, $\sigma_1$ and $N$ be as usual and assume now $\sigma_0>\sigma_1$. Then, if
$$\frac{\sigma_1}{\sigma_0}<\sqrt{1-2\sqrt{\frac{-\log(1-p_A)}{N}}},$$ then $x_1$ is not certified w.r.t $x_0$. 
\end{corollary}

Note that both Theorem~\ref{main theorem} and Corrolary~\ref{main thm corollary} can be adjusted to the case where we have a separate estimate $\overline{p_B}$ and do not put $\overline{p_B}=1-\underline{p_A}$ (see Appendix~\ref{appC: multi-class regime}). We emphasize, that the bounds obtained in Theorem~\ref{main theorem} are very tight. In other words, if the ratio $\sigma_1/\sigma_0$ is just slightly bigger than the minimal possible threshold from Theorem~\ref{main theorem}, $\xi_>(0)$ becomes smaller than $0.5$ and similarly for $\xi_<(0)$. This is because the only two estimates used in the proof of Theorem~\ref{main theorem} are the estimates on the median, which are very tight and constant with respect to $N$, and the multiplicative Chernoff bound, which is generally considered to be tight too and improves for larger $N$. The tightness is depicted in Figure~\ref{fig:tightness}, where we plot the minimal possible threshold $\sigma_1/\sigma_0$ given by Theorem~\ref{main theorem} and minimal threshold for which $\xi_>(0)<0.5$ as a function of $N$. 

\begin{figure}[t!]
    \centering
    \begin{minipage}[b]{0.48\linewidth}
        \includegraphics[width=\textwidth]{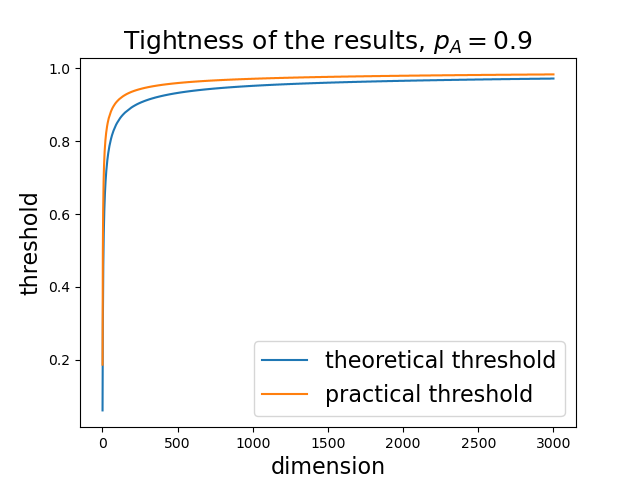}
    \end{minipage}
    \begin{minipage}[b]{0.48\linewidth}
        \includegraphics[width=\textwidth]{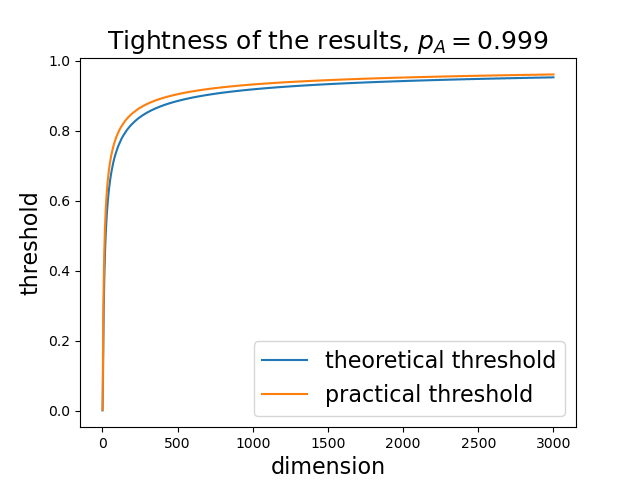}
    \end{minipage}
    \caption{Plots depicting tightness of results of Theorem~\ref{main theorem}. In both figures, the biggest possible threshold of $\sigma_1/\sigma_0$ for which the condition in Theorem~\ref{main theorem} is satisfied (theoretical threshold) and the numerically computed threshold for which $\xi_>(0)$ exceeds the threshold 0.5 (practical threshold) are depicted. \textbf{Left:} Plot for $p_A=0.9$, \textbf{Right:} Plot for $p_A=0.999$.}
    \label{fig:tightness}
\end{figure} 

To get a better feeling about the concrete numbers, we provide the theoretical threshold values from Theorem~\ref{main theorem} in Table~\ref{tab: theoretical thresholds}. If $\sigma_1/\sigma_0$ is smaller than the threshold, we are not able to certify any pair of $x_0, x_1$ using $\sigma_0, \sigma_1$. 

\begin{table}[t!]
\caption{Theoretical lower-thresholds for $\sigma_1/\sigma_0$ for different data dimensions and class $A$ probabilities. The ImageNet spatial size is assumed to be 3x256x256, while CIFAR10 spatial size is 32x32x3 and MNIST spatial size is 28x28x1.}
\vspace{2mm}
\centering
\begin{tabular}{||c||c|c|c|c||} 
\hline\hline
$p_A$ & $0.9$ & $0.99$ & $0.999$ & $0.99993$ \\ 
\hline\hline
MNIST & 0.946 & 0.924 & 0.908 & 0.892 \\ 
\hline
CIFAR10 & 0.973 & 0.961 & 0.953 & 0.945 \\
\hline
ImageNet & 0.997 & 0.995 & 0.994 & 0.993 \\
\hline\hline
\end{tabular}
\label{tab: theoretical thresholds}
\end{table}

Results from Table~\ref{tab: theoretical thresholds} are very restrictive. Assume we have a CIFAR10 sample with $\underline{p_A}=0.999$. For such a probability, \textit{constant} $\sigma=0.5$ is more than sufficient to guarantee the certified radius of more than 1. However, in the non-constant regime, to certify $R \ge 1$, we first need to guarantee that \textit{no} sample within the distance of 1 from $x_0$ uses $\sigma_1<0.953\sigma_0$. To even strengthen this statement, note that one needs to guarantee $\sigma_1$ to be even much closer to $\sigma_0$ in practice. Why? The results of Theorem~\ref{main theorem} lower-bound the $\xi$ functions at 0. However, since $\xi$ functions are strictly increasing (as shown in Appendix~\ref{appF: proofs}), one usually needs $\sigma_0$ and $\sigma_1$ to be much closer to each other to guarantee $\xi(a)$ being smaller than $0.5$ at $a\gg0$. This not only forces the $\sigma(x)$ function to have really small semi-elasticity but also makes it problematic to define a stochastic $\sigma(x)$. For more, see Appendix~\ref{appB: the curse part}. % since then the guarantees about its bounded semi-elasticity hold just probabilistically.

To fully understand how the curse of dimensionality affects the usage of IDRS, we mention two more significant effects. First, with increasing dimension the average distance between samples tends to grow as $\sqrt{N}$. This enables bigger distance to change $\sigma(x)$. On the other hand, the average level of $\sigma(x)$ (like $\sim 0.12, 0.25, \dots$) needs to be adjusted also as $\sqrt{N}$ with increasing dimension. The bigger average level of $\sigma(x)$ we use, the more is the semi-elasticity of $\sigma(x)$ restricted by Theorem~\ref{main theorem} and Theorem~\ref{thm the concrete method}. All together, these two effects combine in a final trend that for $\sigma_0$ and $\sigma_1$ being variances used in two random test samples, $|\sigma_0/\sigma_1-1|$ is restricted to go to 0 as $1/\sqrt{N}.$ For detailed explanation, see Appendix~\ref{appC: total effect of CoD}.

\section{How to Use IDRS properly} \label{sec: practical framework}
As we discuss above, usage of the IDRS is challenging. How can we obtain valid, mathematically justified certified radiuses? Fix some design $\sigma(x)$. If $\sigma(x)$ is not trivial, to get a certified radius at $x_0$, we need to go over all the possible adversaries $x_1$ in the neighborhood of $x_0$ and compute $\sigma_1$ and $\xi_{<,>}(a)$. Then, the certified radius is the infimum over $\norm{x_0-x_1}$ for all \textit{uncertified} $x_1$ points. Of course, this is a priori infeasible. Fortunately, the $\xi$ functions possess a property that helps to simplify this procedure. For convenience, we extend the notation of $\xi$ such that $\xi(a, \sigma_1)$ additionally denotes the dependence on the $\sigma_1$ value.

\begin{theorem} \label{correctness of certification procedure}
Let $x_0$, $x_1$, $p_A$, $\sigma_0$ be as usual and denote $\norm{x_0-x_1}$ by $R$. Then, the following two statements hold:  
\begin{enumerate}
    \item Let $\sigma_1 \le \sigma_0$. Then, for all $\sigma_2: \sigma_1 \le \sigma_2 \le \sigma_0$, if $\xi_>(R, \sigma_2)>0.5$, then $\xi_>(R, \sigma_1)>0.5$.
    \item Let $\sigma_1 \ge \sigma_0$. Then, for all $\sigma_2: \sigma_1 \ge \sigma_2 \ge \sigma_0$, if $\xi_<(R, \sigma_2)>0.5$, then $\xi_<(R, \sigma_1)>0.5$.
\end{enumerate}
\end{theorem}

Theorem~\ref{correctness of certification procedure} serves as a monotonicity property. The main gain is, that for each distance $R$ from $x_0$, it is sufficient to consider just two adversaries -- the one with the biggest $\sigma_1$ (if bigger than $\sigma_0$) and the one with the smallest $\sigma_1$ (if smaller than $\sigma_0$). If we cannot certify some point $x_1$ at the distance $R$ from $x_0$, then we will for sure not be able to certify at least one of the two adversaries with the most extreme $\sigma_1$ values. 

This, however, does not suffice for most of the reasonable $\sigma(x)$ designs, since it might be still too hard to determine the two most extreme $\sigma_1$ values at some distance from $x_0$. Therefore, we assume that $\sigma(x)$ is semi-elastic with coefficient $r$. Then we have a guarantee that \mbox{$\sigma(x_0)\exp(-ra)\le\sigma(x_1)\le\sigma(x_0)\exp(ra)$} holds. Thus, we bound the worst-case extreme $\sigma_1$-s for every distance $a$. Using this, we guarantee the following certified radius. 

\begin{theorem} \label{thm the concrete method}
Let $\sigma(x)$ be an $r$-semi-elastic function and $x_0$, $p_A$, $N$, $\sigma_0$ as usual. Then, the certified radius at $x_0$ guaranteed by our method is
$$
\text{CR}(x_0) = \sup \left\{R \ge 0:~
\begin{matrix}
\xi_>(R, \sigma_0 \exp(-rR))<0.5 \\
\xi_<(R, \sigma_0 \exp(rR))<0.5
\end{matrix}
\right\}
$$
%$$\text{CR}(x_0)= \sup \{R \ge 0; \hspace{1mm} \xi_>(R, \sigma_0 \exp(-rR))< 0.5 \hspace{2mm} \text{and}$$ $$\xi_<(R, \sigma_0 \exp(rR))<0.5\}.$$
If the underlying set is empty, which is equivalent to \mbox{$p_A \le 0.5$}
 (that might occur in practice, as we use a lower bound $\underline{p_A}$ instead of $p_A$), we cannot certify any radius.
%(though this is possible just in practice, where we use $\underline{p_A}$ instead of $p_A$)
\end{theorem}

Note that Theorem~\ref{thm the concrete method} can be adjusted to the case where we have a separate estimate $\overline{p_B}$ and do not put $\overline{p_B}=1-\underline{p_A}$ (see Appendix~\ref{appC: multi-class regime}). Since the bigger the semi-elasticity constant of $\sigma(x)$ is, the worse certifications we obtain, it is important to estimate the constant tightly. Even with a good estimate of $r$, we still get smaller certified radiuses in comparison with using the $\sigma(x)$ exactly, but that is a prize that is inevitable for the feasibility of the method. 

The algorithm is then very easy - we just pick sufficiently dense space of possible radiuses and determine the smallest, for which either $\xi_>(R, \sigma_0 \exp(-rR))$ or $\xi_<(R, \sigma_0 \exp(rR))$ is larger than $0.5$. The only non-trivial part is how to evaluate the $\xi$ functions. For small values of $R$, the $\exp(-rR)$ is very close to 1 and from the definition of $\xi$ functions it is obvious that this results in extremely big inputs to the cdf and quantile function of NCCHSQ. To avoid numerical problems, we employ a simple hack where we assume thresholds for $\sigma_1$ such that for $R$ small enough, these thresholds are used instead of $\sigma_0 \exp(\pm rR))$. Unfortunately, the numerical stability still disables the usage of this method on really high-dimensional datasets like ImageNet. For more details on implementation, see Appendix~\ref{appD: implementation details}.

\section{The Design of $\sigma(x)$ and Experiments} \label{experiments}
The only missing ingredient to finally being able to use IDRS is the $\sigma(x)$ function. As we have seen, this function has to be $r$-semi-elastic for rather small $r$ and ideally deterministic. Yet it should at least roughly fulfill the requirements imposed by the motivation -- it should possess big values for points far from the decision boundary of $f$ and rather small for points close to it. Adhering to these restrictions, we use the following function: \begin{equation} \sigma(x) = \sigma_b \exp\left(r \left(\frac{1}{k}  \sum\limits_{x_i \in \mathcal{N}_k(x)} \norm{x-x_i} -m\right)\right)\label{eq: the sigma fcn}\end{equation} for $\sigma_b$ being a \textit{base standard deviation}, $r$ the required semi-elasticity, $\{x_i\}_{i=1}^d$ the training set, $\mathcal{N}_k(x)$ the $k$ nearest neighbors of $x$ and $m$ the normalization constant. Intuitively, if a sample is far from all other samples, it will be far from the decision boundary, unless the network overfits to this sample. On the other hand, the dense clusters of samples are more likely to be positioned near the decision boundary, since such clusters have a high leverage on the network's weights, forcing the decision boundary to adapt well to the geometry of the cluster (for more insight on the relation between the geometry of the data and distance from decision boundary see, for example, \citet{baldock2021deep}). To use such a function, however, we first prove that it is indeed $r$-semi-elastic. 

\begin{theorem} \label{r semi elasticity of sigma(x)} The standard deviation function $\sigma(x)$ defined in \eqref{eq: the sigma fcn} is $r$-semi-elastic. 
\end{theorem}

We test our IDRS and $\sigma(x)$ function extensively. For both CIFAR10 \citep{cifar} and MNIST \citep{mnist} datasets, we analyze series of different experimental setups, including experiments with an input-dependent train-time Gaussian data augmentation. We present a direct comparison of our method with the constant $\sigma$ method using evaluation strategy from \cite{cohen2019certified} (all other experiments, including ablation studies, and the discussion on the hyperparameter selection are presented in Appendix~\ref{appE: experiments and ablations}). Here, we compare \cite{cohen2019certified}'s evaluations for $\sigma= 0.12, 0.25, 0.50$ with our evaluations, setting $\sigma_b=\sigma$, $r=0.01, 0.02$, $k=20$, $m=5, 1.5$ (for CIFAR10 and MNIST, respectively), applied on models trained with Gaussian data augmentation, using constant standard deviation roughly equal to the average test-time $\sigma(x)$ or test-time $\sigma$. For CIFAR10, these levels of train-time standard deviation are $\sigma_{tr}= 0.126, 0.263, 0.53$ and for MNIST $\sigma_{tr}= 0.124, 0.258, 0.517$. In this way, the levels of $\sigma(x)$ we use in the direct comparison are spread from the values roughly equal to \cite{cohen2019certified}'s constant $\sigma$ to higher values. The results are depicted in Figure~\ref{me vs. cohen main comparison}.

\begin{figure*}[h!]
    \centering
    \begin{minipage}[b]{0.32\linewidth}
        \includegraphics[width=\textwidth]{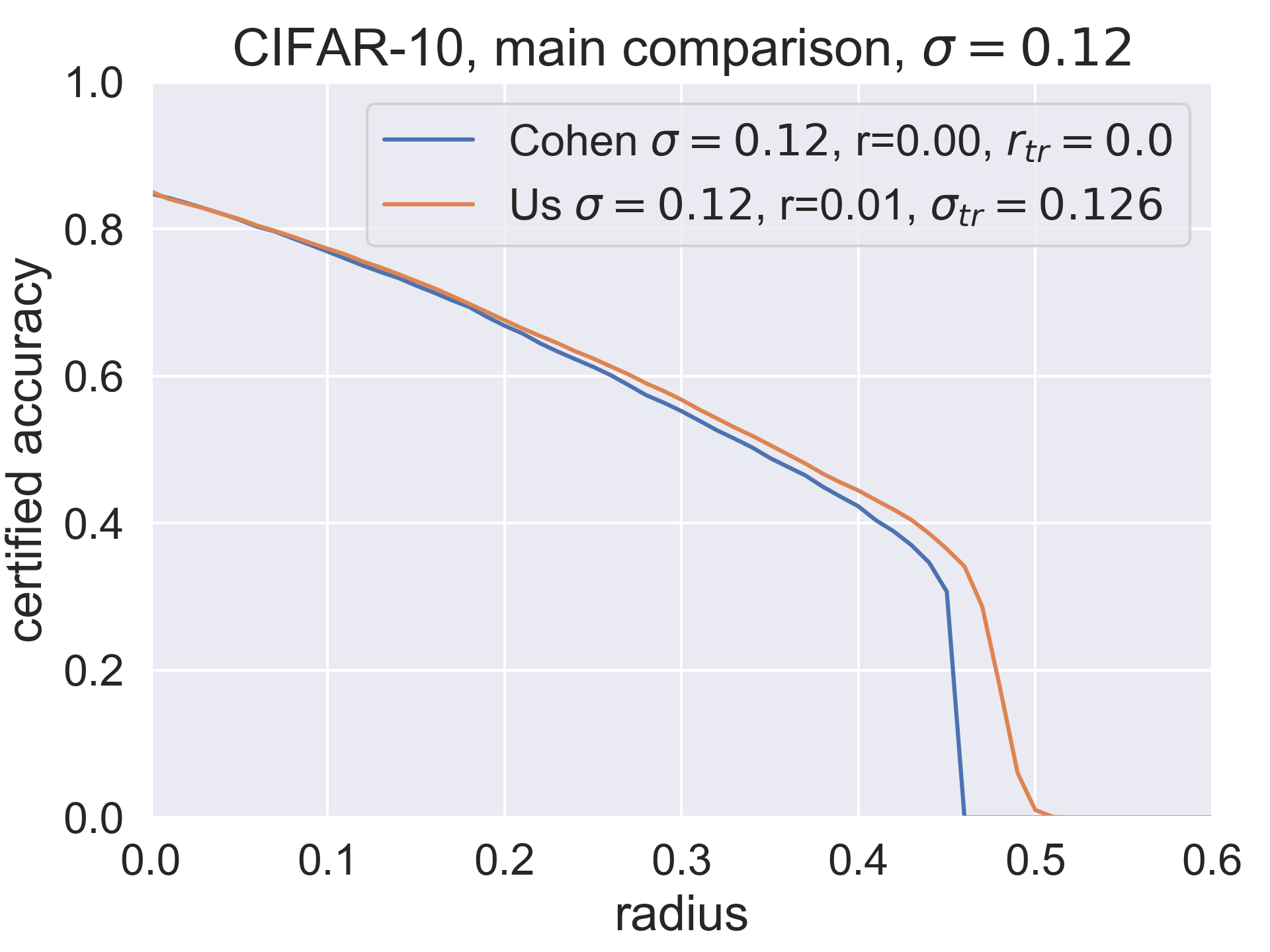}
    \end{minipage}
    \begin{minipage}[b]{0.32\linewidth}
        \includegraphics[width=\textwidth]{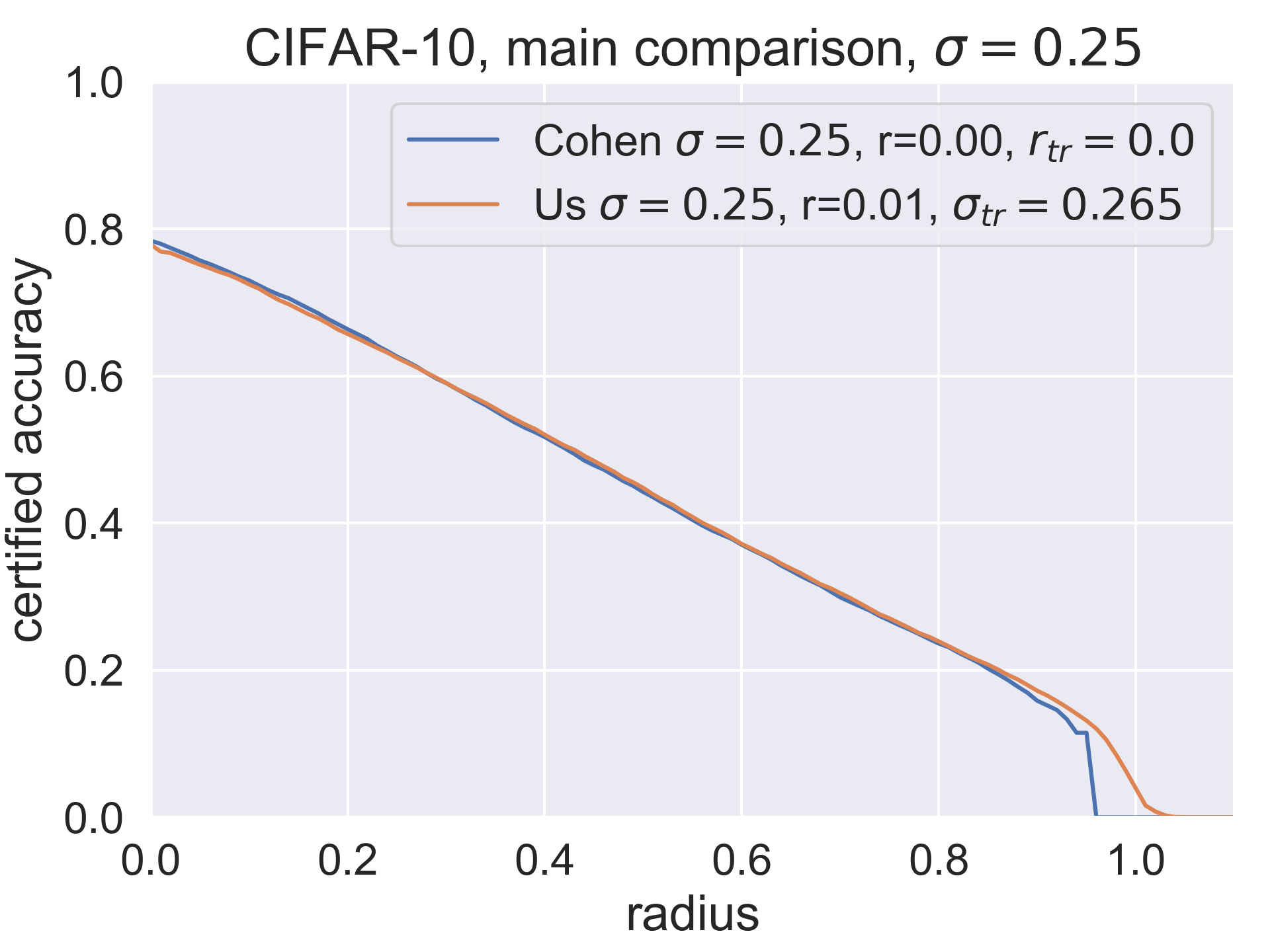}
    \end{minipage}
    \begin{minipage}[b]{0.32\linewidth}
        \includegraphics[width=\textwidth]{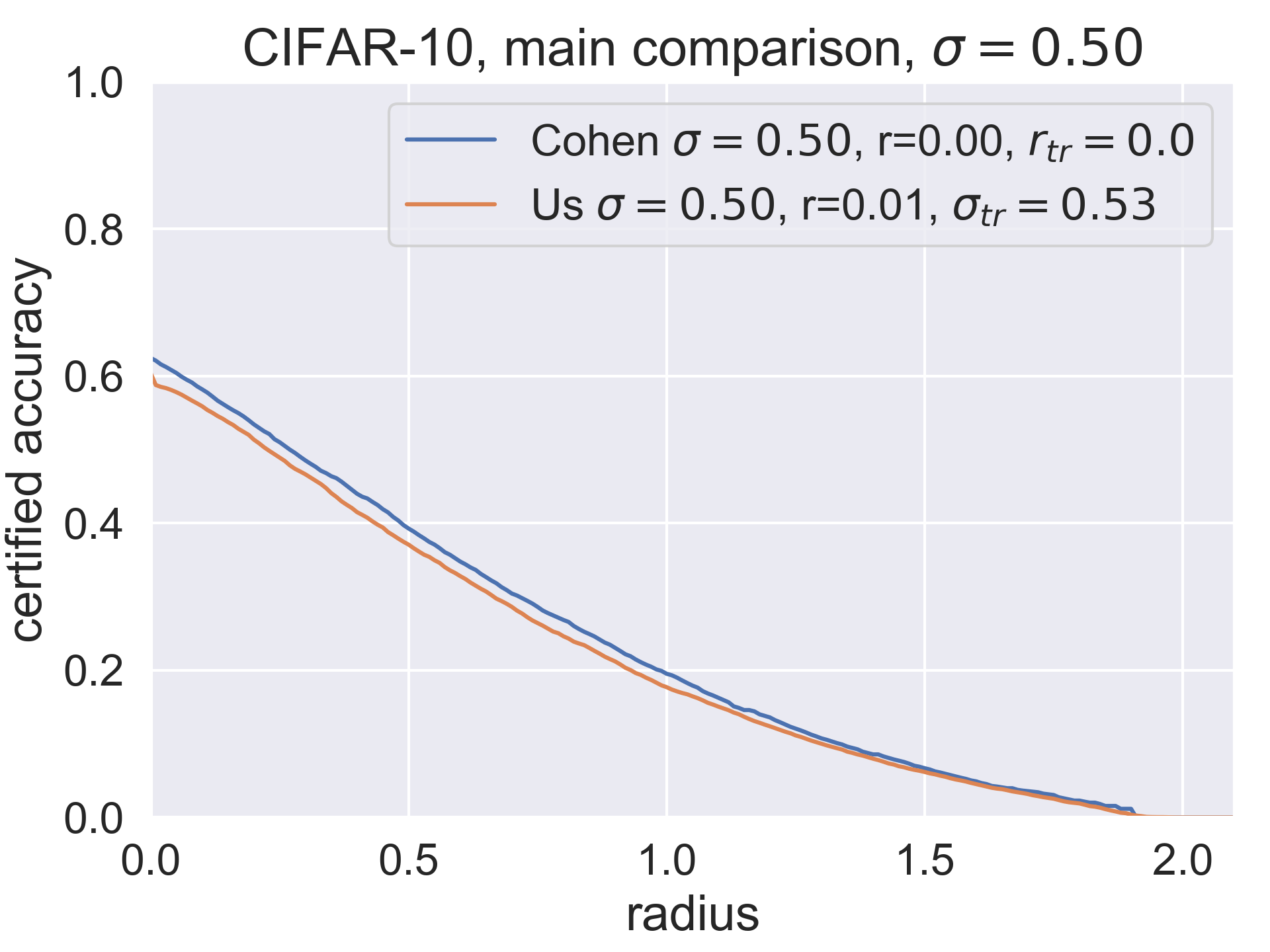}
    \end{minipage}
    \begin{minipage}[b]{0.32\linewidth}
        \includegraphics[width=\textwidth]{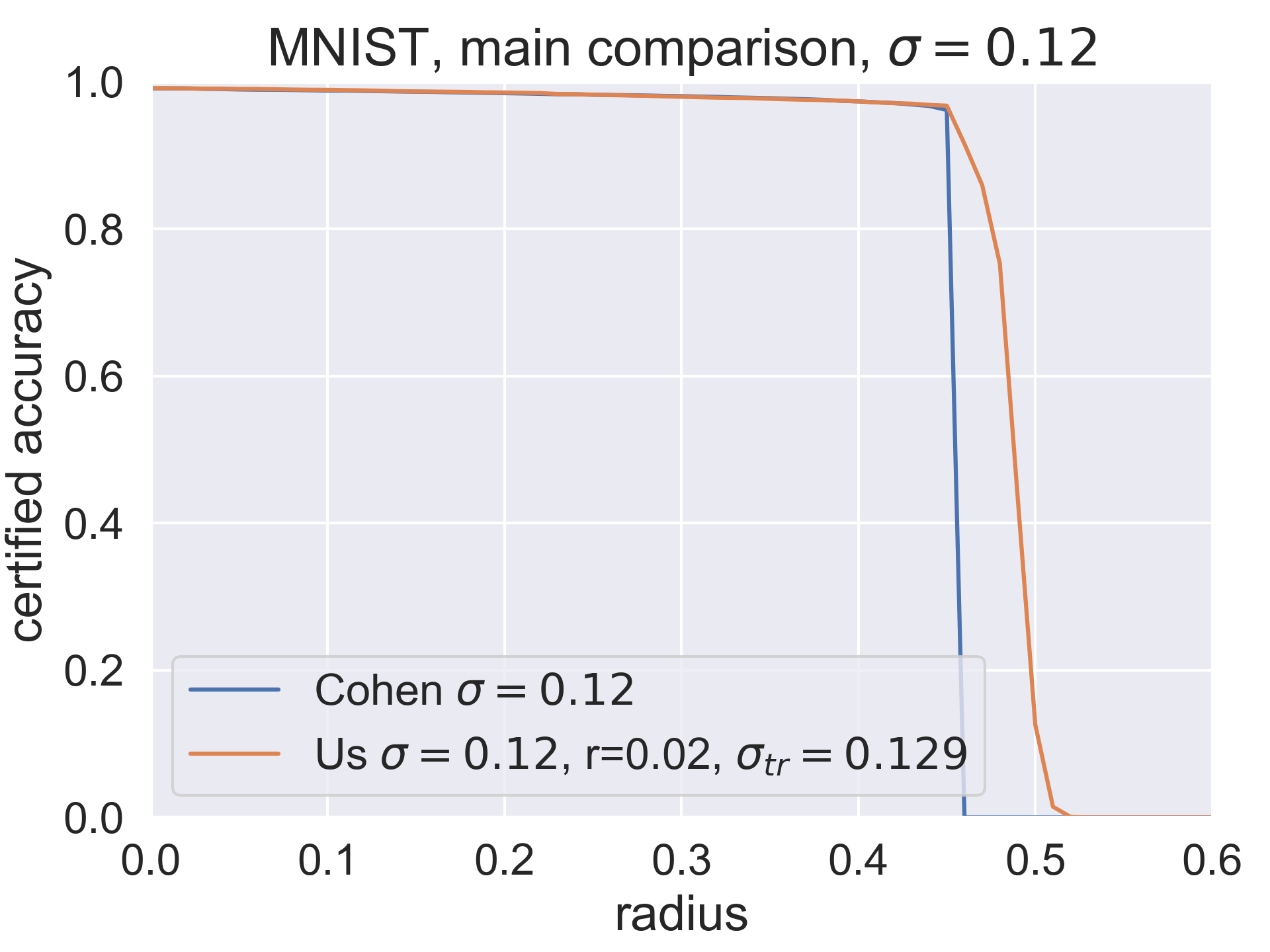}
    \end{minipage}
    \begin{minipage}[b]{0.32\linewidth}
        \includegraphics[width=\textwidth]{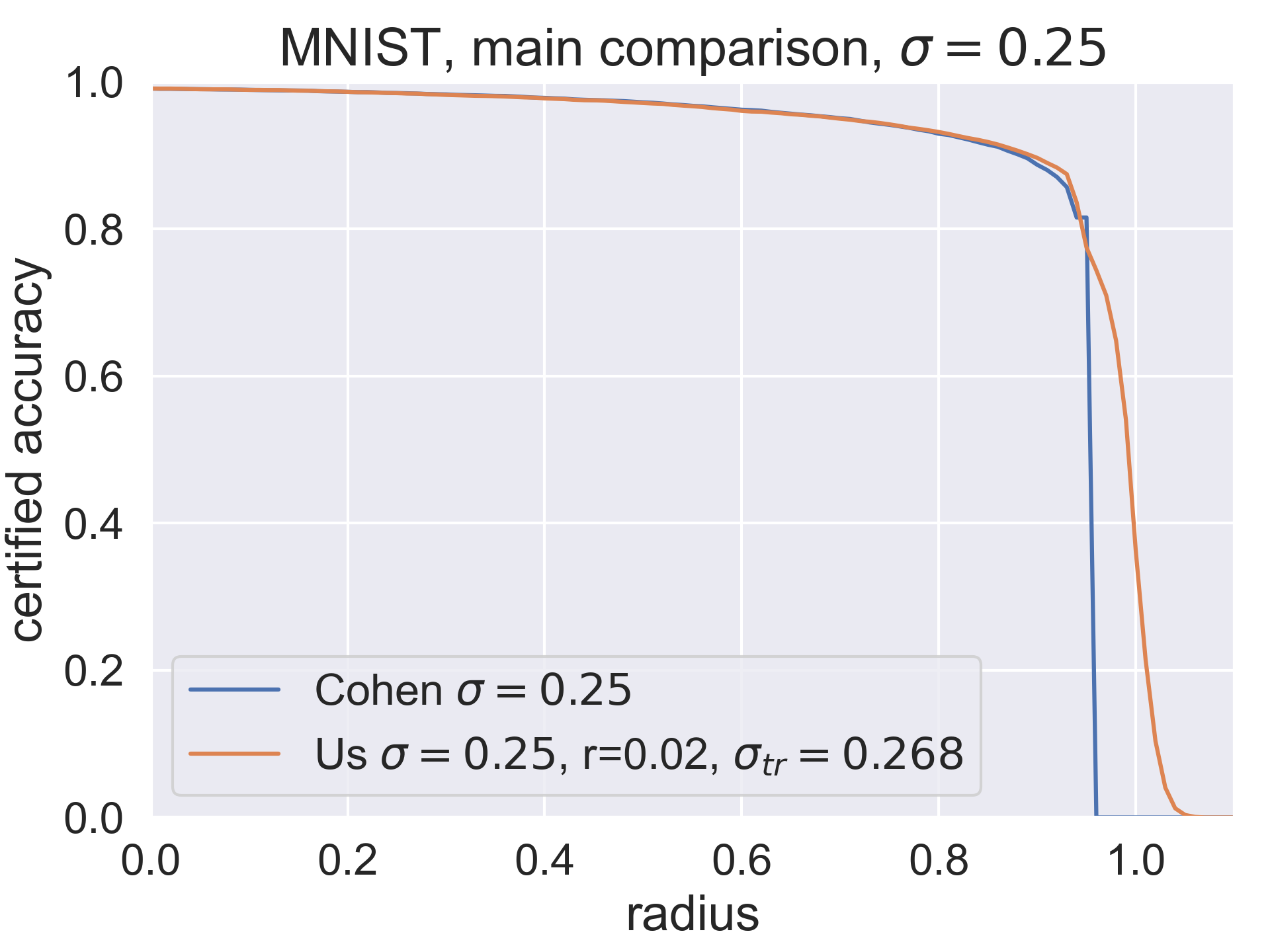}
    \end{minipage}
    \begin{minipage}[b]{0.32\linewidth}
        \includegraphics[width=\textwidth]{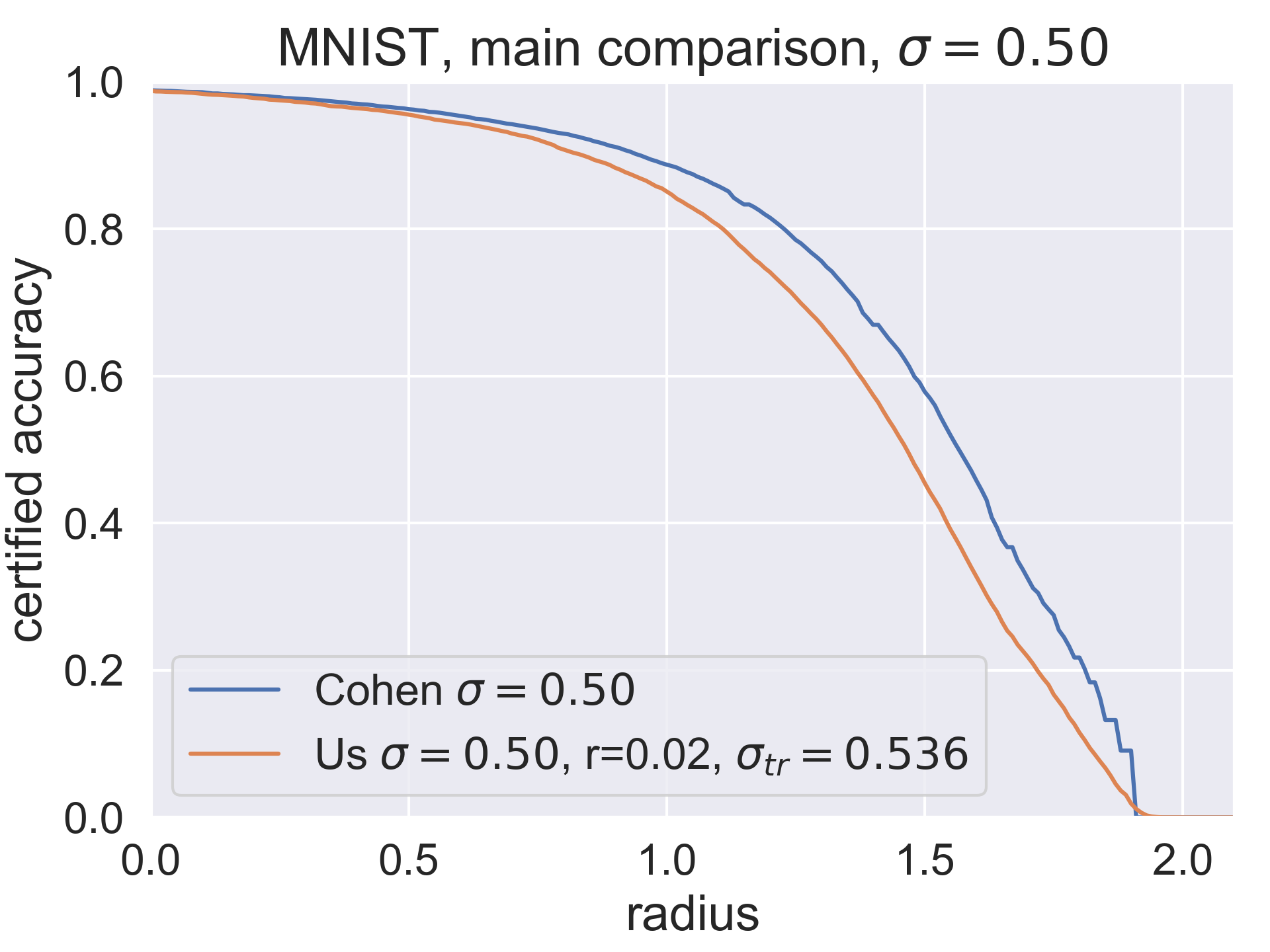}
    \end{minipage}
    \caption{Comparison of certified accuracy plots for \citet{cohen2019certified}'s method and our work.}
    \label{me vs. cohen main comparison}
\end{figure*}

From Figure~\ref{me vs. cohen main comparison} we see that we outperform the constant $\sigma$ for small levels of $\sigma$, such as $0.12$ or $0.25$. On higher levels of $\sigma$, we are, in general, worse (see explanation in Appendix~\ref{appC: ssec: why small sigma values better}). The most visible improvement is in mitigation of the truncation of certified accuracy (certified accuracy waterfall). To comment on the other two issues, we provide \mbox{Tables \ref{accuracy table main text comp} and \ref{standard devs of class-wise accuracies main text comp}} with the clean accuracies and class-wise accuracy standard deviations, respectively. These results are the averages of 8 independent runs and in Table~\ref{accuracy table main text comp}, the displayed error values are equal to empirical standard deviations. 

\begin{table*}[ht]
\centering
\caption{Clean accuracies for both input-dependent and constant $\sigma$ evaluation strategies on CIFAR10 (C) and MNIST (M).}
\vspace{2mm}
\begin{tabular}{||c||c|c|c|c|c||} 
\hline\hline
 & dataset & $\sigma=0.12$ & $\sigma=0.25$ & $\sigma=0.50$ \\ 
\hline\hline
$r=0.01, \sigma_{tr}$ increased & C & $0.852 \pm  0.002$ & $0.780 \pm 0.013$ & $0.673 \pm 0.008$ \\ 
\hline
$r=0.00$ & C & $0.851 \pm 0.006$ & $0.792 \pm 0.004$ & $0.674 \pm 0.018$ \\
\hline\hline
$r=0.01, \sigma_{tr}$ increased & M & $0.9912 \pm 0.0003$ & $0.9910 \pm 0.0006$ & $0.9881 \pm 0.0003$ \\ 
\hline
$r=0.00$ & M & $0.9914 \pm 0.0004$ & $0.9907 \pm 0.0004$ & $0.9886 \pm 0.0005$ \\
\hline\hline
\end{tabular}
\label{accuracy table main text comp}
\end{table*}

\begin{table*}[t!]
\caption{Class-wise accuracy (i.e. how many points of a certain class are correctly classified divided by the size of this class) standard deviations for both input-dependent and constant $\sigma$ evaluation strategies on CIFAR10 (C) and MNIST (M).}
\vspace{2mm}
\centering
\begin{tabular}{||c||c|c|c|c|c||} 
\hline\hline
 & dataset & $\sigma=0.12$ & $\sigma=0.25$ & $\sigma=0.50$ \\ 
\hline\hline
$r=0.01, \sigma_{tr}$ increased & C & 0.076 & 0.099 & 0.120 \\ 
\hline
$r=0.00$ & C & 0.076 & 0.097 & 0.122 \\
\hline\hline
$r=0.01, \sigma_{tr}$ increased & M & 0.00775 & 0.00777 & 0.00930 \\ 
\hline
$r=0.00$ & M & 0.00751 & 0.00778 & 0.00934 \\
\hline\hline
\end{tabular}
\label{standard devs of class-wise accuracies main text comp}
\end{table*}

From \mbox{Tables \ref{accuracy table main text comp} and \ref{standard devs of class-wise accuracies main text comp}}, we draw two conclusions. First, it is not easy to judge about the robustness vs.\ accuracy trade-off, because the differences in clean accuracies are not statistically significant in any of the experiments (not even for CIFAR10 and $\sigma=0.25$, where the difference is at least pronounced). However, the general trend in Table~\ref{accuracy table main text comp} indicates that the clean accuracies tend to slightly decrease with the increasing rate. The decrease is roughly equivalent to the drop in accuracy for the case when we just use constant $\sigma$ during evaluation with its value set to the average $\sigma(x)$. In that case, the certified radiuses would be also roughly equal in average, but ours would still encounter a less severe certified accuracy drop. Second, the standard deviations of the class-wise accuracies, which serve as a good measure of the impact of the shrinking phenomenon and subsequent fairness, don't significantly change after applying the non-constant RS. 

%All in all, we see that the proposed $\sigma(x)$ requires improvement, until IDRS is able to significantly beat constant smoothing, as it happens in our toy example in Appendix~\ref{appA: the motivation}. The difference between CIFAR10 and MNIST and the toy dataset is that $\sigma(x)$ from \eqref{eq: the sigma fcn} is well-suited for the toy dataset's geometry, but not for the geometry of images. The problem is, this $\sigma(x)$ does not correspond to the distances from the decision boundary so well, because euclidean distances between images do not align with the ``visual distances'' as well as the distances implicitly imposed by $f$, which is usually composed of convolutions. We believe, however, that improvements in $\sigma(x)$ design are possible and we leave the investigation in this matter as an open and attractive research question. The way to go is to define a metric on input space that would better reflect the geometry of images and convolutional neural networks.

\section{Related Work}
\label{sec: related work}
Since the vulnerability of deep neural networks against adversarial attacks has been noticed by \citet{szegedy2013intriguing, 10.1007/978-3-642-40994-3_25}, a lot of effort has been put into making neural nets more robust. There are two types of solutions -- empirical and certified defenses. While empirical defenses suggest heuristics to make models robust, certified approaches additionally provide a way to compute a mathematically valid robust radius.

One of the most effective empirical defenses, \textit{adversarial training} \citep{goodfellow2014explaining, kurakin2016adversarial, madry2017towards}, is based on a intuitive idea to use adversarial examples for training. Unfortunately, together with adversarial training, other promising empirical defenses were subsequently broken by more sophisticated adversarial methods (for instance by \citet{carlini2017adversarial, athalye2018robustness, athalye2018obfuscated}, among others). 

Among many certified defenses \citep{NEURIPS2018_48584348, pmlr-v97-anil19a, NIPS2017_e077e1a5, wong2018provable, raghunathan2018certified, pmlr-v80-mirman18b, pmlr-v80-weng18a}, one of the most successful yet is RS. While \citet{lecuyer2019certified} introduced the method within the context of differential privacy, \citet{li2018certified} proceeded via R{\'e}nyi divergences. Possibly the most prominent work on RS is that of \citet{cohen2019certified}, where authors fully established RS and proved tight certification guarantees.  

Later, a lot of authors further worked with RS. The work of \citet{yang2020randomized} generalizes the certification provided by \citet{cohen2019certified} to certifications with respect to the general $l_p$ norms and provide the optimal smoothing distributions for each of the norms. Other works point out different problems or weaknesses of RS such as the curse of dimensionality \citep{kumar2020curse, hayes2020extensions, wu2021completing}, robustness vs.\ accuracy trade-off \citep{gao2020analyzing} or a shrinking phenomenon \citep{mohapatra2020rethinking}, which yields serious fairness issues \citep{mohapatra2020rethinking}. 

The work of \citet{mohapatra2020higher} improves RS further by introducing the first-order information about $g$. In this work, authors not only estimate $g(x)$, but also $\nabla g(x)$, making more restrictions on the possible base models $f$ that might have created $g$. \citet{zhai2020macer} and \citet{salman2019provably} improve the training procedure of $f$ to yield better robustness guarantees of $g$. \citet{salman2019provably} directly use adversarial training of the base classifier $f$. Finally, \citet{zhai2020macer} introduce the so-called \textit{soft smoothing}, which enables to compute gradients directly for $g$ and construct a training method, which optimizes directly for the robustness of $g$ via the gradient descent.

To address several issues connected to randomized smoothing, there have already been four works that introduce the usage of IDRS. \citet{wang2021pretraintofinetune} divide $\mathbb{R}^N$ into several regions $R_i, i \in \{1, \dots, K\}$ and optimize for $\sigma_i, i \in \{1, \dots, K\}$ locally, such that $\sigma_i$ is a most suitable choice for the region $R_i$. Yet this work partially solves some problems of randomized smoothing, it also possesses some practical and philosophical issues (see Appendix~\ref{appB: concurrent work}). \citet{alfarra2020data, eiras2021ancer, chen2021insta} suggest to optimize for locally optimal $\sigma_i$, for each sample $x_i$ from the test set. A similar strategy is proposed by these works in the training phase, with the intention of obtaining the base model $f$ that is most suitable for the construction of the smoothed classifier $g$. They demonstrate, that by using this input-dependent approach, one can overcome some of the main problems of randomized smoothing. However, as we demonstrate in Appendix~\ref{appB: concurrent work}, their methodology is not valid and therefore their results are not trustworthy. In the latest version of their work, \citet{alfarra2020data} themselves point out this issue and try to fix it by a very similar strategy to the one of \citet{wang2021pretraintofinetune}. Thus, their work still carries all the problems connected to this method. 
%Overall, there is no method that correctly generalizes the approach \cite{cohen2019certified} for the usage of not locally constant $\sigma$.

\section{Conclusions} \label{conclusions}
We show in this work that input-dependent randomized smoothing suffers from the curse of dimensionality. In the high-dimensional regime, the usage of input-dependent $\sigma(x)$ is put under strict constraints. The $\sigma(x)$ function is forced to have very small semi-elasticity. This is in conflict with some recent works, which have used the input-dependent randomized smoothing without mathematical justification and therefore claim invalid results, or results of questionable comparability and practical usability. It seems that input-dependent randomized smoothing has limited potential of improvement over the classical, constant-$\sigma$ RS. Moreover, due to numerical instability, the computation of certified radiuses on high-dimensional datasets like ImageNet remains to be an open challenge. 

On the other hand, we prepare a ready-to-use mathematically underlined framework for the usage of the input-dependent RS and show that it works well for small to medium-sized problems. We also show, via extensive experiments, that our concrete design of the $\sigma(x)$ function reasonably mitigates the truncation issue connected to constant-$\sigma$ RS and is capable of mitigating the robustness vs.\ accuracy one on simpler datasets. The most intriguing and promising direction for the future work lies in the development of new $\sigma(x)$ functions, which could treat the mentioned issues even more efficiently.

This improvement is necessary to make IDRS able to significantly beat constant smoothing, as it happens in our toy example in Appendix~\ref{appA: the motivation}. The difference between CIFAR10 and MNIST and the toy dataset is that $\sigma(x)$ from \eqref{eq: the sigma fcn} is well-suited for the toy dataset's geometry, but not to the same extent for the geometry of images. One possible reason is that this $\sigma(x)$ does not correspond to the distances from the decision boundary across the entire dataset, because the euclidean distances between images do not align with the ``distances in images' content''. We believe, however, that improvements in $\sigma(x)$ design are possible and we leave the investigation in this matter as an open and attractive research question. One way to go is to define a metric on the input space that would better reflect the geometry of images and convolutional neural networks.

%\section{Reproducibility statement}
%\input{sections/reproducibility}

\bibliography{example_paper}

\begin{thebibliography}{38}
\providecommand{\natexlab}[1]{#1}
\providecommand{\url}[1]{\texttt{#1}}
\expandafter\ifx\csname urlstyle\endcsname\relax
  \providecommand{\doi}[1]{doi: #1}\else
  \providecommand{\doi}{doi: \begingroup \urlstyle{rm}\Url}\fi

\bibitem[Alfarra et~al.(2020)Alfarra, Bibi, Torr, and Ghanem]{alfarra2020data}
Alfarra, M., Bibi, A., Torr, P.~H., and Ghanem, B.
\newblock Data dependent randomized smoothing.
\newblock \emph{arXiv preprint arXiv:2012.04351}, 2020.

\bibitem[Anil et~al.(2019)Anil, Lucas, and Grosse]{pmlr-v97-anil19a}
Anil, C., Lucas, J., and Grosse, R.
\newblock Sorting out {L}ipschitz function approximation.
\newblock In Chaudhuri, K. and Salakhutdinov, R. (eds.), \emph{Proceedings of
  the 36th International Conference on Machine Learning}, volume~97 of
  \emph{Proceedings of Machine Learning Research}, pp.\  291--301. PMLR, 09--15
  Jun 2019.
\newblock URL \url{http://proceedings.mlr.press/v97/anil19a.html}.

\bibitem[Athalye \& Carlini(2018)Athalye and Carlini]{athalye2018robustness}
Athalye, A. and Carlini, N.
\newblock On the robustness of the cvpr 2018 white-box adversarial example
  defenses.
\newblock \emph{arXiv preprint arXiv:1804.03286}, 2018.

\bibitem[Athalye et~al.(2018)Athalye, Carlini, and
  Wagner]{athalye2018obfuscated}
Athalye, A., Carlini, N., and Wagner, D.
\newblock Obfuscated gradients give a false sense of security: Circumventing
  defenses to adversarial examples.
\newblock In \emph{International conference on machine learning}, pp.\
  274--283. PMLR, 2018.

\bibitem[Baldock et~al.(2021)Baldock, Maennel, and Neyshabur]{baldock2021deep}
Baldock, R., Maennel, H., and Neyshabur, B.
\newblock Deep learning through the lens of example difficulty.
\newblock \emph{Advances in Neural Information Processing Systems},
  34:\penalty0 10876--10889, 2021.

\bibitem[Biggio et~al.(2013)Biggio, Corona, Maiorca, Nelson,
  {\v{S}}rndi{\'{c}}, Laskov, Giacinto, and Roli]{10.1007/978-3-642-40994-3_25}
Biggio, B., Corona, I., Maiorca, D., Nelson, B., {\v{S}}rndi{\'{c}}, N.,
  Laskov, P., Giacinto, G., and Roli, F.
\newblock Evasion attacks against machine learning at test time.
\newblock In Blockeel, H., Kersting, K., Nijssen, S., and {\v{Z}}elezn{\'y}, F.
  (eds.), \emph{Machine Learning and Knowledge Discovery in Databases}, pp.\
  387--402, Berlin, Heidelberg, 2013. Springer Berlin Heidelberg.
\newblock ISBN 978-3-642-40994-3.

\bibitem[Carlini \& Wagner(2017)Carlini and Wagner]{carlini2017adversarial}
Carlini, N. and Wagner, D.
\newblock Adversarial examples are not easily detected: Bypassing ten detection
  methods.
\newblock In \emph{Proceedings of the 10th ACM workshop on artificial
  intelligence and security}, pp.\  3--14, 2017.

\bibitem[Chen et~al.(2021)Chen, Kong, Yu, Luque, Goldstein, and
  Huang]{chen2021insta}
Chen, C., Kong, K., Yu, P., Luque, J., Goldstein, T., and Huang, F.
\newblock Insta-rs: Instance-wise randomized smoothing for improved robustness
  and accuracy.
\newblock \emph{arXiv preprint arXiv:2103.04436}, 2021.

\bibitem[Cohen et~al.(2019)Cohen, Rosenfeld, and Kolter]{cohen2019certified}
Cohen, J., Rosenfeld, E., and Kolter, Z.
\newblock Certified adversarial robustness via randomized smoothing.
\newblock In \emph{International Conference on Machine Learning}, pp.\
  1310--1320. PMLR, 2019.

\bibitem[Eiras et~al.(2021)Eiras, Alfarra, Kumar, Torr, Dokania, Ghanem, and
  Bibi]{eiras2021ancer}
Eiras, F., Alfarra, M., Kumar, M.~P., Torr, P.~H., Dokania, P.~K., Ghanem, B.,
  and Bibi, A.
\newblock Ancer: Anisotropic certification via sample-wise volume maximization.
\newblock \emph{arXiv preprint arXiv:2107.04570}, 2021.

\bibitem[Eykholt et~al.(2018)Eykholt, Evtimov, Fernandes, Li, Rahmati, Xiao,
  Prakash, Kohno, and Song]{eykholt2018robust}
Eykholt, K., Evtimov, I., Fernandes, E., Li, B., Rahmati, A., Xiao, C.,
  Prakash, A., Kohno, T., and Song, D.
\newblock Robust physical-world attacks on deep learning visual classification.
\newblock In \emph{Proceedings of the IEEE conference on computer vision and
  pattern recognition}, pp.\  1625--1634, 2018.

\bibitem[Gao et~al.(2020)Gao, Rosenberg, Fawaz, Jha, and Hsu]{gao2020analyzing}
Gao, Y., Rosenberg, H., Fawaz, K., Jha, S., and Hsu, J.
\newblock Analyzing accuracy loss in randomized smoothing defenses.
\newblock \emph{arXiv preprint arXiv:2003.01595}, 2020.

\bibitem[Goodfellow et~al.(2014)Goodfellow, Shlens, and
  Szegedy]{goodfellow2014explaining}
Goodfellow, I.~J., Shlens, J., and Szegedy, C.
\newblock Explaining and harnessing adversarial examples.
\newblock \emph{arXiv preprint arXiv:1412.6572}, 2014.

\bibitem[G{\"u}hring et~al.(2020)G{\"u}hring, Raslan, and
  Kutyniok]{guhring2020expressivity}
G{\"u}hring, I., Raslan, M., and Kutyniok, G.
\newblock Expressivity of deep neural networks.
\newblock \emph{arXiv preprint arXiv:2007.04759}, 2020.

\bibitem[Hayes(2020)]{hayes2020extensions}
Hayes, J.
\newblock Extensions and limitations of randomized smoothing for robustness
  guarantees.
\newblock In \emph{Proceedings of the IEEE/CVF Conference on Computer Vision
  and Pattern Recognition Workshops}, pp.\  786--787, 2020.

\bibitem[Hein \& Andriushchenko(2017)Hein and
  Andriushchenko]{NIPS2017_e077e1a5}
Hein, M. and Andriushchenko, M.
\newblock Formal guarantees on the robustness of a classifier against
  adversarial manipulation.
\newblock In Guyon, I., Luxburg, U.~V., Bengio, S., Wallach, H., Fergus, R.,
  Vishwanathan, S., and Garnett, R. (eds.), \emph{Advances in Neural
  Information Processing Systems}, volume~30. Curran Associates, Inc., 2017.
\newblock URL
  \url{https://proceedings.neurips.cc/paper/2017/file/e077e1a544eec4f0307cf5c3c721d944-Paper.pdf}.

\bibitem[Krizhevsky(2009)]{cifar}
Krizhevsky, A.
\newblock Learning multiple layers of features from tiny images, 2009.
\newblock URL
  \url{http://www.cs.toronto.edu/~kriz/learning-features-2009-TR.pdf}.

\bibitem[Kumar et~al.(2020)Kumar, Levine, Goldstein, and Feizi]{kumar2020curse}
Kumar, A., Levine, A., Goldstein, T., and Feizi, S.
\newblock Curse of dimensionality on randomized smoothing for certifiable
  robustness.
\newblock In \emph{International Conference on Machine Learning}, pp.\
  5458--5467. PMLR, 2020.

\bibitem[Kurakin et~al.(2016)Kurakin, Goodfellow, and
  Bengio]{kurakin2016adversarial}
Kurakin, A., Goodfellow, I., and Bengio, S.
\newblock Adversarial machine learning at scale.
\newblock \emph{arXiv preprint arXiv:1611.01236}, 2016.

\bibitem[LeCun et~al.(1999)LeCun, Cortes, and Burges]{mnist}
LeCun, Y., Cortes, C., and Burges, C.~J.
\newblock The mnist database of handwritten digits, 1999.
\newblock URL \url{http://yann.lecun.com/exdb/mnist/}.

\bibitem[Lecuyer et~al.(2019)Lecuyer, Atlidakis, Geambasu, Hsu, and
  Jana]{lecuyer2019certified}
Lecuyer, M., Atlidakis, V., Geambasu, R., Hsu, D., and Jana, S.
\newblock Certified robustness to adversarial examples with differential
  privacy.
\newblock In \emph{2019 IEEE Symposium on Security and Privacy (SP)}, pp.\
  656--672. IEEE, 2019.

\bibitem[Li et~al.(2019)Li, Chen, Wang, and Carin]{li2018certified}
Li, B., Chen, C., Wang, W., and Carin, L.
\newblock Certified adversarial robustness with additive noise.
\newblock \emph{Advances in Neural Information Processing Systems},
  32:\penalty0 9464--9474, 2019.

\bibitem[Madry et~al.(2017)Madry, Makelov, Schmidt, Tsipras, and
  Vladu]{madry2017towards}
Madry, A., Makelov, A., Schmidt, L., Tsipras, D., and Vladu, A.
\newblock Towards deep learning models resistant to adversarial attacks.
\newblock \emph{arXiv preprint arXiv:1706.06083}, 2017.

\bibitem[Mirman et~al.(2018)Mirman, Gehr, and Vechev]{pmlr-v80-mirman18b}
Mirman, M., Gehr, T., and Vechev, M.
\newblock Differentiable abstract interpretation for provably robust neural
  networks.
\newblock In Dy, J. and Krause, A. (eds.), \emph{Proceedings of the 35th
  International Conference on Machine Learning}, volume~80 of \emph{Proceedings
  of Machine Learning Research}, pp.\  3578--3586. PMLR, 10--15 Jul 2018.
\newblock URL \url{http://proceedings.mlr.press/v80/mirman18b.html}.

\bibitem[Mohapatra et~al.(2020{\natexlab{a}})Mohapatra, Ko, Liu, Chen, Daniel,
  et~al.]{mohapatra2020rethinking}
Mohapatra, J., Ko, C.-Y., Liu, S., Chen, P.-Y., Daniel, L., et~al.
\newblock Rethinking randomized smoothing for adversarial robustness.
\newblock \emph{arXiv preprint arXiv:2003.01249}, 2020{\natexlab{a}}.

\bibitem[Mohapatra et~al.(2020{\natexlab{b}})Mohapatra, Ko, Weng, Chen, Liu,
  and Daniel]{mohapatra2020higher}
Mohapatra, J., Ko, C.-Y., Weng, T.-W., Chen, P.-Y., Liu, S., and Daniel, L.
\newblock Higher-order certification for randomized smoothing.
\newblock \emph{Advances in Neural Information Processing Systems}, 33,
  2020{\natexlab{b}}.

\bibitem[Raghunathan et~al.(2018)Raghunathan, Steinhardt, and
  Liang]{raghunathan2018certified}
Raghunathan, A., Steinhardt, J., and Liang, P.
\newblock Certified defenses against adversarial examples.
\newblock In \emph{International Conference on Learning Representations}, 2018.

\bibitem[Robert(1990)]{robert1990some}
Robert, C.
\newblock On some accurate bounds for the quantiles of a non-central chi
  squared distribution.
\newblock \emph{Statistics \& probability letters}, 10\penalty0 (2):\penalty0
  101--106, 1990.

\bibitem[Salman et~al.(2019)Salman, Li, Razenshteyn, Zhang, Zhang, Bubeck, and
  Yang]{salman2019provably}
Salman, H., Li, J., Razenshteyn, I.~P., Zhang, P., Zhang, H., Bubeck, S., and
  Yang, G.
\newblock Provably robust deep learning via adversarially trained smoothed
  classifiers.
\newblock In \emph{NeurIPS}, 2019.

\bibitem[Szegedy et~al.(2013)Szegedy, Zaremba, Sutskever, Bruna, Erhan,
  Goodfellow, and Fergus]{szegedy2013intriguing}
Szegedy, C., Zaremba, W., Sutskever, I., Bruna, J., Erhan, D., Goodfellow, I.,
  and Fergus, R.
\newblock Intriguing properties of neural networks.
\newblock \emph{arXiv preprint arXiv:1312.6199}, 2013.

\bibitem[Tsuzuku et~al.(2018)Tsuzuku, Sato, and Sugiyama]{NEURIPS2018_48584348}
Tsuzuku, Y., Sato, I., and Sugiyama, M.
\newblock Lipschitz-margin training: Scalable certification of perturbation
  invariance for deep neural networks.
\newblock In Bengio, S., Wallach, H., Larochelle, H., Grauman, K.,
  Cesa-Bianchi, N., and Garnett, R. (eds.), \emph{Advances in Neural
  Information Processing Systems}, volume~31. Curran Associates, Inc., 2018.
\newblock URL
  \url{https://proceedings.neurips.cc/paper/2018/file/485843481a7edacbfce101ecb1e4d2a8-Paper.pdf}.

\bibitem[Van~Erven \& Harremos(2014)Van~Erven and Harremos]{van2014renyi}
Van~Erven, T. and Harremos, P.
\newblock R{\'e}nyi divergence and kullback-leibler divergence.
\newblock \emph{IEEE Transactions on Information Theory}, 60\penalty0
  (7):\penalty0 3797--3820, 2014.

\bibitem[Wang et~al.(2021)Wang, Zhai, He, Wang, and
  Jian]{wang2021pretraintofinetune}
Wang, L., Zhai, R., He, D., Wang, L., and Jian, L.
\newblock Pretrain-to-finetune adversarial training via sample-wise randomized
  smoothing.
\newblock 2021.
\newblock URL \url{https://openreview.net/forum?id=Te1aZ2myPIu}.

\bibitem[Weng et~al.(2018)Weng, Zhang, Chen, Song, Hsieh, Daniel, Boning, and
  Dhillon]{pmlr-v80-weng18a}
Weng, L., Zhang, H., Chen, H., Song, Z., Hsieh, C.-J., Daniel, L., Boning, D.,
  and Dhillon, I.
\newblock Towards fast computation of certified robustness for {R}e{LU}
  networks.
\newblock In Dy, J. and Krause, A. (eds.), \emph{Proceedings of the 35th
  International Conference on Machine Learning}, volume~80 of \emph{Proceedings
  of Machine Learning Research}, pp.\  5276--5285. PMLR, 10--15 Jul 2018.
\newblock URL \url{http://proceedings.mlr.press/v80/weng18a.html}.

\bibitem[Wong \& Kolter(2018)Wong and Kolter]{wong2018provable}
Wong, E. and Kolter, Z.
\newblock Provable defenses against adversarial examples via the convex outer
  adversarial polytope.
\newblock In \emph{International Conference on Machine Learning}, pp.\
  5286--5295. PMLR, 2018.

\bibitem[Wu et~al.(2021)Wu, Bojchevski, Kuvshinov, and
  G{\"u}nnemann]{wu2021completing}
Wu, Y., Bojchevski, A., Kuvshinov, A., and G{\"u}nnemann, S.
\newblock Completing the picture: Randomized smoothing suffers from the curse
  of dimensionality for a large family of distributions.
\newblock In \emph{International Conference on Artificial Intelligence and
  Statistics}, pp.\  3763--3771. PMLR, 2021.

\bibitem[Yang et~al.(2020)Yang, Duan, Hu, Salman, Razenshteyn, and
  Li]{yang2020randomized}
Yang, G., Duan, T., Hu, J.~E., Salman, H., Razenshteyn, I., and Li, J.
\newblock Randomized smoothing of all shapes and sizes.
\newblock In \emph{International Conference on Machine Learning}, pp.\
  10693--10705. PMLR, 2020.

\bibitem[Zhai et~al.(2020)Zhai, Dan, He, Zhang, Gong, Ravikumar, Hsieh, and
  Wang]{zhai2020macer}
Zhai, R., Dan, C., He, D., Zhang, H., Gong, B., Ravikumar, P., Hsieh, C.-J.,
  and Wang, L.
\newblock Macer: Attack-free and scalable robust training via maximizing
  certified radius.
\newblock \emph{arXiv preprint arXiv:2001.02378}, 2020.

\end{thebibliography}
\bibliographystyle{icml2022}

%%%%%%%%%%%%%%%%%%%%%%%%%%%%%%%%%%%%%%%%%%%%%%%%%%%%%%%%%%%%%%%%%%%%%%%%%%%%%%%
%%%%%%%%%%%%%%%%%%%%%%%%%%%%%%%%%%%%%%%%%%%%%%%%%%%%%%%%%%%%%%%%%%%%%%%%%%%%%%%
% APPENDIX
%%%%%%%%%%%%%%%%%%%%%%%%%%%%%%%%%%%%%%%%%%%%%%%%%%%%%%%%%%%%%%%%%%%%%%%%%%%%%%%
%%%%%%%%%%%%%%%%%%%%%%%%%%%%%%%%%%%%%%%%%%%%%%%%%%%%%%%%%%%%%%%%%%%%%%%%%%%%%%%
\newpage
\appendix
\onecolumn
%\iffalse
\section{The Issues of Constant $\sigma$ Smoothing} \label{appA: the motivation}
\subsection{Toy Example} \label{ssec: toy example}
To better demonstrate our ideas, we prepared a two-dimensional simple toy dataset. This dataset can be seen in Figure~\ref{toy dataset}. The dataset is generated in polar coordinates, having uniform angle and the distance distributed as a square root of suitable chi-square distribution. The classes are positioned in a circle sectors, one in a sector with a very sharp angle. The number of training samples is 500 for each class, number of test samples is 100 for each class (except demonstrative figures, where we increased it to 300). The model that was trained on this dataset was a simple fully connected three-layer neural network with ReLU activations and a maximal width of 20. 

\begin{figure}[b]
    \centering
    \begin{minipage}[b]{0.45\linewidth}
        \includegraphics[width=\textwidth]{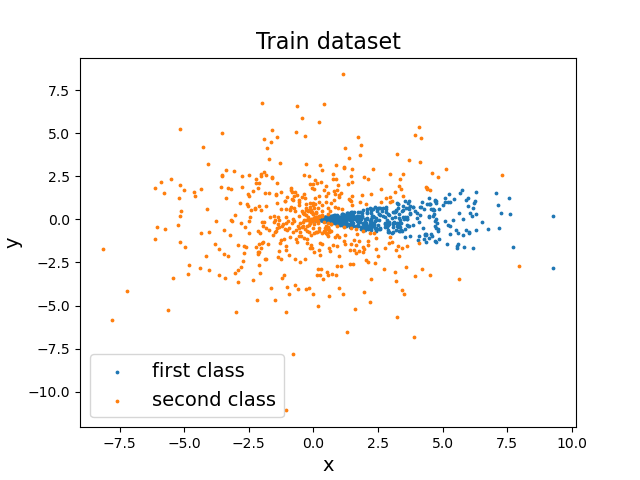}
    \end{minipage}
    \hspace{0.5cm}
    \begin{minipage}[b]{0.45\linewidth}
        \includegraphics[width=\textwidth]{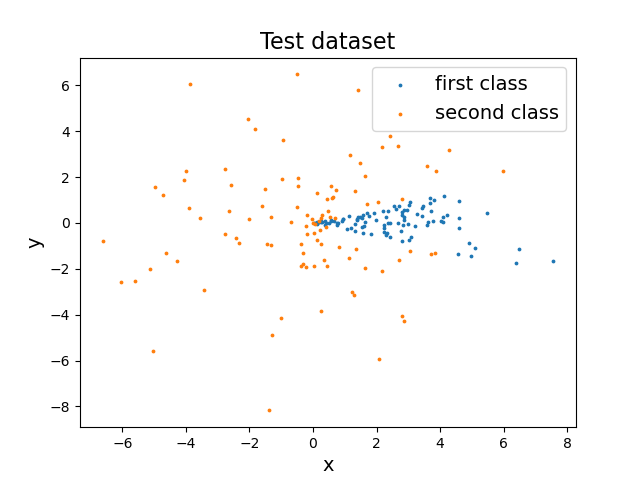}
    \end{minipage}
    \caption{The toy dataset.}
    \label{toy dataset}
\end{figure} 

\subsection{Undercertification Caused by the Use of Lower Confidence Bounds} \label{sec:undercertification}

As we mention in Section~\ref{intro}, one can not usually obtain exact values of $p_A$ and $p_B$. However, it is obvious, that for vast majority of evaluated samples, $\underline{p_A}<p_A$ and $\overline{p_B}>p_B$. Given the nature of our certified radius, it follows that $\underline{R}<R$, where $\underline{R}$ denotes the certified radius coming from the certification procedure with $\underline{p_A}$ and $\overline{p_B}$, while $R$ here stands for the certified radius corresponding to true values $p_A, p_B$. 

It is, therefore, clear, that we face a certain level of under-certification. But how serious under-certification it is? Assume the case with a linear base classifier. Imagine, that we move the point $x$ further and further away from the decision boundary. Therefore, $p_A \xrightarrow[]{} 1$. At some point, the probability will be so large, that with high probability, all $n$ samplings in our evaluation of $\underline{p_A}$ will be classified as $A$, obtaining $\hat p_A=1$ - the empirical probability. The lower confidence bound $\underline{p_A}$ is therefore bounded by having $\hat p_A=1$. Thus, from some point, the certification will yield the same $\underline{p_A}$ regardless of the true value of $p_A$. So in practice, we have an upper bound on the certified radius $R$ in the case of the linear boundary. In Figure~\ref{the lcb problems} (left), we see the truncation effect. Using $\sigma=1$, from a distance of roughly 4, we can no longer achieve a better certified radius, despite its theoretical value equals the distance. Similarly, if we fix a distance of $x$ from decision boundary and vary $\sigma$, for very small values of $\sigma$, the value of $\Phi^{-1}(\underline{p_A})$ will no longer increase, but the values of $\sigma$ will pull $R$ towards zero. This behaviour is depicted in Figure~\ref{the lcb problems} (right). 

We can also look at it differently - what is the ratio between $\Phi^{-1}(p_A)$ and $\Phi^{-1}(\underline{p_A})$ for different values of $p_A$? Since $\underline{R}=\sigma \Phi^{-1}(\underline{p_A})$ and $R=\sigma \Phi^{-1}(p_A)$, the ratio represents the ``undercertification rate''. In Figure~\ref{the radius ratio} we plot $\frac{\Phi^{-1}(\underline{p_A})}{\Phi^{-1}(p_A)}$ as a function of $p_A$ for two different ranges of values. The situation is worst for very small and very big values of $p_A$. In the case of very big values, this can be explained due to extreme nature of $\Phi^{-1}$. For small values of $p_A$, it can be explained as a consequence of a fact, that even small difference between $p_A$ and $\underline{p_A}$ will yield big ratio between $\Phi^{-1}(p_A)$ and $\Phi^{-1}(\underline{p_A})$ due to the fact, that these values are close to 0. 

\begin{figure}[t!]
    \centering
    \begin{minipage}[b]{0.45\linewidth}
        \includegraphics[width=\textwidth]{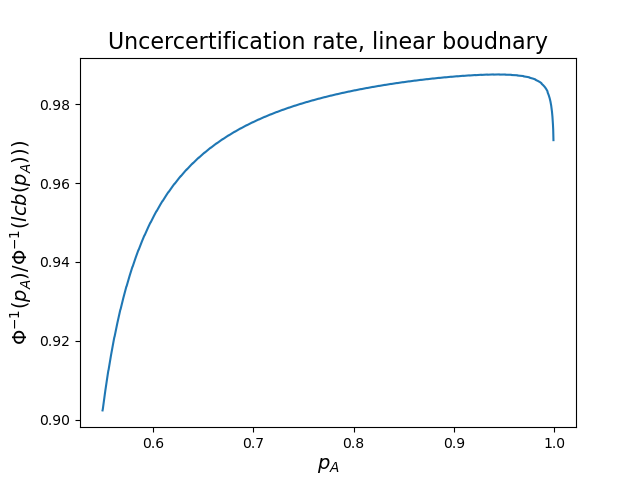}
    \end{minipage}
    \hspace{0.5cm}
    \begin{minipage}[b]{0.45\linewidth}
        \includegraphics[width=\textwidth]{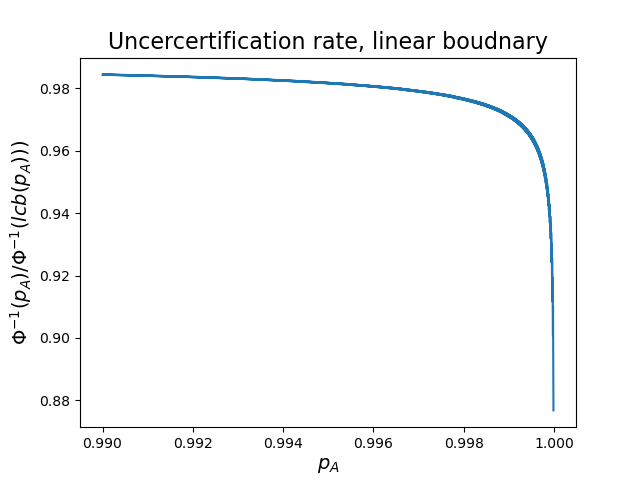}
    \end{minipage}
    \caption{The ratio between certified radius if using lower confidence bounds and if using exact values for the case of linear boundary.}
    \label{the radius ratio}
\end{figure} 

\begin{figure}[t!]
    \centering
    \begin{minipage}[b]{0.45\linewidth}
        \includegraphics[width=\textwidth]{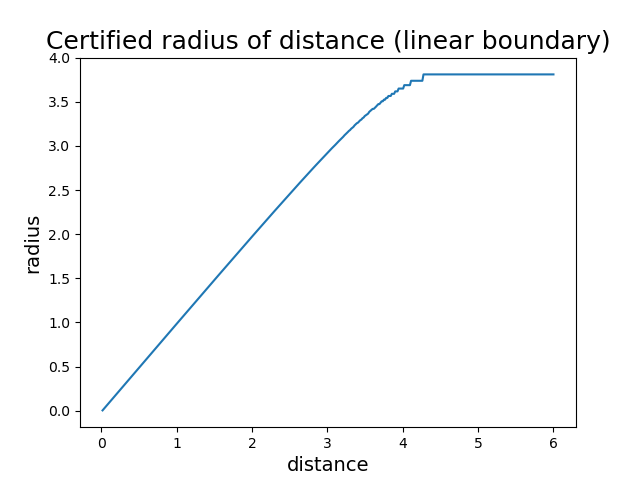}
    \end{minipage}
    \hspace{0.5cm}
    \begin{minipage}[b]{0.45\linewidth}
        \includegraphics[width=\textwidth]{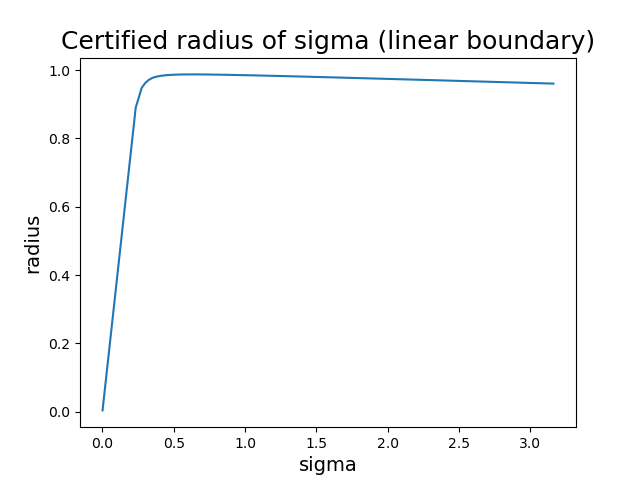}
    \end{minipage}
    \caption{\textbf{Left:} Certified radius as a function of distance in linear boundary case. The truncation is due to the use of lower confidence bounds. The parameters are $n=100000, \alpha=0.001, \sigma=1$. \textbf{Right:} Certified radius for a point $x$ at fixed distance 1 from linear boundary as a function of used $\sigma$. The undercertification follows from usage of lower confidence bounds.}
    \label{the lcb problems}
\end{figure} 

If we look at the left plot on Figure~\ref{toy cohen results} we see, that the certified accuracy plots also possess the truncations. Above some radius, no sample is certified anymore. The problem is obviously more serious for small values of $\sigma$. On the right plot of Figure~\ref{toy cohen results}, we see, that samples far from the decision boundary are obviously under-certified. We can also see, that certified radiuses remain constant, even though in reality they would increase with increasing distance from the decision boundary. 

\begin{figure}[t!]
    \centering
    \begin{minipage}[b]{0.47\linewidth}
        \includegraphics[width=\textwidth]{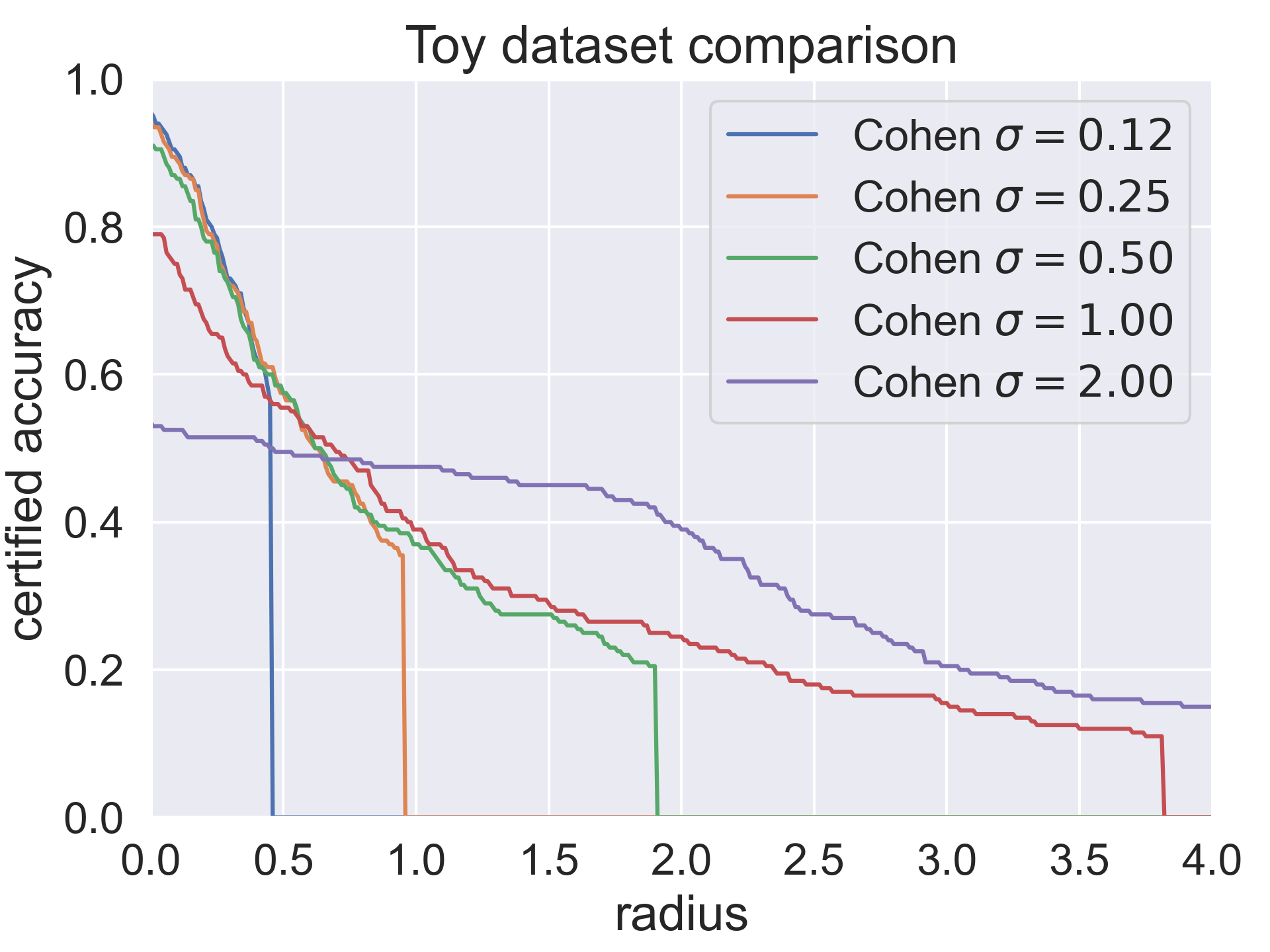}
    \end{minipage}
    \hspace{0.5cm}
    \begin{minipage}[b]{0.47\linewidth}
        \includegraphics[width=\textwidth, height=0.78\textwidth]{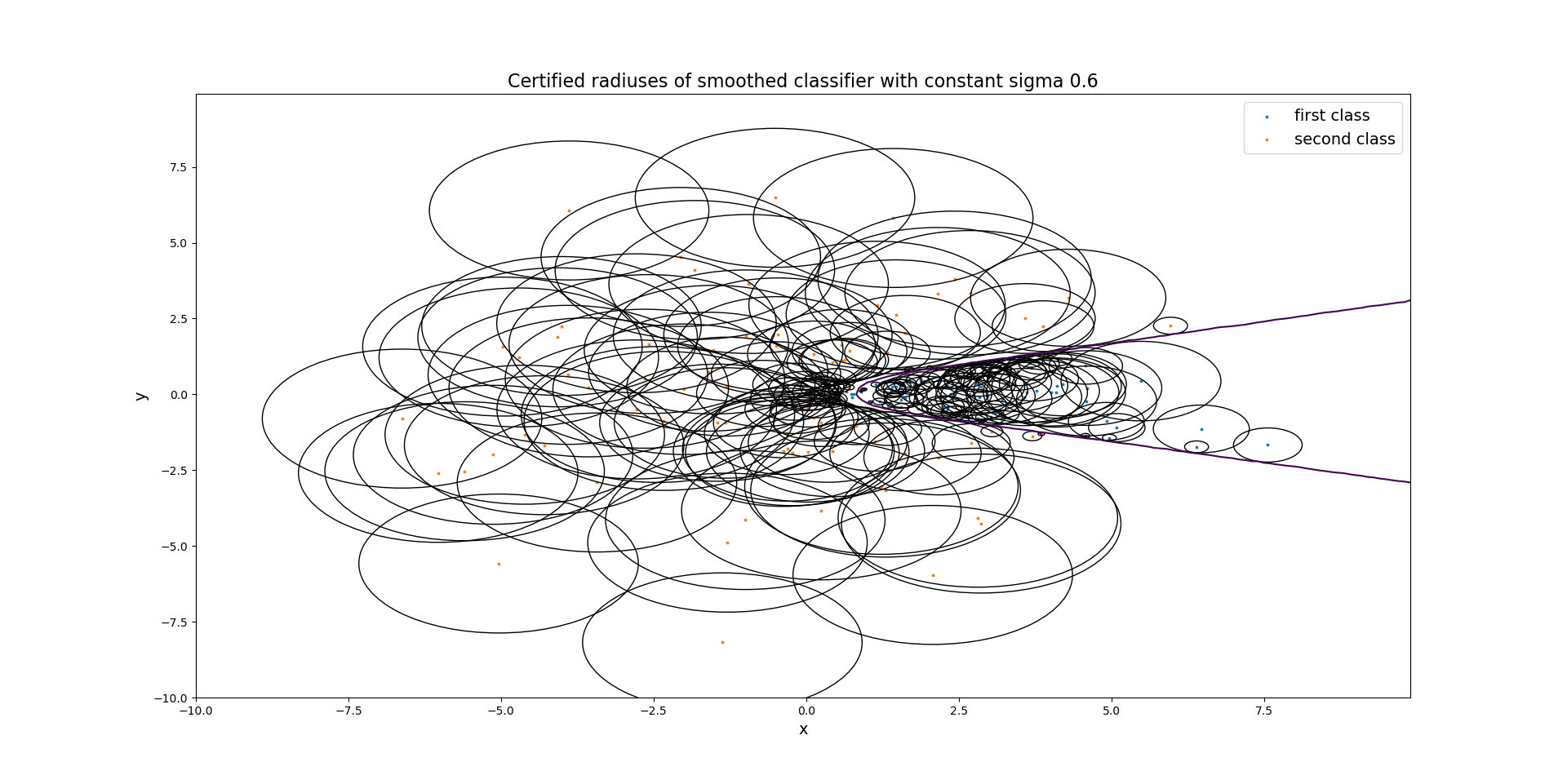}
    \end{minipage}
    \caption{Results of certification on toy dataset. \textbf{Left:} Certified accuracy for different levels of $\sigma$. \textbf{Right:} Certified radiuses and decision boundary of $g$ visualized directly on test set.}
    \label{toy cohen results}
\end{figure} 

All the observations so far motivate us to use rather large values of $\sigma$ in order to avoid the truncation problem. However, as we will see in the next sections, using a large $\sigma$ carries a different, yet equally serious burden. 

\subsection{Robustness vs.\ Accuracy Trade-off} \label{sec: robustness vs. accuracy}

As we demonstrate in the previous subsection, it is be useful to use large values of $\sigma$ to prevent the under-certification. But does it come without a prize? If we have a closer look at Figure~\ref{toy cohen results} (right), we might notice, that the accuracy on the threshold $0$, i.e. ``clean accuracy'', decreases as $\sigma$ increases. This effect has been noticed in the literature \citep{cohen2019certified, gao2020analyzing, mohapatra2020rethinking} and is called \textit{robustness vs.\ accuracy tradeoff}.

There are several reasons, why this problem occurs. Generally, changing $\sigma$ changes the decision boundary of $g$ and we might assume, that due to the high complexity of the boundary of $f$, the decision boundary of $g$ becomes smoother. If $\sigma$ is too large, however, the decision boundary will be so smooth, that it might lose some amount of the base classifier's expressivity. Another reason for the accuracy drop is also the increase in the number of samples, for which the evaluation is abstained. This is because using big values of $\sigma$ makes more classes ``within the reach of our distribution'', making the $p_A$ and $\underline{p_A}$ small. If $\underline{p_A} < 0.5$ and we do not estimate $p_B$ but set $\overline{p}_B = 1 - \underline{p_A}$, then we are not able to classify the sample as class $A$, yet we cannot classify it as a different class either, which forces us to abstain. To demonstrate these results, we computed not only the clean accuracies of \cite{cohen2019certified} evaluations but also the abstention rates. Results are depicted in the Table~\ref{cohen acc, abstention table}. 

\begin{table}[t!]
\centering
\begin{tabular}{||c||c|c|c||} 
\hline\hline
 & Accuracy & Abstention rate & Misclassification rate \\ 
\hline\hline
$\sigma=0.12$ & 0.814 & 0.038 & 0.148 \\ 
\hline
$\sigma=0.25$ & 0.748 & 0.086 & 0.166 \\
\hline
$\sigma=0.50$ & 0.652 & 0.166 & 0.182 \\
\hline
$\sigma=1.00$ & 0.472 & 0.29 & 0.238 \\
\hline\hline
\end{tabular}
\vspace{2mm}
\caption{Accuracies, rates of abstentions and misclassification rates of \cite{cohen2019certified} evaluation for different levels of $\sigma$.}
\label{cohen acc, abstention table}
\end{table}

From the table, it is obvious, that the abstention rate is possibly even bigger cause of accuracy drop than the ``clean misclassification''. This problem can be partially solved if one estimated $\overline{p_B}$ together with $\underline{p_A}$ too. In this way, using big $\sigma$ yields generally small estimated class probabilities, but since $p_A \ge p_B$, the problematic $\overline{p_B} \ge \underline{p_A}$ occur just very rarely. Another option is to increase the number of Monte-Carlo samplings for the classification decision, what is almost for free. 

Yet another reason for the decrease in the accuracy is the so-called \textit{shrinking phenomenon}, which we will discuss in the next subsection. 

In contrast with the truncation effect, the robustness vs.\ accuracy trade-off motivates the usage of smaller values of $\sigma$ in order to prevent the accuracy loss, which is definitely a very serious issue. 

\subsection{Shrinking Phenomenon} \label{sec: shrinking phenomenon}

How exactly does the decision boundary of $g$ change, as we change the $\sigma$? For instance, if $f$ is a linear classifier, then the boundary does not change at all. To answer this question, we employ the following experiment: For our toy base classifier $f$ on our toy dataset, we increase $\sigma$ and plot the heatmap of $f, g$, together with its decision boundary. This experiment is depicted on Figure~\ref{heatmaps}. As we see from the plots, increasing $\sigma$ causes several effects. First of all, the heatmap becomes more and more blurred, what proves, that stronger smoothing implies stronger smoothness.

Second, crucial, effect is that the bigger the $\sigma$, the smaller the decision boundary of a submissive class is. The shrinkage becomes pronounced from $\sigma=1$. Already for $\sigma=4$, there is hardly any decision boundary anymore. Generally, as $\sigma \xrightarrow[]{} \infty$, $g$ will predict the class with the biggest volume in the input space (in the case of bounded input space, like in image domain, this is very well defined). For extreme values of sigma, the $p_A$ will practically just be the ratio between the volume of $A$ and the actual volume of the input space (for bounded input spaces). 

\begin{figure}[t!]
    \centering
    \begin{minipage}[b]{0.32\linewidth}
        \includegraphics[width=\textwidth]{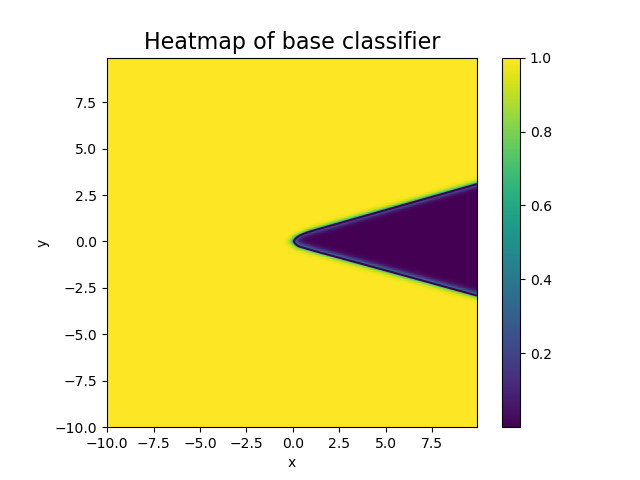}
    \end{minipage}
    \begin{minipage}[b]{0.32\linewidth}
        \includegraphics[width=\textwidth]{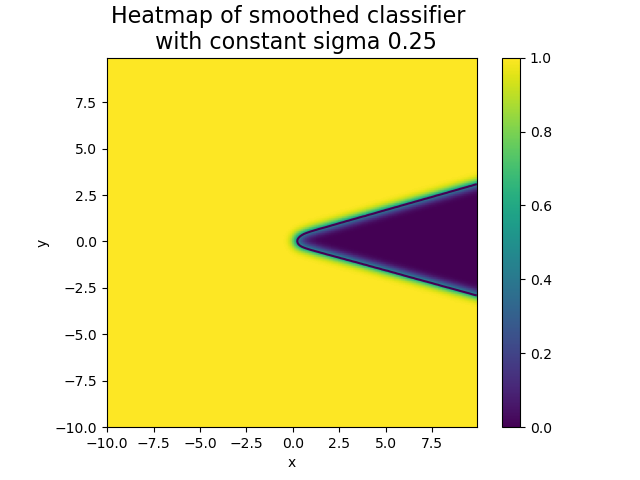}
    \end{minipage}
    \begin{minipage}[b]{0.32\linewidth}
        \includegraphics[width=\textwidth]{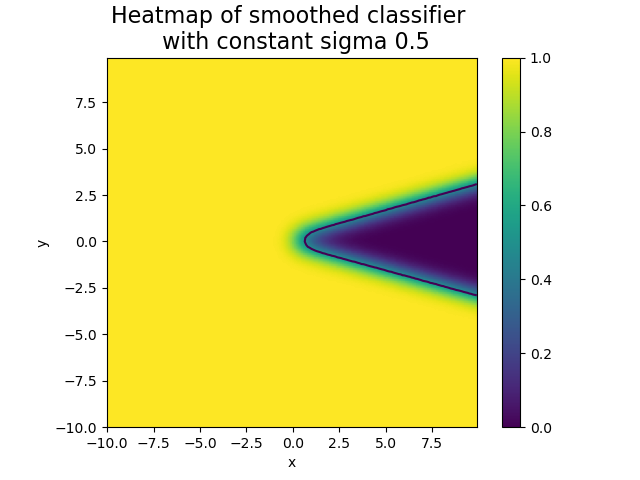}
    \end{minipage}
    \begin{minipage}[b]{0.32\linewidth}
        \includegraphics[width=\textwidth]{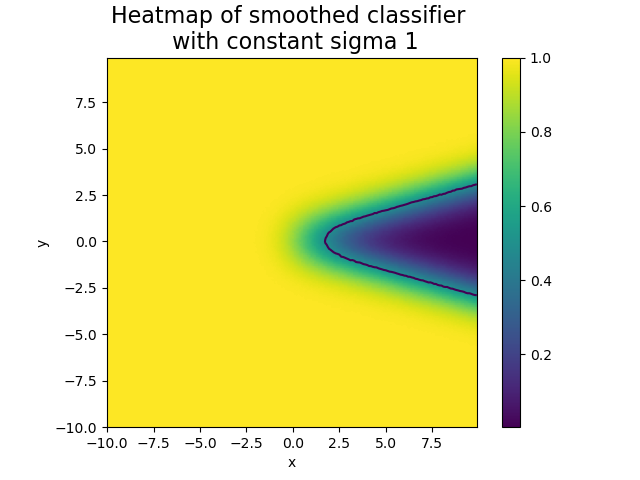}
    \end{minipage}
    \begin{minipage}[b]{0.32\linewidth}
        \includegraphics[width=\textwidth]{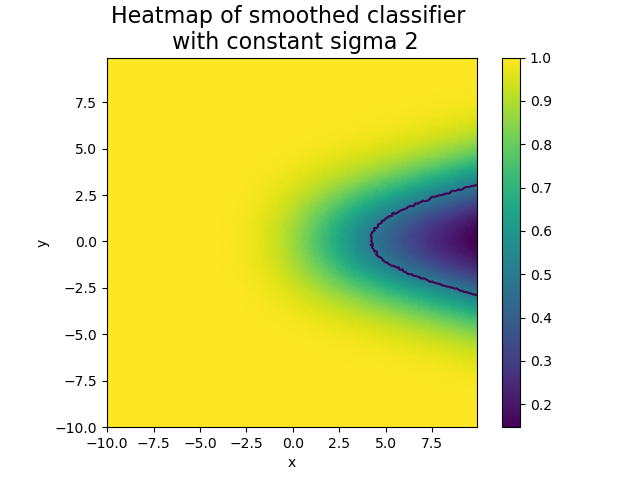}
    \end{minipage}
    \begin{minipage}[b]{0.32\linewidth}
        \includegraphics[width=\textwidth]{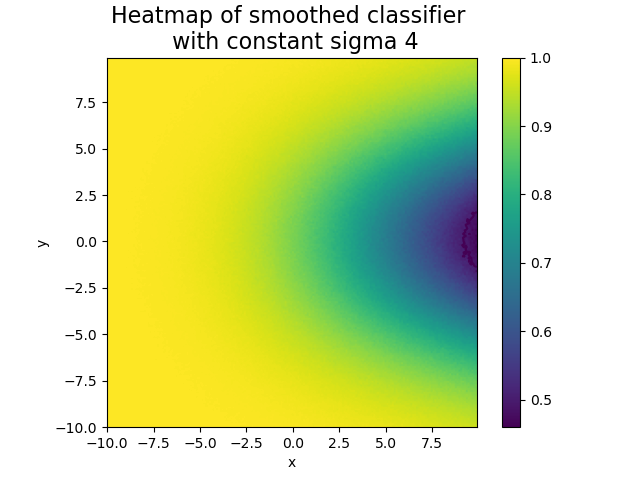}
    \end{minipage}
    \caption{Heatmaps and decision boudnary of base classifier (top left) and the smoothed classifier for increasing levels of $\sigma$. As $\sigma$ increases, the classifier is more smooth and the decision boundary recedes.}
    \label{heatmaps}
\end{figure} 

Following from these results, but also from basic intuition, it seems, that an undesired effect becomes present as $\sigma$ increases - the bounded/convex regions become to shrink, like in Figure~\ref{heatmaps}, while the unbounded/big/anti-convex regions expand. This is called \textit{shrinking phenomenon}. \cite{mohapatra2020rethinking} investigate this effect rather closely. They define the shrinkage and vanishing of regions formally and prove rigorously, that if $\sigma \xrightarrow[]{} \infty$, bounded regions, or semi-bounded regions (see \cite{mohapatra2020rethinking}) will eventually vanish. We formulate the main result in this direction. 

\begin{theorem} \label{mohapatra main}
Let us have $K$ the number of classes and the dimension $N$. Assume, that we have some bounded decision region $\mathcal{D}$ for a specific class roughly centered around 0. Further assume, that this is the only region where the class is classified. Let $R$ be a smallest radius such that $\mathcal{D} \subset B_R(0)$. Then, this decision region will vanish at most for $\sigma \ge \frac{R\sqrt{K}}{\sqrt{N}}$. 
\end{theorem}
\begin{proof}
The idea of the proof is not very hard. First, the authors prove, that the smoothed region will be a subset of the smoothed $B_R(0)$. Then, they upper-bound the vanishing threshold of such a ball in two steps. First, they show, that if 0 is not classified as the class, then no other point will be (this is quite an intuitive geometrical statement. The $B_R(0)$ has the biggest probability under $\mathcal{N}(x, \sigma^2 I)$ if $x \equiv 0$). Second, they upper-bound the threshold for $\sigma$, under which $B_R(0)$ will have probability below $\frac{1}{K}$ (since they use slightly different setting as \cite{cohen2019certified}) for the $\mathcal{N}(x, \sigma^2 I)$. Using some insights about incomplete gamma function, which is known to be also the cdf of central chi-square distribution, and some other integration tricks, they obtain the resulting bound. 
\end{proof}

Besides Theorem~\ref{mohapatra main}, authors also claim many other statements abound shrinking, including shrinking of semi-bounded regions. Moreover, they also conduct experiments on CIFAR10 and ImageNet to support their theoretical findings. They also point out serious fairness issue that comes out as a consequence of the shrinkage phenomenon. For increasing levels of $\sigma$, they measure the class-wise clean accuracy of the smoothed classifier. If $f$ is trained with Gaussian data augmentation (what is known to be a good practice in randomized smoothing), using $\sigma=0.12$, the worst class \textit{cat} has test accuracy of 67\%, while the best class \textit{automobile} attains the accuracy of 92\%. The figures, however, change drastically, if we use $\sigma=1$ instead. In this case, the worst predicted class \textit{cat} has accuracy of poor 22\%, while \textit{ship} has reasonable accuracy 68\%. As authors claim, this is a consequence of the fact that samples of \textit{cat} are situated more in bounded, convex regions, that suffer from shrinking, while samples of \textit{ship} are mostly placed in expanded regions of anti-convex shape that will expand as the $\sigma$ grows. In addition, the authors also show, that the Gaussian data augmentation or adversarial training will reduce the shrinking phenomenon just partially and for moderate and high values of $\sigma$, this effect will be present anyway. 

We must emphasize, that this is a serious fairness issue, that has to be treated before randomized smoothing can be fully used in practice. For instance, if we trained a neural network to classify humans into several categories, fairness of classification is inevitable and the neural network cannot be used until this issue is solved. 

Similarly as the robustness vs.\ accuracy trade-off, this issue also motivates to use rather smaller values of $\sigma$. We see, that it is not possible to address all three problems consistently because they disagree on whether to use smaller, or bigger values of $\sigma$.

\subsection{Experiments on High-dimensional toy Dataset} \label{AppA: experiments on multitoy dataset}

In this subsection, we present the results of our motivational experiment on a synthetic dataset. Before reading this section, please read our main text, because we will use the necessary notation of the paper. 

The dataset we evaluated our method on is a generalization of the dataset visualized on Figure~\ref{fig: toy experiment main text}. The data points from one class lie in a cone of small angle and the points are generated such that the density is higher near the vertex of the cone (which is put in origin). Points from other class are generated from a spherically symmetrical distribution (where points sampled into the cone are excluded) with density again highest in the center (note, that the density peak is more pronounced than in the case of normal distribution, where the density around the center resembles uniform distribution). This dataset is chosen so that the $\sigma(x)$ function designed in Equation~\ref{eq: the sigma fcn} well corresponds to the geometry of the decision boundary. Moreover it is chosen so that the conic decision region will shrink rather fast with increasing $\sigma$. The motivation of this example is to show that if the $\sigma(x)$ function is well-designed, our IDRS can outperform the constant RS considerably. 

The setup of our experiment is as follows: We evaluate dimensions $N= 2, 6, 18, 60, 180, 400$. The $\sigma$ used for constant smoothing is $\sigma = 0.5, 0.5, 1.0, 1.0, 2.0, 2.0$ respectively. The $\sigma_b$ used is $0.4, 0.5$ for $N=2$, $0.4$ for $N=6$, $0.8$ for $N=18$, $0.8, 1.0$ for $N=60$, $1.9$ for $N=180$ and $1.95$ for $N=400$. The rates are $r = 0.2, 0.1, 0.05, 0.03, 0.01, 0.005$ respectively. The training was executed without data augmentation (because samples from different classes are very close to each other). Moreover, we have set maximal $\sigma(x)$ threshold for numerical purposes, because some samples were outliers and were way too far from other samples (and if the $\sigma(x)$ is way too big, the method encounteres numerical problems). In this case we set $\sigma(x) \le 5 \sigma_b$, but we are aware that also much bigger thresholds would have been possible. We present our comparisons in Figure~\ref{fig: multi toy experiment} and Table~\ref{tab: multi toy experiment}. 

\begin{figure}[t!]
    \centering
    \begin{minipage}[b]{0.40\linewidth}
        \includegraphics[width=\textwidth]{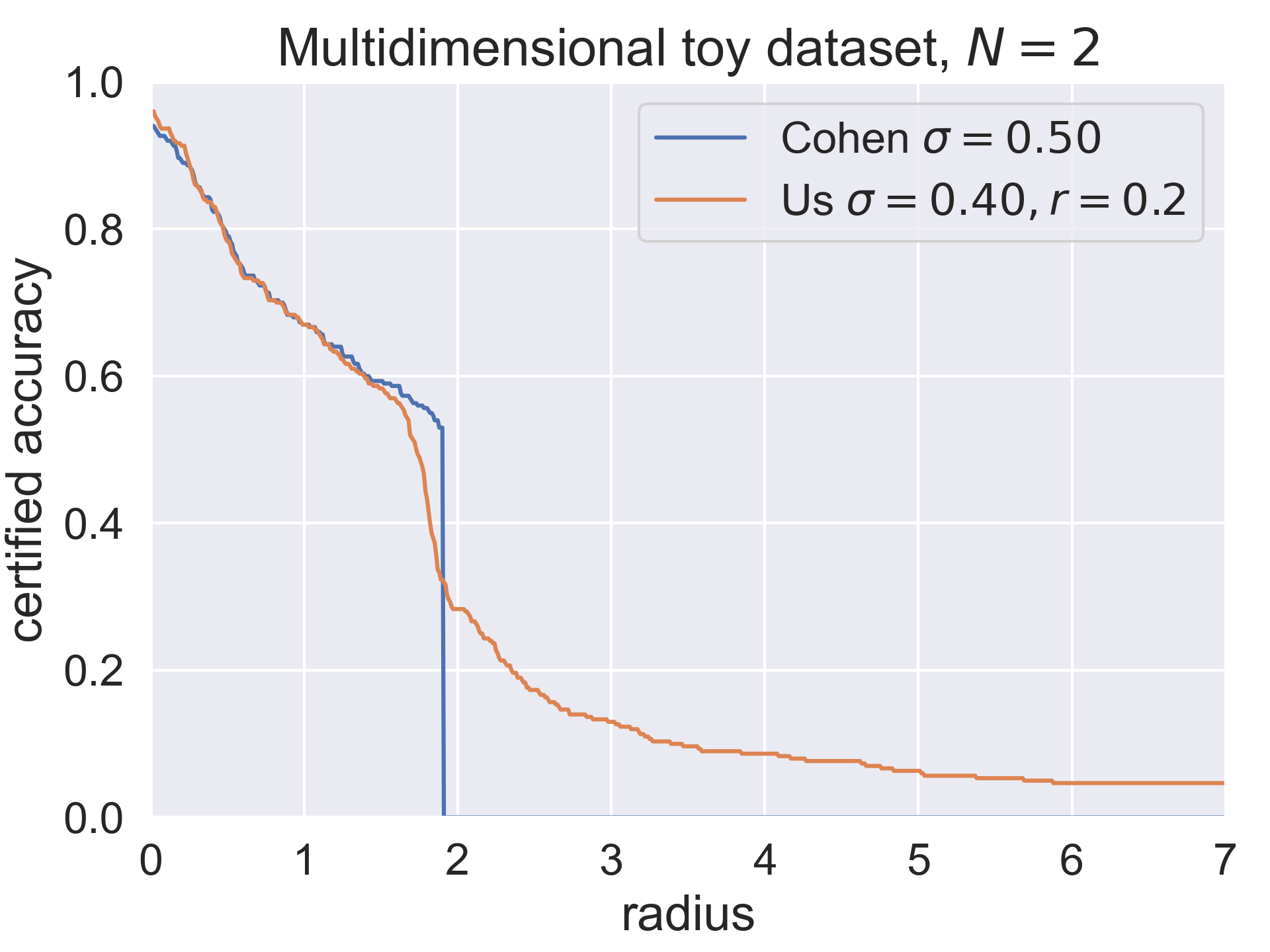}
    \end{minipage}
    \begin{minipage}[b]{0.40\linewidth}
        \includegraphics[width=\textwidth]{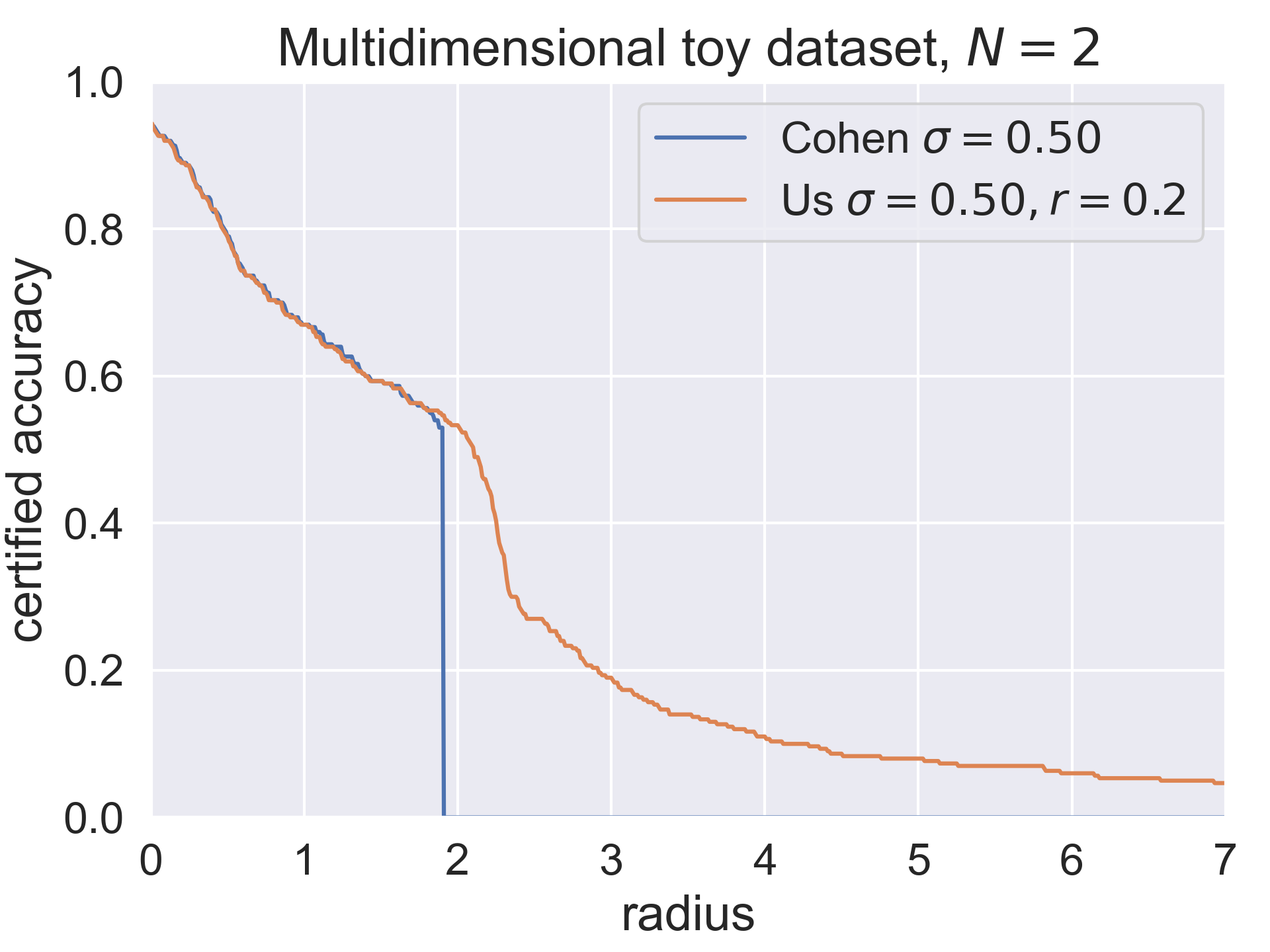}
    \end{minipage}
    \begin{minipage}[b]{0.40\linewidth}
        \includegraphics[width=\textwidth]{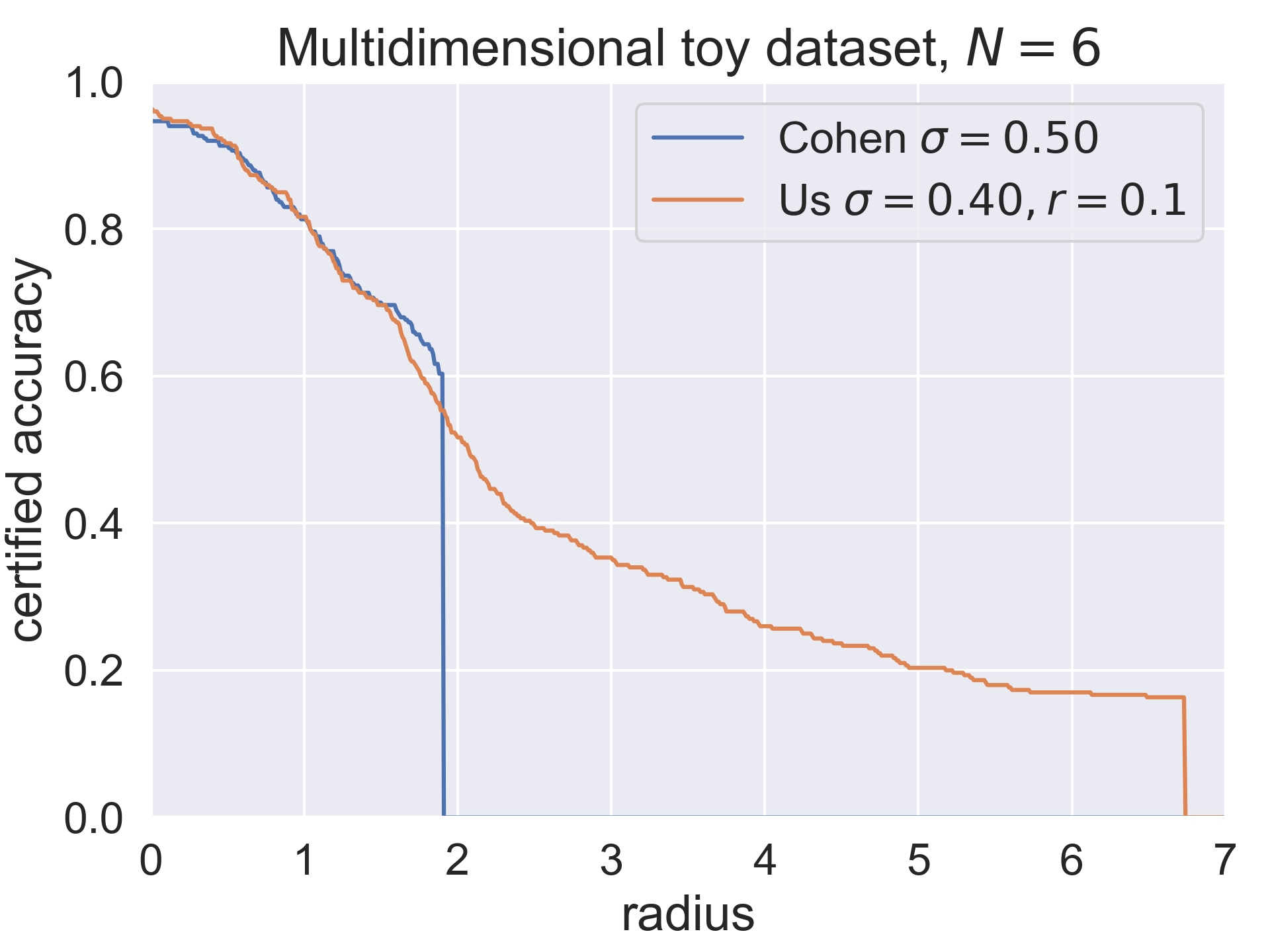}
    \end{minipage}
    \begin{minipage}[b]{0.40\linewidth}
        \includegraphics[width=\textwidth]{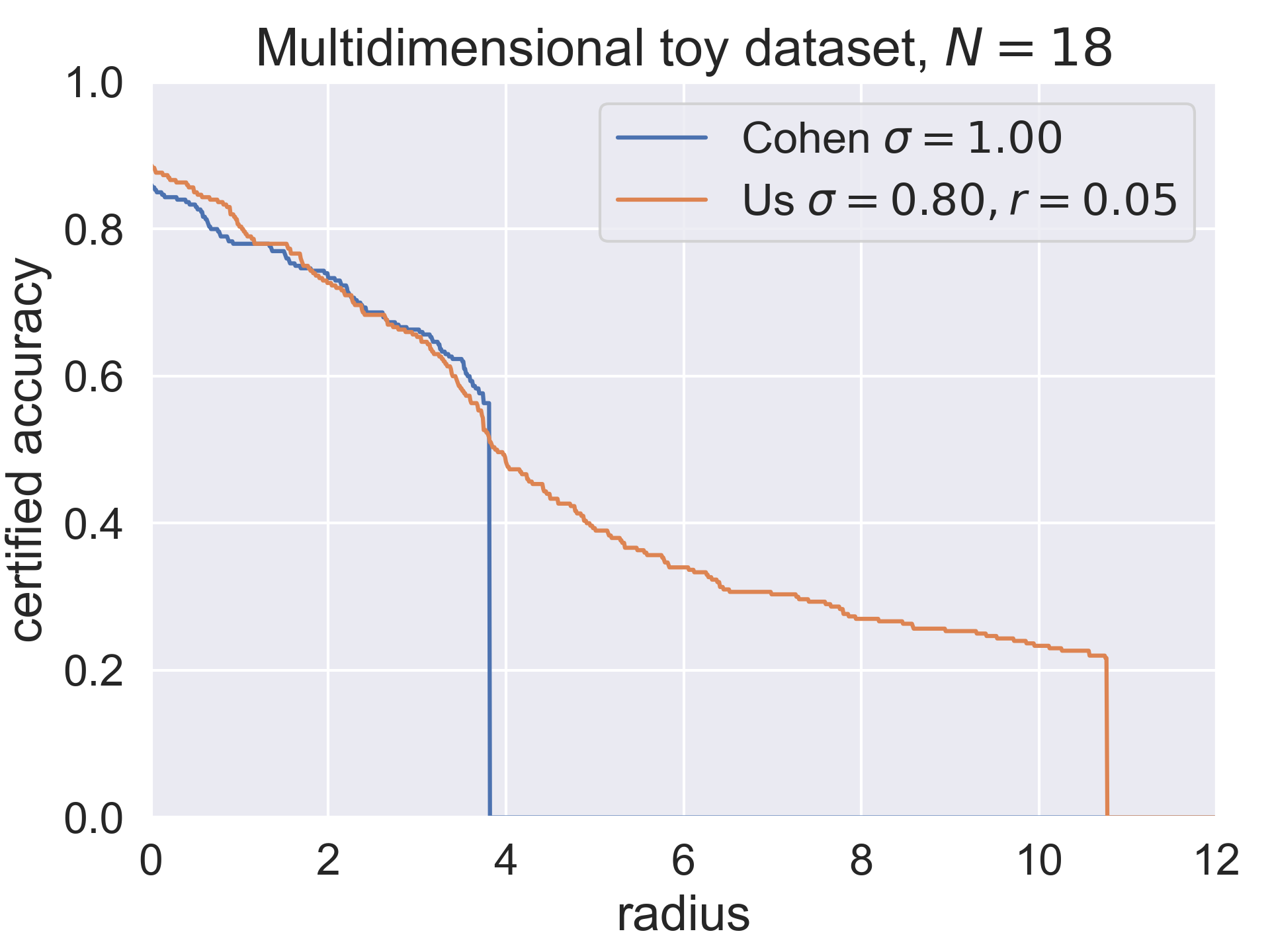}
    \end{minipage}
    \begin{minipage}[b]{0.40\linewidth}
        \includegraphics[width=\textwidth]{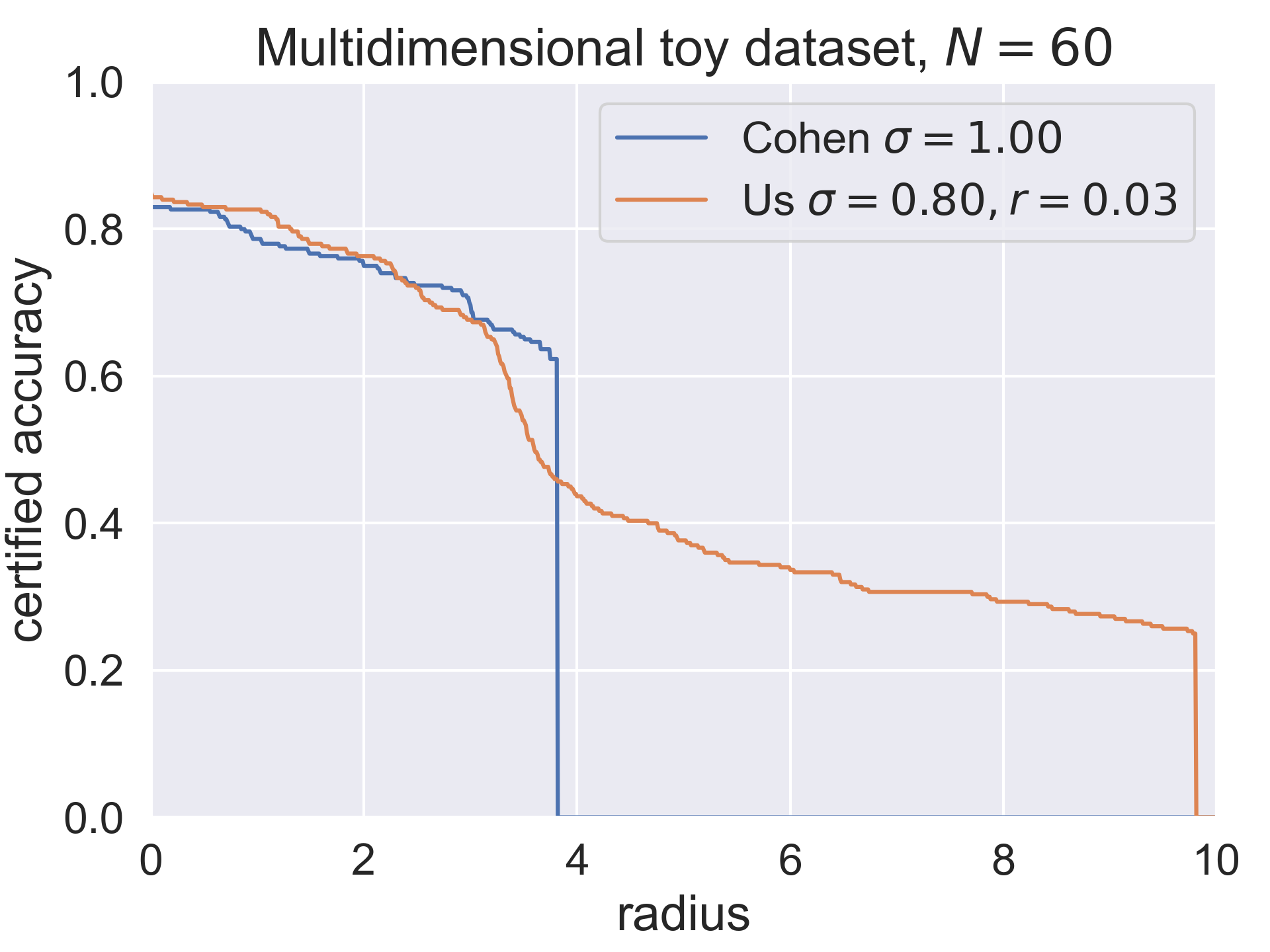}
    \end{minipage}
    \begin{minipage}[b]{0.40\linewidth}
        \includegraphics[width=\textwidth]{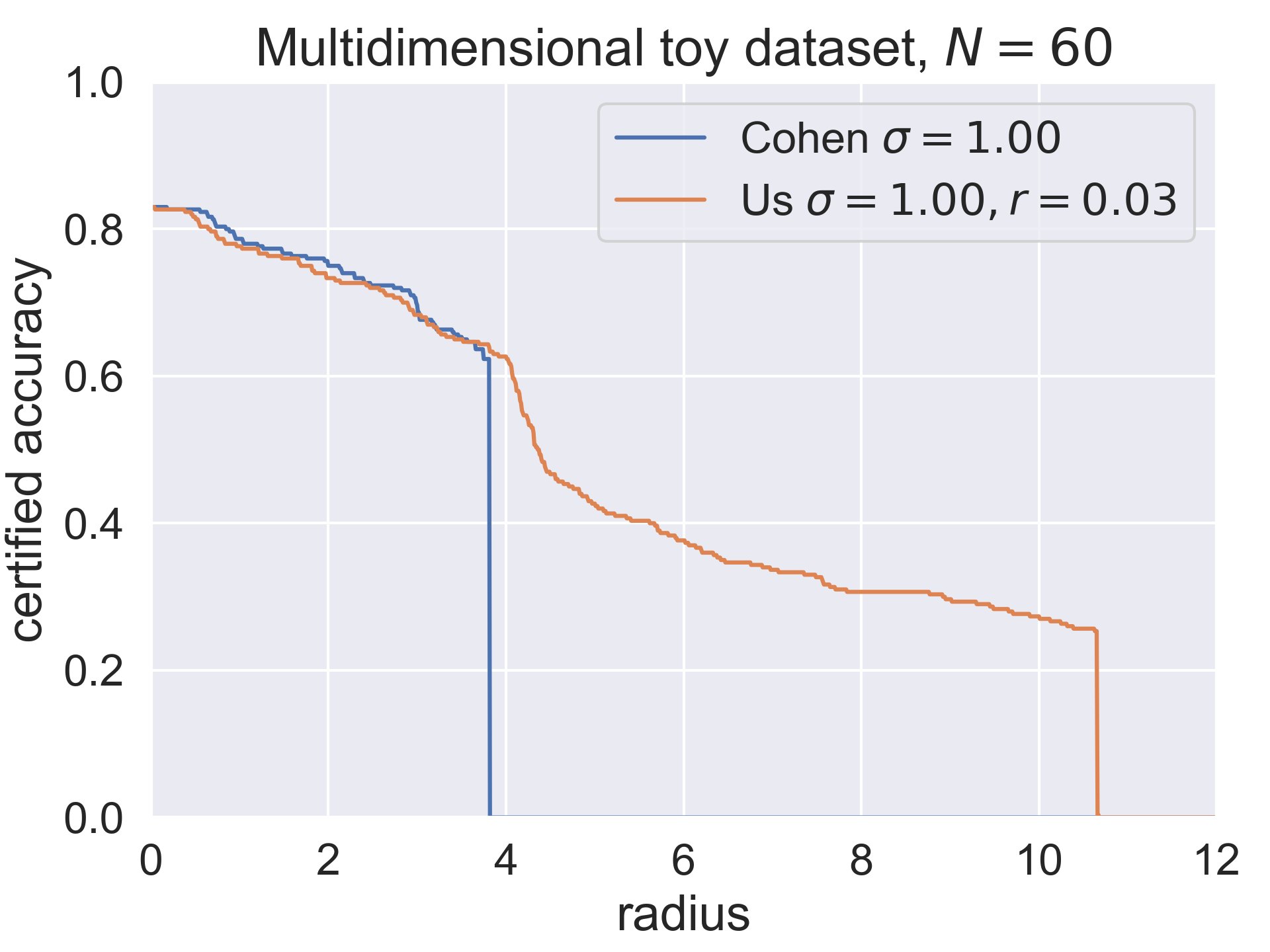}
    \end{minipage}
    \begin{minipage}[b]{0.40\linewidth}
        \includegraphics[width=\textwidth]{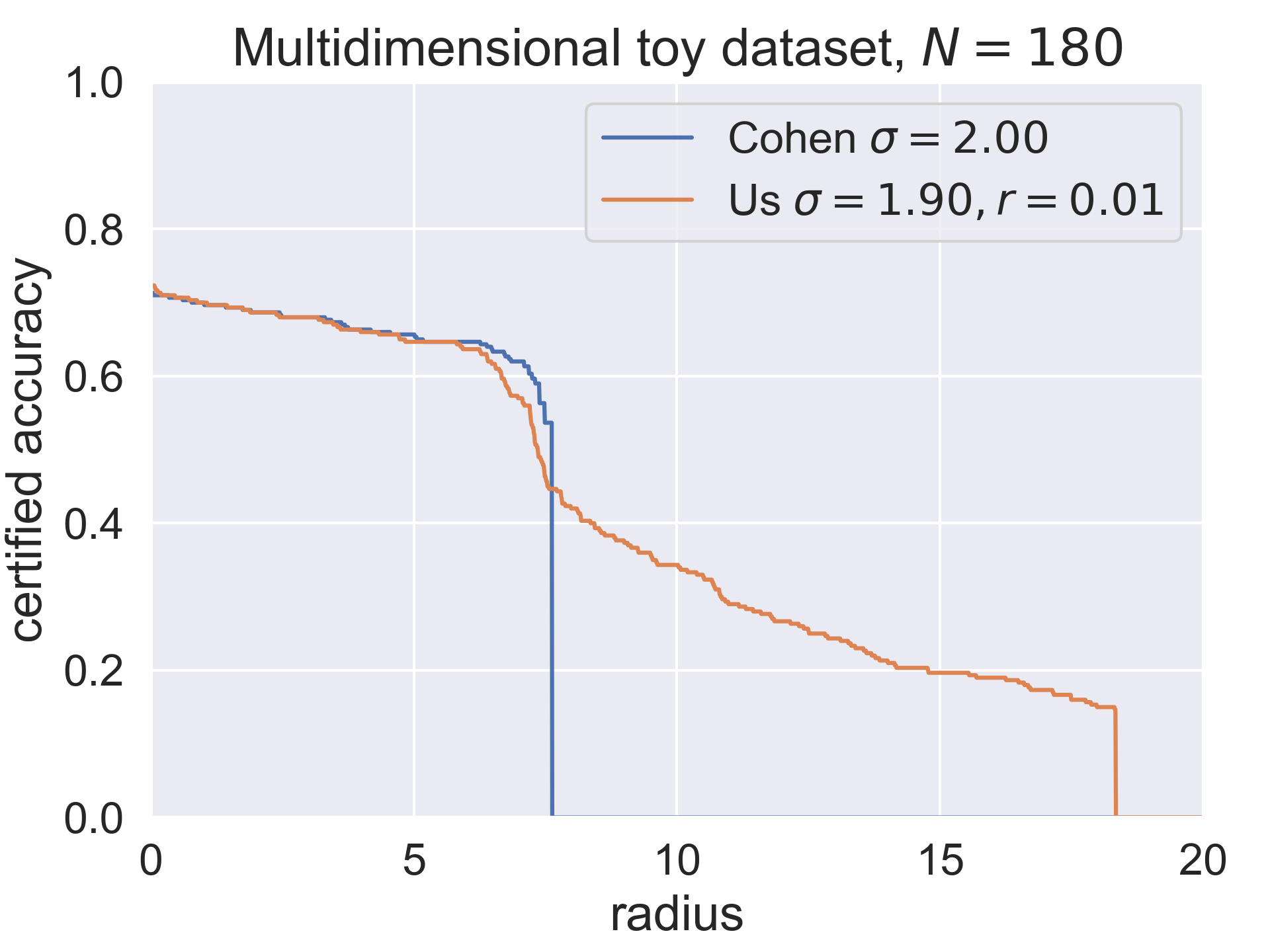}
    \end{minipage}
    \begin{minipage}[b]{0.40\linewidth}
        \includegraphics[width=\textwidth]{figures/comparison_dim_180_cohen_sigma_2_00_us_sigma_1_90_rate_0_01.png}
    \end{minipage}
    \caption{Certified accuracy plots of our multidimensional toy experiments.}
    \label{fig: multi toy experiment}
\end{figure} 

\begin{table}[t!]
\centering
\begin{tabular}{||c||c|c|c||} 
\hline\hline
Dimension & $\sigma$ ($\sigma_b$) & $r$ & Accuracy \\ 
\hline\hline
2 & 0.5 & - & 0.943 \\ 
\hline
2 & 0.4 & 0.2 & 0.96 \\
\hline
2 & 0.5 & 0.2 & 0.943 \\
\hline
6 & 0.5 & - & 0.946 \\
\hline
6 & 0.4 & 0.1 & 0.963 \\
\hline
18 & 1.0 & - & 0.86 \\
\hline
18 & 0.8 & 0.05 & 0.886 \\
\hline
60 & 1.0 & - & 0.83 \\
\hline
60 & 0.8 & 0.03 & 0.85 \\
\hline
60 & 1.0 & 0.03 & 0.83 \\
\hline
180 & 2.0 & - & 0.713 \\
\hline
180 & 1.9 & 0.01 & 0.726 \\
\hline
400 & 2.0 & - & 0.623 \\
\hline
400 & 1.95 & 0.005 & 0.623 \\
\hline\hline
\end{tabular}
\vspace{2mm}
\caption{Clean accuracies of different evaluations of our toy experiment. }
\label{tab: multi toy experiment}
\end{table}

From both the Figure~\ref{fig: multi toy experiment} an Table~\ref{tab: multi toy experiment} it is clear that the IDRS can outperform the constant $\sigma$ RS considerably, if we use really suitable $\sigma(x)$ function. We manage to improve significantly the certified radiuses without losing a single correct classification. On the other hand, in cases where $\sigma_b<\sigma$, we outperform constant $\sigma$ both in clean accuracy and in certified radiuses. This example is synthetic and designed in our favour. The main message is not how perfect our design of $\sigma(x)$ is, but the fact, that if $\sigma(x)$ is designed well, the IDRS can bring real advantages, even in moderate dimensions. 

\section{More on Theory} \label{appC: more on theory}
\subsection{Generalization of Results by \citet{li2018certified}} \label{appB: the li part}
In our main text, we mostly focus on the generalization of the methods from \cite{cohen2019certified}. This is because these methods yield tight radiuses and because the application of Neyman-Pearson lemma is beautiful. However, the methodology from \cite{li2018certified} can also be generalized for the input-dependent RS. To be able to do it, we need some auxiliary statements about the R{\'e}nyi divergence. 

\begin{lemma} \label{renyi non-constant}
The R{\'e}nyi divergence between two one-dimensional normal distributions is as follows: 
$$D_\alpha(\mathcal{N}(\mu_1, \sigma_1^2)||\mathcal{N}(\mu_0, \sigma_0^2))=\frac{\alpha(\mu_1-\mu_2)^2}{2\sigma_\alpha^2}+\frac{1}{1-\alpha}\log\left(\frac{\sigma_\alpha}{\sigma_1^{1-\alpha}\sigma_0^\alpha}\right),$$
provided, that $\sigma_\alpha^2:=(1-\alpha)\sigma_1^2+\alpha\sigma_0^2 \ge 0$. 
\end{lemma}
\begin{proof}
See \cite{van2014renyi}.
\end{proof}

Note, that this proposition induces some assumptions on how $\sigma_0, \sigma_1, \alpha$ should be related. If $\sigma_0>\sigma_1$, then the required inequality holds for any $1 \neq \alpha>0$. If $\sigma_0<\sigma_1$, then $\alpha$ is restricted and we need to keep that in mind.

\begin{lemma} \label{renyi product measure}
Assume, we have some one-dimensional distributions $\mathcal{P}_1, \mathcal{P}_1, \dots, \mathcal{P}_N$ and $\mathcal{Q}_1, \mathcal{Q}_2, \dots, \mathcal{Q}_N$ defined on common space for pairs with the same index. Then, assuming product space with product $\sigma$-algebra, we have the following identity: 

$$D_\alpha(\mathcal{P}_1 \times \mathcal{P}_2 \times \dots \times \mathcal{P}_N || \mathcal{Q}_1 \times \mathcal{Q}_2 \times \dots \times \mathcal{Q}_N) = \sum_{i=1}^N D_\alpha(\mathcal{P}_i || \mathcal{Q}_i).$$
\end{lemma}
\begin{proof}
See \cite{van2014renyi}.
\end{proof}

Using these two propositions, we are now able to derive a formula for R{\'e}nyi divergence between two multivariate isotropic normal distributions: 

\begin{lemma}
$$D_\alpha(\mathcal{N}(x_1, \sigma_1^2 I) || \mathcal{N}(x_0, \sigma_0^2 I)) = \frac{\alpha\norm{x_0-x_1}^2}{2\sigma_1^2+2\alpha(\sigma_0^2-\sigma_1^1)}+N\frac{\log\left(\frac{\sigma_\alpha}{\sigma_1}\right)}{1-\alpha}-N\frac{\alpha}{1-\alpha}\log\left(\frac{\sigma_0}{\sigma_1}\right).$$
\end{lemma}
\begin{proof}
Imporant property that is needed here is, that isotropic gaussian distributions factorize to one-dimensinal independent marignals. In other words: 
$$\mathcal{N}(x_1, \sigma_1^2 I)=\mathcal{N}(x_{11}, \sigma_1^2) \times \mathcal{N}(x_{12}, \sigma_1^2) \times \dots \times \mathcal{N}(x_{1N}, \sigma_1^2),$$ and analogically for $x_0$. 
Therefore, using Lemma~\ref{renyi product measure} we see: 
$$D_\alpha(\mathcal{N}(x_1, \sigma_1^2 I) || \mathcal{N}(x_0, \sigma_0^2 I))=\sum_{i=1}^N D_\alpha(\mathcal{N}(x_{1i}, \sigma_1^2) || \mathcal{N}(x_{0i}, \sigma_0^2)).$$
Now, it suffices to plug in the formula from Proposition~\ref{renyi non-constant} to obtain the required result: 
\begin{align*}
D_\alpha(\mathcal{N}(x_{1i}, \sigma_1^2 I) || \mathcal{N}(x_{0i}, \sigma_0^2 I))&=\frac{\alpha(x_{1i}-x_{2i})^2}{2\sigma_\alpha^2}+\frac{1}{1-\alpha}\log\left(\frac{\sigma_\alpha}{\sigma_1^{1-\alpha}\sigma_0^\alpha}\right) \\&=\frac{\alpha(x_{1i}-x_{2i})^2}{2\sigma_1^2+2\alpha(\sigma_0^2-\sigma_1^1)}+\frac{\log\left(\frac{\sigma_\alpha}{\sigma_1}\right)}{1-\alpha}-\frac{\alpha}{1-\alpha}\log\left(\frac{\sigma_0}{\sigma_1}\right)
\end{align*}
Now it suffices to sum up over $i$ and the result follows. 
\end{proof}

To obtain the certified radius, we also need a result from \cite{li2018certified}, which gives a guarantee that two measures on the set of classes will share the modus if the Rényi divergence between them is small enough. 

\begin{lemma} \label{li discrete dists}
Let $\mathbb{P}=(p_1, p_2, \dots, p_K)$ and $\mathbb{Q}=(q_1, q_2, \dots, q_K)$ two discrete measures on $\mathcal{C}$. Let $p_A, p_B$ correspond to two biggest probabilities in distribution $\mathbb{P}$. Let $M_1(a,b)=\frac{a+b}{2}$ and $M_{1-\alpha}(a,b)=(\frac{a^{1-\alpha}+b^{1-\alpha}}{2})^{\frac{1}{1-\alpha}}$ If $$D_\alpha(\mathbb{Q}||\mathbb{P}) \le -\log(1-2M_1(p_A, p_B)+2M_{1-\alpha}(p_A, p_B)),$$

then the distributions $\mathbb{P}$ and $\mathbb{Q}$ agree on the class with maximal assigned probability. 
\end{lemma}
\begin{proof}
This lemma can be proved by directly computing the minimal required $D_\alpha$ to be able to disagree on the maximal class probabilities via a constrained optimization problem (with variables $p_i, q_i, i \in \{1, \dots, K\}$), solving KKT conditions. For details, consult \cite{li2018certified}.
\end{proof}

Having explicit formula for the Rényi divergence, we can mimic the methodology of \cite{li2018certified} to obtain the certified radius: 

\begin{theorem}
Given $x_0, p_A, p_B, \sigma_0, N$, the certified radius squared for all $x_1$ such that fixed $\sigma_1$ is used is:
\begin{align*}
R^2=\underset{\alpha \in S_{\sigma_0, \sigma_1}}{\sup} \frac{2\sigma_1^2+2\alpha(\sigma_0^2-\sigma_1^1)}{\alpha}\Bigg(N\frac{\alpha}{1-\alpha}\log\left(\frac{\sigma_0}{\sigma_1}\right)-N\frac{\log\left(\frac{\sigma_\alpha}{\sigma_1}\right)}{1-\alpha} \\-\log(1-2M_1(p_A, p_B)+2M_{1-\alpha}(p_A, p_B))\Bigg),
\end{align*} where $S_{\sigma_0, \sigma_1}= \mathbb{R}_+$, if $\sigma_0>\sigma_1$ and $S_{\sigma_0, \sigma_1}= \left( 0, \frac{\sigma_1^2}{\sigma_1^2-\sigma_0^2} \right]$ if $\sigma_0<\sigma_1$.
\end{theorem}
\begin{proof}
Let us fix $x_1$ and assume, that $\alpha \in S_{\sigma_0, \sigma_1}$. Then, due to post-processing inequality for Renyi divergence, it follows that 
\begin{align*}
&D_\alpha(f(x_1+\mathcal{N}(0,\sigma_1^2 I))||f(x_0+\mathcal{N}(0, \sigma_0^2 I))) \le D_\alpha(x_1+\mathcal{N}(0, \sigma_1^2 I)|| x_0+\mathcal{N}(0, \sigma_0^2 I))\\ &=\frac{\alpha\norm{x_0-x_1}^2}{2\sigma_1^2+2\alpha(\sigma_0^2-\sigma_1^1)}+N\frac{\log\left(\frac{\sigma_\alpha}{\sigma_1}\right)}{1-\alpha}-N\frac{\alpha}{1-\alpha}\log\left(\frac{\sigma_0}{\sigma_1}\right).
\end{align*}
Due to Lemma~\ref{li discrete dists}, it suffices that the following inequality holds for \textit{some} $\alpha \in S_{\sigma_0, \sigma_1}$: 
\begin{align*}
&\frac{\alpha\norm{x_0-x_1}^2}{2\sigma_1^2+2\alpha(\sigma_0^2-\sigma_1^1)}+N\frac{\log\left(\frac{\sigma_\alpha}{\sigma_1}\right)}{1-\alpha}-N\frac{\alpha}{1-\alpha}\log\left(\frac{\sigma_0}{\sigma_1}\right) \le \\ &-\log(1-2M_1(p_A, p_B)+2M_{1-\alpha}(p_A, p_B)).
\end{align*}
This can be rewritten w.r.t. $\norm{x_0-x_1}^2$:
\begin{align*}
\norm{x_0-x_1}^2 \ge \frac{2\sigma_1^2+2\alpha(\sigma_0^2-\sigma_1^1)}{\alpha}\Bigg(N\frac{\alpha}{1-\alpha}\log\left(\frac{\sigma_0}{\sigma_1}\right)-N\frac{\log\left(\frac{\sigma_\alpha}{\sigma_1}\right)}{1-\alpha} \\-\log(1-2M_1(p_A, p_B)+2M_{1-\alpha}(p_A, p_B))\Bigg).    
\end{align*}
The resulting certified radius squared is now simply obtained by taking maximum over $\norm{x_0-x_1}^2$ s.t. $\exists \alpha \in S_{\sigma_0, \sigma_1}$ such that the preceding inequality holds. 
\end{proof}

Note, that this theorem is formulated assuming, that except in $x_0$, we use $\sigma_1$ everywhere. It would require some further work to generalize this for general $\sigma(x)$ functions, but to demonstrate the next point, it is not even necessary. Looking at the expression, we can observe that $$N\frac{\alpha}{1-\alpha}\log\left(\frac{\sigma_0}{\sigma_1}\right)-N\frac{\log\left(\frac{\sigma_\alpha}{\sigma_1}\right)}{1-\alpha}$$ depends highly on $N$ and even for a ratio of $\frac{\sigma_0}{\sigma_1}$ close to 1, we already obtain very strong negative values for high dimensions. The expression $\log(1-2M_1(p_A, p_B)+2M_{1-\alpha}(p_A, p_B))$ is far less sensitive w.r.t $p_A$ and for large dimensions of $N$ it is easily ``beaten'' by the first expression. Therefore, the higher the dimension $N$ is, the bigger $p_A$ or the closer to 1 the $\frac{\sigma_0}{\sigma_1}$ has to be in order to obtain even valid certified radius (not to speak about big). This points out that also the method of \cite{li2018certified} suffers from the curse of dimensionality, as we know it must have done. This method is not useful for big $N$, because the conditions on $p_A, \sigma_0, \sigma_1$ are so extreme, that barely any inputs would yield a positive certified radius. This fact is depicted in the Figure~\ref{cr li fcn of N}.

\begin{figure}[h!]
    \centering
    \includegraphics[width=0.60\textwidth]{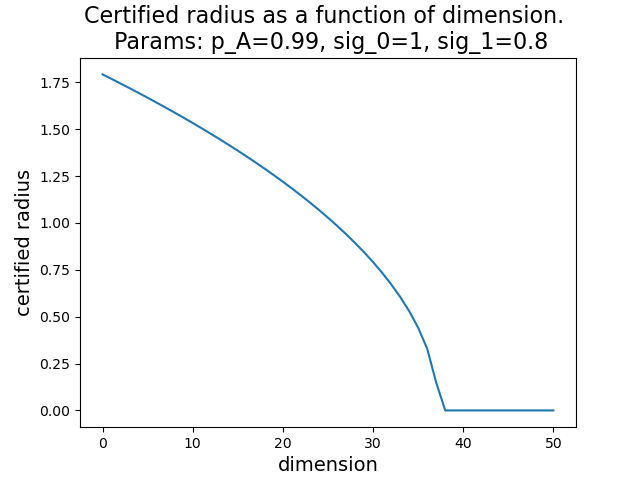}
    \caption{The certified radius as a function of dimension. Paremeters are $p_A=0.99, \sigma_0=1, \sigma_1=0.8$} 
    \label{cr li fcn of N}
\end{figure}

The key reason why this happens if done via R{\'e}nyi divergences is that while the divergence \\ $D_\alpha(\mathcal{N}(x_1, \sigma_1^2 I) || \mathcal{N}(x_0, \sigma_0^2 I))$ grows independently of dimension as $\norm{x_0-x_1}$ grows, it drastically increases for big $N$ even if $x_1=x_0$! This reflects the effect, that if $\sigma_0 \neq \sigma_1$, then the more dimensions we have, the more dissimilar are $\mathcal{N}(x_1, \sigma_1^2 I)$ and $\mathcal{N}(x_0, \sigma_0^2 I))$. We can think of it as a consequence of standard fact from statistics that the more data we have, the more confident statistics against the null hypothesis $\sigma_0=\sigma_1$ will we get if the null hypothesis is false. Since isotropic normal distributions can be actually treated as a sample of one-dimensional normal distributions, this is in accordance with our multivariate distributions setting.

\subsection{The Explanation of the Curse of Dimensionality} \label{appB: the curse part}

In the Section~\ref{sec:theory} we show that input-dependent RS suffers from the curse of dimenisonality. Now we will elaborate a bit more on this phenomenon and try to explain why it occurs. First, it is obvious from the Subsection~\ref{appB: the li part}, that also the generalized method of \cite{li2018certified} suffers from the curse of dimensionality, because the R{\'e}nyi divergence between two isotropic Gaussians with different variances grows considerably with respect to dimension. This suggests that the input-dependent RS might suffer from the curse of dimensionality in general. To motivate this idea even further, we present this easy observation: 

\begin{theorem} \label{thm the intuition crs}
Denote $R_C$ to be a certified radius given for $p_A$ and $\sigma_0$ at $x_0$ assuming the constant $\sigma_0$ and following the certification of \cite{cohen2019certified} \footnote{The ``C'' in the subscript of certified radius might come both from ``constant'' and ``Cohen et. al.''}. Assume, that we do the certification for each $x_1$ by assuming the worst case-classifier as in Theorem~\ref{lrt set}. Then, for any $x_0$, any function $\sigma(x)$ and any $p_A$, the following inequality holds: 
$$R \le R_C$$
\end{theorem}
\begin{proof}
Fix $x_1$ and $\sigma_1$. From Theorem~\ref{lrt set} we know that the worst-case classifier $f^*$ defines a ball $B$ such that $\mathbb{P}_0(B)=1-p_A$. From this it obviously follows, that the linear classifier $f_l$ and the linear space $B_l$ that assume constant $\sigma_0$ also for $x_1$ and is the worst-case for $\sigma_0$ such that $\mathbb{P}_0(B_l)=1-p_A$ is \textit{not} worst-case for the case of using $\sigma_1$ instead. Therefore, $\mathbb{P}_1(B_l) \le \mathbb{P}_1(B)$. 

Moreover, let $\mathbb{P}_1^C$ be a probability measure corresponding to $\mathcal{N}(x_1, \sigma_0 I)$, i.e. the probability measure assuming constant $\sigma_0$. It is easy to see that $\mathbb{P}_1^C(B_l)>0.5 \iff \mathbb{P}_1(B_l)>0.5$ because the probability of a linear half-space under isotropic normal distribution is bigger than half if and only if the mean is contained in the half-space. 

Assume, for contradiction that $R > R_C$. From that, it exists a particular $x_1$ such that $\mathbb{P}_1^C(B_l) > 0.5 > \mathbb{P}_1(B)$, because otherwise there would be no such point, which would cause $R > R_C$. However, $\mathbb{P}_1^C(B_l) > 0.5 \implies \mathbb{P}_1(B_l) > 0.5$, thus $\mathbb{P}_1(B_l)>\mathbb{P}_1(B)$ and that is contradiction. 
\end{proof}

This theorem shows, that we can never achieve a better certified radius at $x_0$ using $\sigma_0$ and having probability $p_A$ than that, which we would get by \cite{cohen2019certified}'s certification. Of course, this does not mean, that using non-constant $\sigma$ is useless, since $\sigma_0$ can vary. The question is, how much do we lose using non-constant $\sigma$. To get a better intuition, we plot the functions $\xi_<$ and $\xi_>$ under different setups in Figure~\ref{real ball probs}, together with $\mathbb{P}_1(B_l)$ from the proof of Theorem~\ref{thm the intuition crs}. From the top row we can deduce that dimension $N$ has a very significant impact on the probabilities and therefore also on the certified radius. We particularly point out the fact, that even $\xi_>(0), \xi_<(0)$ can have significant margin w.r.t. to the probability coming out of linear classifier.\footnote{Notice the similarity with R{\'e}nyi divergence, which also has positive value even for $x_0=x_1$ if $\sigma_0 \neq \sigma_1$ and then grows rather reasonably with distance.} Already for $N=90$, we are not able to certify $p_A=0.99$ for rather conservative value of $\frac{\sigma_0}{\sigma_1}$. From middle row we see, that decreasing $\frac{\sigma_0}{\sigma_1}$ can mitigate this effect strongly. For instance, for $\sigma_0=1, \sigma_1=0.95$ the difference between $\mathbb{P}_1(B)$ and $\mathbb{P}_1(B_l)$ is almost negotiated. Bottom row compares $\xi_>(a), \xi_<(a)$ and the respective linear classifier probabilities. We can see, that the case $\sigma_0<\sigma_1$ might cause stronger restrictions on our certification (yet we deduce it just form the picture). 

What is the reason for $\xi_>(a), \xi_<(a)$ being so big even at 0? The problem is following: Assume $\sigma_0>\sigma_1$. If $x_0=x_1$, the worst-case classifier coming from Lemma~\ref{lrt set} will be a ball $B$ centered right at $x_0$, such that $\mathbb{P}_0(B)=1-p_A$. If we look at $\mathbb{P}_1(B)$, we see, that we have the same ball centered directly at the mean, but the variance of the distribution is smaller. Using spherical symmetry of the isotropic gaussian distribution, this is equivalent to evaluating the probability of a bigger ball. If we fix $\frac{\sigma_0}{\sigma_1}$ and look at the ratio of probabilities $\frac{\mathbb{P}_1(B)}{\mathbb{P}_0(B)}$ with increasing $N$, the curse of dimensionality comes into the game. For $N=2$, the ratio is not too big. However, if $N=3072$, like in CIFAR10, this ratio is far bigger. This can be intuitively seen from a property of chi-square distribution (which is present in the case $x_0=x_1$), that while expectation is $N$, the standard deviation is ``just'' $\sqrt{2N}$, i.e. $\frac{\sqrt{Var(\chi^2_N)}}{\mathbb{E}(\chi^2_N)} \xrightarrow[]{} 0$ as $N \xrightarrow[]{} \infty$.

\begin{figure}[t!]
    \centering
    \begin{minipage}[b]{0.40\linewidth}
        \includegraphics[width=\textwidth]{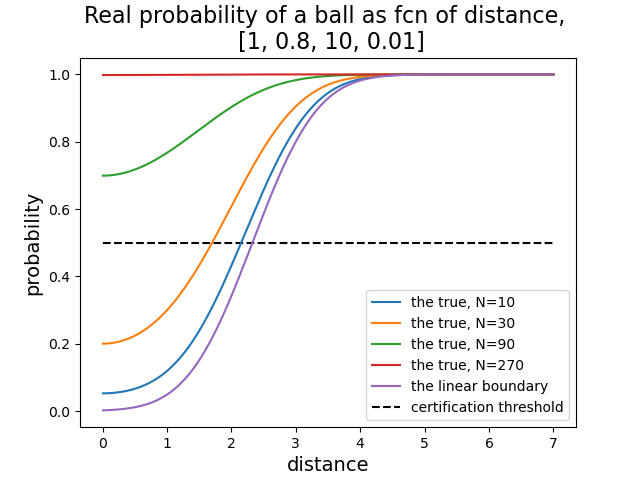}
    \end{minipage}
    \begin{minipage}[b]{0.40\linewidth}
        \includegraphics[width=\textwidth]{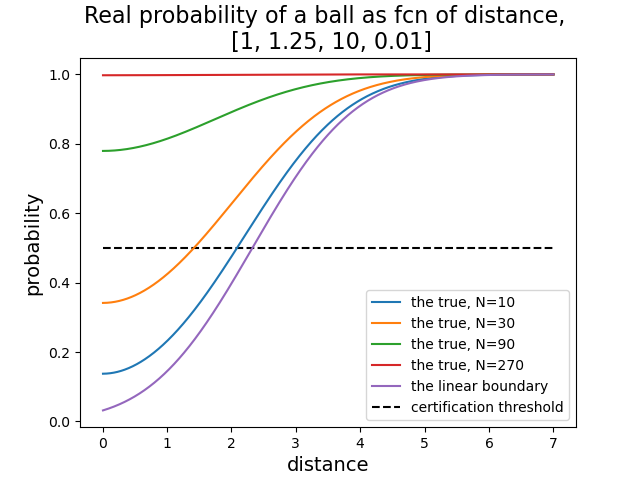}
    \end{minipage}
    \begin{minipage}[b]{0.40\linewidth}
        \includegraphics[width=\textwidth]{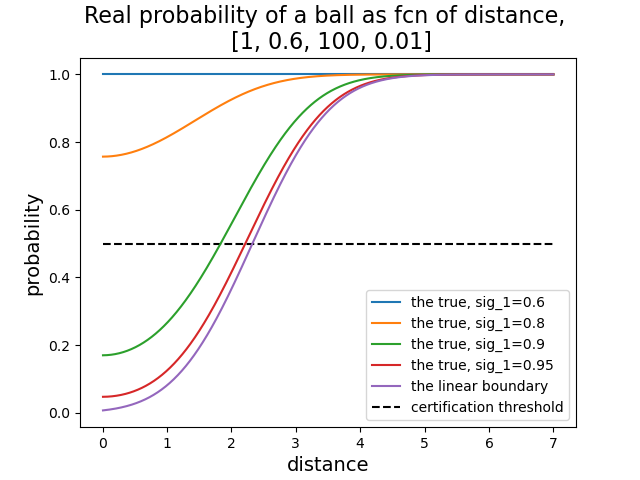}
    \end{minipage}
    \begin{minipage}[b]{0.40\linewidth}
        \includegraphics[width=\textwidth]{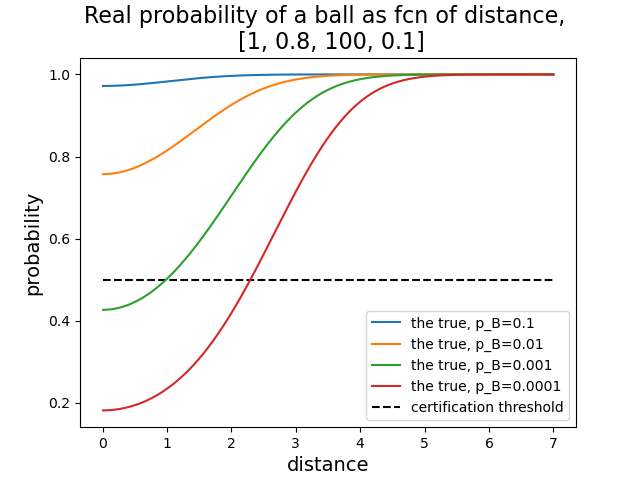}
    \end{minipage}
    \begin{minipage}[b]{0.40\linewidth}
        \includegraphics[width=\textwidth]{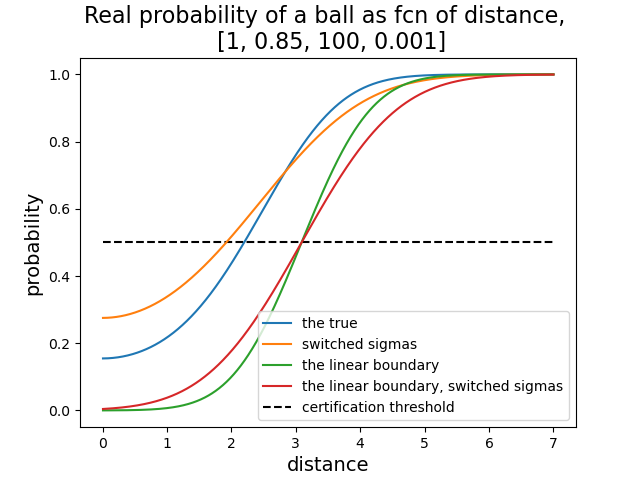}
    \end{minipage}
    \caption{Plots of $\xi_>(a), \xi_<(a)$ for different setups. Coding for parameters is: $[\sigma_0, \sigma_1, N, p_B]$ \textbf{Top:} $\xi_>(a)$ left, $\xi_<(a)$ right, varying values of $N$. \textbf{Center:} On the left, $\xi_>(a)$ for varying $\sigma_1$, on the right $\xi_>(a)$ for varying $p_B$. \textbf{Bottom:} $\xi_>(a)$ and $\xi_<(a)$ compared.}
    \label{real ball probs}
\end{figure} 

\subsection{Why Does the Input-dependent Smoothing Work Better for Small $\sigma$ Values?} \label{appC: ssec: why small sigma values better}
As can be observed in Section~\ref{experiments} and Appendix~\ref{appE: experiments and ablations}, the bigger the $\sigma_b=\sigma$ we use, the harder it is to keep up to standards of constant smoothing. An interesting question is, why is the usage of small $\sigma_b=\sigma$ helpful for the input-dependent smoothing? 

Assume fixed $\sigma$, say $\sigma_b=\sigma=0.12$. The theoretical bound on the certified radius given 100000 Monte-Carlo samplings and 0.001 confidence level using constant smoothing is about 0.48. Having $\sigma(x) \sim 0.12$, we cannot expect much bigger certified radius. Therefore, if we follow Theorem~\ref{thm the concrete method}, the values of $\exp(-rR)$ and $\exp(rR)$ in the critical distance $\sim 0.5$ will be much closer to 1, than the values of $\exp(-rR)$ and $\exp(rR)$ if we used $\sigma_b=\sigma=0.50$ instead, where the critical values of $R$ could be much bigger than 0.5. Therefore, the ``gain'' in $\mathbb{P}_1(B)$ imposed by the curse of dimensionality, compared to $\mathbb{P}_1(B)$ assuming constant $\sigma$ will not be that severe yet. This means, that the loss in certified radius caused by the curse of dimensionality will be much less pronounced on the ``active'' range of certified radiuses (those for which the constant smoothing still works), compared to using big $\sigma_b=\sigma$. To support this idea, we demonstrate it on Figure~\ref{fig: small sigma working better}, where we depict the certified radius as a function of distance from decision boundary, assuming $f$ to be a linear classifier, using $\sigma_b=\sigma=0.12$ and $\sigma_b=\sigma=0.50$ for comparison. 

\begin{figure}[t!]
    \centering
    \begin{minipage}[b]{0.48\linewidth}
        \includegraphics[width=\textwidth]{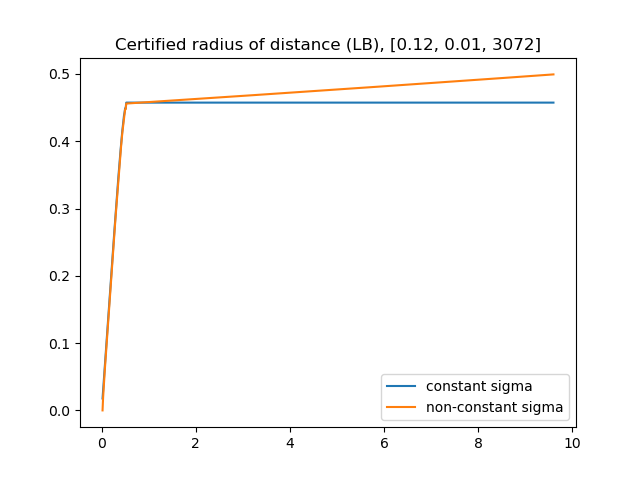}
    \end{minipage}
    \begin{minipage}[b]{0.48\linewidth}
        \includegraphics[width=\textwidth]{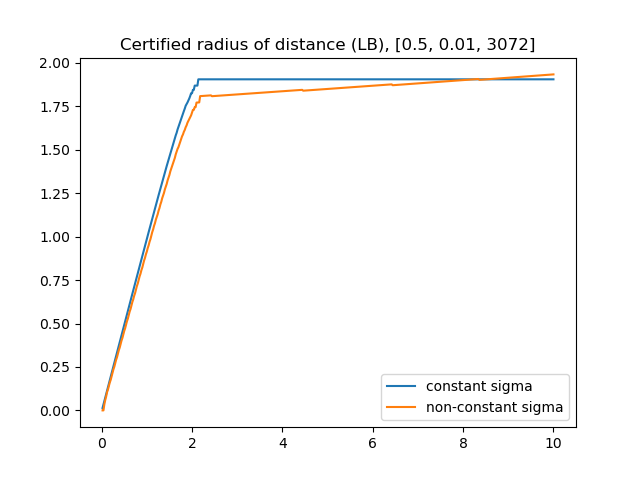}
    \end{minipage}
    \caption{Comparison of certified radius as a function of distance for constant and input-dependent smoothing. Left: $\sigma_b=\sigma=0.12$, right: $\sigma_b=\sigma=0.50$.}
    \label{fig: small sigma working better}
\end{figure}

\subsection{How Does the Curse of Dimensionality Affect the Total Possible Variability of $\sigma(x)$?} \label{appC: total effect of CoD}

Fix certain type of task, say RGB image classification with images of similar object, but consider many possible resolutions (dimensions $N$). Given two random images from the test set ($x_0, x_1$), what is the biggest reasonable value of $|\sigma(x_0)/\sigma(x_1)-1|$? Theoretically, the expression is bounded by $|\exp(\pm r \norm{x_0-x_1})-1|$, given that $r$ is the semi-elasticity constant of $\sigma(x)$. However, the average distance between two samples from a test set of constant size, but increasing dimension scales as $\sqrt{N}$. Therefore, with constant $r$, this upper-bound increases.

The increasing distance between samples is, therefore, a countereffect to the curse of dimensionality. In simple words, we have ``more distance to change $\sigma(x_0)$ to $\sigma(x_1)$''. Even if the $r$ decreased just as $1/\sqrt{N}$, the increasing distances would cancel the effect of the curse of dimensionaity and as a result, the maximal reasonable value of $|\sigma(x_0)/\sigma(x_1)-1|$ would remain roughly constant w.r.t. $N$. However, we need to take into account another effect. As the dimension increases, also the average distance of samples from the decision boundary increases. This is because the distances in general grow with dimension and if we assume that the number of intersections of a line segment between $x_0$ and $x_1$ with the decision boundary of the network remains roughly constant then the average distance from the decision boundary grows as $\sqrt{N}$ too. In order to compensate for this, we need to adjust the basic level of $\sigma(x)$ (which we later call $\sigma_b$ and can be understood as the general offset of our $\sigma(x)$) as $\sqrt{N}$ too. This is because the maximal attainable certified radius given fixed confidence level $\alpha$ and the number of Monte-Carlo samples is a constant multiple of $\sigma(x)$. 

However, with increased $\sigma$, we need to decrease the semi-elasticity rate $r$ in order to obtain full certifications (see also Appendix~\ref{appC: ssec: why small sigma values better} for intuition behind this).

As a sketch of proof, we provide a simple computation, which tells us the approximate asymptotic behavior of $|\sigma(x_0)/\sigma(x_1)-1|$. By Theorem~\ref{main thm corollary} it holds: $$\exp(-rc\sqrt{N})\ge \sqrt{1-2\sqrt{\frac{-\log(p_B)}{N}}},$$

if we want to be able to predict a certified radius of $c\sqrt{N}$ (though this is just a necessary condition. For sufficiency, the LHS must be much closer to 1). After simple manipulation, we obtain: $$r \le -\frac{1}{2c\sqrt{N}} \log\left(1-2\sqrt{\frac{-\log(p_B)}{N}}\right) \sim -\frac{1}{2c\sqrt{N}}(-1)2\sqrt{\frac{-\log(p_B)}{N}}=\frac{\sqrt{-\log(p_B)}}{cN}.$$

So the rate scales as $1/N$. Now we have:
\begin{align*}
&|\sigma(x_0)/\sigma(x_1)-1| \le |\exp(\pm r \norm{x_0-x_1})-1| \le \exp(r \norm{x_0-x_1})-1 \le \\ &\exp(\frac{\sqrt{-\log(p_B)}}{cN}C\sqrt{N})-1 \sim \frac{\sqrt{-C\log(p_B)}}{c\sqrt{N}}.
\end{align*}

However, we must also note another important effect -- while it is \textit{desirable} to scale the base $\sigma$ as $\sqrt{N}$ with the dimension, it might not be \textit{possible}. The reason is that the noise norm also grows with $\sqrt{N},$ therefore increasing the $\sigma$ would increase the relative strength of perturbations. From this point of view, to maintain equally strong perturbations we should keep the $\sigma$ independent of dimension. In this case, the curse of dimensionality of input-dependent smoothign would be mitigated. However, in computer vision, it was demonstrated in \cite{cohen2019certified} that the increased dimension allows for bigger noise thanks to the noise-cancelling effect which grows with increasing dimension too. All-in-all, it is rather non-trivial to predict how strong will the curse of dimensionality act in practice, as many countering effect play significant role here. However it is clear that $r=\mathcal{O}(N^{-1})$ is a strong scaling that will probably influence any practical learning regime. 

\subsection{Does the Curse of Dimensionality Apply in Multi-class Regime?} \label{appC: multi-class regime}
In the main text, we presented a setup, where $\overline{p_B}$ is set to be $1-\underline{p_A}$. This is equivalent to pretending that we have just 2 classes. By not estimating the proper value of $p_B$ we lose some amount of power and the resulting certified radius is smaller than it could have been, did we have the $\overline{p_B}$ as well. This is most pronounced for datasets with many classes. The natural question, therefore, is, whether we could avoid the curse of dimensionality by properly estimating the $p_B$ together with $p_A$. The answer is no. The problem is that the the theory in Section~\ref{sec:theory} already implicitly works with the estimate of $p_B$ in a form of $1-\underline{p_A}$. The theory would work also with any other estimate of $p_B$. Assuming constant $p_B$, instead of constant $p_A$, as we did in Section~\ref{sec:theory}, will, therefore, yield the same conclusions. Moreover, there is neither theoretical, nor practical reason, why should $p_B$ decrease with increasing dimension. 

This insight even applies to the question of the usage of input-dependent RS in practice. The assumption $p_B=1-p_A$ is no more important in Section~\ref{sec: practical framework} than in Section~\ref{sec:theory}. Therefore, we can apply our method also for the $\overline{p_B}$ obtained directly by Monte-Carlo sampling for the class $B$ (or by any other estimation method).

\section{Competitive Work} \label{appB: concurrent work}
As we mention in Section~\ref{intro}, the idea to use input-dependent RS is not new. It has popped out in years 2020 and 2021 in at least four works from three completely distinct groups of authors, even though none of these works has been successfully published yet. We find it necessary to comment on all of these works because of two orthogonal reasons. First, it is a good practice to compare our work with the competitive work to see what are pros and cons of these similar approaches and to what extend the approaches differ. Second, we are convinced, that three of these four works claim results, which are not mathematically valid. We find this to be a particularly critical problem in a domain such as certifiable robustness, which is by definition based on rigorous, mathematical certifications. 

\subsection{The Work of \citet{wang2021pretraintofinetune}}
In this work, authors have two main contributions -- first, they propose a two-phase training, where in the second phase, for each sample $x_i$, roughly the optimal $\sigma_i$ is being found and then this sample $x_i$ is being augmented with this $\sigma_i$ as an augmentation standard deviation. Authors call this method \textit{pretrain to finetune}. Second, they provide a specific version of input-dependent RS. Essentially, they try to overcome the mathematical problems connected to the usage of non-constant $\sigma(x)$ by splitting the input space in so called \textit{robust regions} $R_i$, where the constant $\sigma_i$ is guaranteed to be used. All the certified balls are guaranteed to lie within just one of these robust regions, making sure that within one certified region, constant level of $\sigma$ is used. Authors test this method on CIFAR10 and MNIST and show, that the method can outperform existing state-of-the-art approaches, mainly on the more complex CIFAR10 dataset. 

However, we make several points, which make the results of this work, as well as the proposed method less impressive: 
\begin{itemize}
    \item The computational complexity of both their train-time and test-time algorithms seems to be quite high.
    \item The final smoothed classifier depends on the order of the incoming samples. As a consequence, it is not clear, whether the method works well for any permutation of the would-be tested samples. This creates another adversarial attack possibility - to attack the final smoothed classifier by manipulating the test set so that the order of samples is inappropriate for the good functionality of the final smoothed classifier. 
    \item Even more, the fact, that the smoothed classifier depends on the order of the would-be tested samples makes it necessary, that the same smoothed classifier is used all the time for some test session in a real-world applications. For instance, a camera recognizing faces to approve an entry to a high-security building would need to keep the same model for its whole functional life, because restarting the model would enable attackers to create attacks on the predictions from the previous session. This might lead to significant restrictions on the practical usability of this method. 
\end{itemize}

\subsection{The Works of \citet{alfarra2020data} and \citet{eiras2021ancer}} \label{sec: alfarra}
In these works, similarly as in the work of \cite{chen2021insta}, authors suggest to optimize \textit{in each} test point $x$ for such a $\sigma(x)$, that maximizes the certified radius given by \cite{zhai2020macer}, which is an extension of \cite{cohen2019certified}'s certified radius for soft smoothing. The optimization for $\sigma(x)$ differs but is similar in some respect (as will be discussed). 

Besides, all three works further propose input-dependent training procedure, for which $\sigma(x)$ - the standard deviation of gaussian data augmentation is also optimized. Altogether, both authors claim strong improvements over all the previous impactful works like \cite{cohen2019certified, zhai2020macer, salman2019provably}. The only significant difference between the works of \cite{alfarra2020data} and \cite{eiras2021ancer} (which have strong author intersections) is that in \cite{eiras2021ancer}, authors build upon \cite{alfarra2020data}'s work and move from the isotropic smoothing to the smoothing with some specific anisotropic distributions. 

As mentioned, authors first deviate from the setup of \cite{cohen2019certified} and turn to the setup introduced by \cite{zhai2020macer}, i.e. they use soft smoothed classifier $G$ defined as $${G_F(x)}_C = \mathbb{E}_{\delta \sim \mathcal{N}(0, \sigma^2 I)} F(x+\delta)_C.$$

The key property of soft smoothed classifiers is that the \cite{cohen2019certified}'s result on certified radius holds for them too. 

\begin{theorem}[certified radius for soft smoothed classifiers] \label{cohen cr soft}
Let $G$ be the soft smoothed probability predictor. Let $x$ be s.t. $$G(x)_A \ge \underline{E_A} \ge \overline{E_B} \ge G(x)_B.$$
Then, the smoothed classifier $g$ is robust at $x$ with radius $$R=\frac{\sigma}{2}(\Phi^{-1}(\underline{E_A})-\Phi^{-1}(\overline{E_B}))=\sigma \frac{\Phi^{-1}(\underline{E_A})+\Phi^{-1}(1-\overline{E_B}))}{2},$$ where $\Phi^{-1}$ denotes the quantile function of standard normal distribution. 
\end{theorem}
\begin{proof}
Is provided in \cite{zhai2020macer}.
\end{proof}

Note, that it is, similarly as in the hard randomized smoothing version of this theorem, essential to provide lower and upper confidence bounds for $G(x)_A$ and $G(x)_B$, otherwise we cannot use this theorem with the required probability that the certified radius is valid. Denote $G(x,\sigma)$ to be the soft smoothed classifier using $\sigma$ in $x$. Authors propose to use the following theoretical $\sigma(x)$ function: 
\begin{equation} \label{alfarra theoretical fcn}
\sigma(x)=\underset{\sigma>0}{\arg\max} \frac{\sigma}{2}(\Phi^{-1}(G(x,\sigma(x))_A)-\Phi^{-1}(G(x,\sigma(x))_B)).
\end{equation}

It is of course not possible to optimize for this particular function since it is not known. It is also not feasible to run the Monte-Carlo sampling for each $\sigma$, because that is too costly and moreover due to stochasticity, it would lead to discontinuous function. Treatment of this problem is probably the most pronounced difference between the works of \cite{alfarra2020data} and \cite{chen2021insta}. 

\cite{alfarra2020data} use the following easy observation: $\mathcal{N}(0, \sigma^2 I) \equiv \sigma\mathcal{N}(0, I)$. Assume we have $\delta_i, i \in \{1, \dots, M\}$ be i.i.d. sample from $\mathcal{N}(0, I)$. Obviously, $G(x,\sigma(x))_A \sim \frac{1}{M}\sum\limits_{i=1}^M F(x+\sigma\delta_i)_A,$ since this is just the empirical mean of the theoretical expectation. Then, Expression~\ref{alfarra theoretical fcn} can be approximated as: 

\begin{equation} \label{alfarra practical fcn}
\sigma(x)=\underset{\sigma>0}{\arg\max} \frac{\sigma}{2}\left(\Phi^{-1}\left(\frac{1}{M}\sum\limits_{i=1}^M F(x+\sigma\delta_i)_A\right)-\Phi^{-1}\left(\frac{1}{M}\sum\limits_{i=1}^M F(x+\sigma\delta_i)_B\right)\right).
\end{equation}

Here, $M$ is the number of Monte-Carlo samplings used to approximate this function. Note, that this function is a random realization of stochastic process in $\sigma$ which is driven by the stochasticity in the sample $\delta_i, i \in \{1, \dots, M\}$. To find the maximum of this function, authors furhter propose to use simple gradient ascent, which is possible due to the simple differentiable form of Expression~\ref{alfarra practical fcn}. This differentiability is one of the main motivations to switch from hard to soft randomized smoothing. Now, we are able to state the exact optimization algorithm of \cite{alfarra2020data}: 

\begin{algorithm}[t!] 
\caption{Data dependent certification \citep{alfarra2020data}}
\begin{algorithmic}
\Function{OptimizeSigma}{$F, x, \beta, \sigma_0, M, K$}:
    \For{$k = 0, \dots, K$}
        \State $\text{sample} \hspace{1mm} \delta_1, \dots, \delta_M \sim \mathcal{N}(0, I)$
        \State $\phi(\sigma_k)=\frac{1}{M}\sum\limits_{i=1}^M F(x+\sigma\delta_i)$
        \State $\hat E_A(\sigma_k) = \max_C \phi(\sigma_k)_C$
        \State $\hat E_B(\sigma_k) = \max_{C \neq A} \phi(\sigma_k)_C$
        \State $R(\sigma_k)=\frac{\sigma_k}{2}(\Phi^{-1}(\hat E_A(\sigma_k))-\Phi^{-1}(\hat E_B(\sigma_k)))$
        \State $\sigma_{k+1} \gets \sigma_k + \beta \nabla_{\sigma_k} R(\sigma_k)$
    \EndFor
    \State $\sigma^* = \sigma_K$
    \State \Return $\sigma^*$
\EndFunction
\end{algorithmic}
\label{alfarra algo}
\end{algorithm}

Note, that being done in this way, this algorithm can be viewed as a stochastic gradient ascent. After obtaining $\sigma^* \equiv \sigma(x)$, authors further run the Monte-Carlo sampling to estimate the certified radius exactly as in \cite{cohen2019certified}, but with $\sigma(x)$ instead of some global $\sigma$. Using this algorithm, authors achieve significant improvement over the \cite{cohen2019certified}'s results, particularly getting rid of the first problem mentioned in Appendix~\ref{appA: the motivation}, the truncation issue. For the results, we refer to \cite{alfarra2020data}. We will now give several comments on this algorithm and this method. 

To begin with, in this optimization, authors do not adjust the estimated expectations and therefore don't use lower confidence bounds, but rather raw estimates. This is not incorrect, since these estimates are not used directly for the estimation of certified radius, but it is inconsistent with the resulting estimation. In other words, authors optimize for a slightly different function than they then use. The difference is, however, not very big apart from extreme values of $E_A$, where the difference might be really significant. 

To overcome slightly this inconsistence, authors further (without comment) use clamping of the $\hat E_A(\sigma_k)$ and $\hat E_B(\sigma_k)$ on the interval $[0.02, 0.98]$. I.e. if $\hat E_A(\sigma_k)>0.98$, it will be set to $0.98$ and this is also taken into account in the computation of gradients. This way, authors get rid of the inconvenient issue, that if $G(x)_A \sim 1$, then $\hat E_A(\sigma_k) \sim 1$ for $\sigma_k \sim 0$, what might cause very big value of  
$\Phi^{-1}(\hat E_A(\sigma_k))$, yielding strong inconsistency with what would be obtained, if lower confidence bound was used instead. 

However, the clamping causes even stronger inconsistence in the end. Note, that if $G(x)_A \sim 1$, then the true value of $E_A(\sigma_k)$ would be really close to $1$, yielding high values of $\Phi^{-1}(E_A(\sigma_k))$. This value would be far better approximated by the lower confidence bound than with the clamping, since the lower confidence bound of 1 for $M=100000$ and $\alpha=0.001$ is more than $0.9999$, while the clamped value is just $0.98$. This makes small values of $\sigma$ highly disadvantageous, since $\frac{\sigma}{2} \xrightarrow[]{} 0$ as $\sigma \xrightarrow[]{} 0$, yet $\Phi^{-1}(\hat E_A(\sigma_k))$ is being stuck on $\Phi^{-1}(0.98)$. In other words, this way authors artificially force the resulting $\sigma(x)$ to be big enough, s.t. $E(\sigma(x))_A \le 0.98$. This assumption is not commented in the article and might result in intransparent behaviour. 

Second of all, authors use $M=1$ for their experiments. This can be interpreted as using batch size 1 in classical SGD. We suppose that this small batch size is suboptimal since it yields an insanely high variance of the gradient. 

Third of all, during the search for $\sigma(x)$, it is not taken into account, whether the prediction is correct or not. This is, of course, a scientifically correct approach, since we cannot look at the label of the test sample before the very final evaluation. However, it is also problematic, since the function in Expression~\ref{alfarra theoretical fcn} might attain its optimum in such a $\sigma(x)$, which leads to misclassification. This could have been avoided if constant $\sigma$ was used instead. 

To further illustrate this issue, assume $F(x)=\mathbf{1}(B_1(0))$, i.e. $F(x)$ predicts class 1 if and only if $\norm{x}\le 1$, otherwise predicts class 0. Assume we are certifying $x \equiv 0$ and assume that $\sigma_0$ in Algorithm~\ref{alfarra algo} is initialized such that class 0 is already dominating. Then, we will have positive gradient $\nabla_{\sigma_k} R(\sigma_k)$ in all steps, because $F(\sigma \delta_i)$ is obviously non-increasing, so the number of points classified as class 1 for fixed sample $\delta_i, i \in 1, \dots, M$ is decreasing, yielding $\hat E_A(\sigma_k)$ non-decreasing in $\sigma_k$, while $\frac{\sigma_k}{2}$ strictly increasing in $\sigma_k$. This way, the $\sigma_k$ will diverge to $\infty$ for $k \xrightarrow[]{} \infty$. However, point $x \equiv 0$ is classified as class 1, yielding misclassification which is, moreover, assigned very high certified radius. 

This issue is actually even more general - the function in Expression~\ref{alfarra theoretical fcn} does in most cases (assuming infinite region $\mathbb{R}^N$) \textit{not} possess global maximum, because usually $$\underset{\sigma \xrightarrow[]{} \infty}{\lim} \frac{\sigma}{2}(\Phi^{-1}(G(x,\sigma(x))_A)-\Phi^{-1}(G(x,\sigma(x))_B)) = \infty.$$ This can be seen, for instance, easily for the $F(x)=\mathbf{1}(B_1(0))$, but it is the case for any hard classifier, for which one region becomes to have dominating area as the radius around some $x_0$ goes to infinity. This is because, if some region becomes to be dominating (for instance if all other regions are bounded), then $\frac{\sigma}{2}$ grows, while $\Phi^{-1}(G(x,\sigma(x))_A)-\Phi^{-1}(G(x,\sigma(x))_B)$ either grows too, or stagnates, making the whole function strictly increasing with sufficiently high slope. 

This issue also throws the hyperparameter $K$ under closer inspection. What is the effect of this hyperparameter on the performance of the algorithm? From the previous paragraph, it seems, that this parameter serves not only as the `` scaled number of epochs'', but also as some stability parameter, which, however, does not have theoretical, but rather practical justification. 

Another issue is, that the function in Expression~\ref{alfarra theoretical fcn} might be non-convex and might possess many different local minima, from which not all (or rather just a few) are actually reasonable. Therefore, the Algorithm~\ref{alfarra algo} is very sensitive to initialization $\sigma_0$. 

However, probably the biggest issue of all is connected to the impossibility result showed in Section~\ref{sec:theory}, which shows, that the Algorithm~\ref{alfarra algo} actually yields invalid certified radiuses. Why it is so?

First of all, we must justify, that our impossibility result is applicable also for the \textit{soft} randomized smoothing. This is because classifiers of type $F(x)_C=\textbf{1}(x \in R_C)$ for $R_C$ being decision region for class $C$ are among applicable classifiers s.t. $G(x, \sigma)_A=E_A$. With such classifiers, however, there is no difference between soft and hard smoothing and moreover $E_A \equiv p_A$ from our setup. This way we can construct the worst-case classifiers $F^*$ exactly as in our setup and therefore the same worst-case classifiers and subsequent adversarial examples are applicable here as well. In other words, for fixed value of soft smoothed $G(x, \sigma)_A=E_A$ we can denote $p_A = E_A$ and find the worst-case hard classifier $F$ defined as indicator of the worst-case ball, which will yield $E_B \equiv \mathbb{P}_1(B)$ from Theorem~\ref{thm ncchsq} in some queried point $x_1$. 

As we have seen in previous paragraphs, the resulting $\sigma(x)$ yielded in Algorithm~\ref{alfarra algo} is very instable and stochastic - it depends heavily on $F, \sigma_0, K, \beta, M$ and of course $\delta_i, i  \in \{1, \dots, M\}$ for each iteration of the for cycle. Now, for instance for CIFAR10 and $p_A=0.99$, we have the minimal possible ratio $\frac{\sigma_0}{\sigma_1}$ equal to more than $0.96$. It is hard to believe, that such instable, highly stochastic and non-regularized (except for $K, \beta$) method will yield $\sigma(x)$ sufficiently slowly varying such that within the certified radius around $x_0$, there will be no $x_1$ for which $\sigma_1$ deviates more than by this strict threshold from $\sigma_0$. This is even more pronounced on ImageNet, where the minimal possible ratio $\frac{\sigma_1}{\sigma_0}$ is above 0.99 for any $p_A$ or $E_A$. 

Even without the help of curse of dimensionality, we can construct a counterexample for which the algorithm will not yield valid certified radius. Assume again $F(x)=\mathbf{1}(B_1(0))$ and assume modest dimension $N=2$. Assume we try to certify point $x_0 \equiv [50, 0]$. Then, the theoretical $\sigma$-dependent function from Equation~\ref{alfarra theoretical fcn} is depicted on Figure~\ref{alfarra counterexample}.

\begin{figure}[t!]
    \centering
    \includegraphics[width=0.60\textwidth]{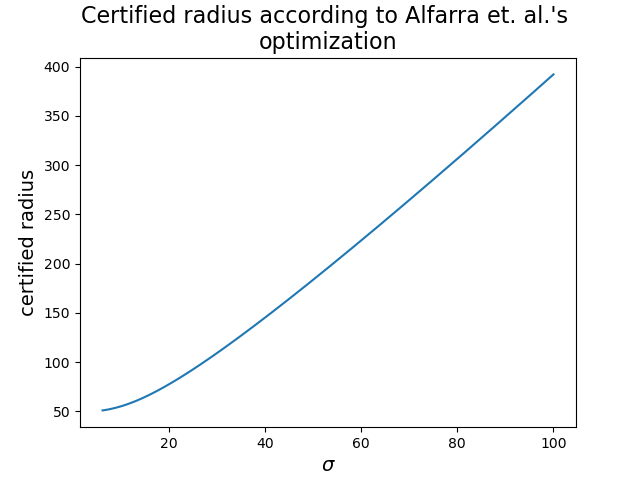}
    \caption{The theoretical certified radius as in Expression~\ref{alfarra theoretical fcn}. The function is monotonically increasing on interval $[0, 100]$ and will further be increasing too.} 
    \label{alfarra counterexample}
\end{figure}

We can see, that the resulting $\sigma(x_0)$ will be as big as our regularizers $K$ and $\beta$ in Algorithm~\ref{alfarra algo} will allow. Therefore, if we run the algorithm for $K$ high-enough, surely the resulting certified radius will be far bigger than 50. However, if we certify the point $x_0 \equiv [0,0]$ and we start with $\sigma_0=0.25$, for instance, then the $\sigma$-dependent certified radius in Expression~\ref{alfarra theoretical fcn} will be decreasing in this $\sigma_0$, yielding $\sigma(x_0)<0.25$, which will result in classification of class 1. This point $[0, 0]$ lies within the ``certified'' range of $[50, 0]$, yet it is not classified the same, because, obviously, $[50, 0]$ is classified as class 0. This is therefore a counterexample to the validity of \cite{alfarra2020data}'s certification method and their results. 

Note that even though our counterexample is a bit ``extreme'' and one could argue that in practice such a situation would not occur, we must emphasize, that this counterexample is constructed even without the help of the curse of dimensionality. In practice, it fully suffices, that for $x_0$ and some certified radius $R$ in $x_0$, there exists $x_1$ within the range of this certified radius, s.t. $\sigma_1$ is quite dissimilar to $\sigma_0$. If such situation occurs, then $R$ surely is \textit{not} a valid certified radius. 

\begin{remark}
In the latest version of their work, the authors themselves point out the issue outlined here. To overcome this issue, the authors further propose the following fix. Let test set be $\{x_i\}_{i=1}^d$. The samples are being processed consecutively from 1 to $d$. After obtaining a radius $R_i$ for a sample $x_i$, they ``reserve'' the ball around $x_i$ with radius $R_i$ so that every point falling in this ball from now on will be classified the same class as $x_i$. They also adjust a newly-computed certified radius $R_j$ of a sample $x_j$ so that the ball around $x_j$ doesn't intersect with any previously reserved ball $x_i, \hspace{1mm} i<j$. If the sample falls into the already reserved region, the prediction on this sample is being forced to be the prediction of the sample in the region, and the certified radius is set such that the ball around the new sample is a subset of a ball reserved by the old sample. 

Let us note that this fix by far does not solve the issues. Technically, it partially gets rid of the mathematical invalidity of certified radiuses, since now we at least have a guarantee that within the session for which we store all the previous test samples, there will be no direct violations of the certificates. However, this fix brings many other subtle and serious problems. 

\begin{itemize}
    \item Similarly as in \citet{wang2021pretraintofinetune} it does depend on the order of incoming samples, creating possibilities for new adversarial attacks. To compare fairly, one would need to compare this method taking the worst possible certified accuracy result from all orderings of the test set. 
    \item This method requires storing all the previously checked test samples, which might be a really big burden in some inference-heavy applications such as autonomous driving, but also many others. 
    \item The classical randomized smoothing or also our framework provide guarantees which are session-agnostic. In other words, given a fixed model $f$ and a fixed variance function $\sigma(x)$ (possibly constant), the certified radius for a sample $x_0$ holds anywhere in the world at any time, on any device. Yet the model of \citet{alfarra2020data} requires synchronization of all the testing devices and if this synchronization is not provided then the guarantees hold just locally. Imagine our cell phones use some fixed classifier. If the phones were not synchronized, then within the certified radius around $x_0$ provided by one phone, there might be an adversarial example in some other phone. 
    \item Probably the biggest issue is that the method uses a formula for the computation of certified radiuses which completely loses its meaning when the guarantees are actually based just on the fix. The \citet{cohen2019certified}'s formula for certified radius is completely tailored for constant smoothing and used in this way it is just a random formula. A very simple and naive method would outperform everything existent. It just suffices to classify with $f$ or $g$ created with very small, constant $\sigma$ and then give constant certified radius to all the samples, computed such that given the size of the test set (which is known for instance in scientific comparisons) and the average distances between closest samples, we would choose radius which is just small enough to statistically not create intersections of certified regions. For instance for CIFAR10 and 10000 test samples, certified radiuses of 4 or 5 would be completely safe and wouldn't create possibly any intersections. However, it is clear that such a method doesn't make sense. 
\end{itemize}

Why this method is so appealing at the first sight? The problem is that in practical comparisons on CIFAR or MNIST, there are no clashes of certified regions. That would not be the case if the test sets would be bigger. The average distance from the closest neighbor goes as $\frac{1}{\sqrt{d}}$ where $d$ is the number of samples.

One could argue that there might be something as ``blessing of dimensionality''. The distances between samples grow with dimension and so the chance of clashes is smaller and smaller and thus extremely big amounts of test data would be needed to create some problems. However, this is not correct reasoning since with growing dimension, also the required $\sigma$ grows to provide appropriately big certified radiuses. From this point of view, the amount of test samples necessary to create troubles is roughly constant with respect to dimension and depends just on the topology of the data distribution. 
\end{remark}

\subsection{The Work of \citet{chen2021insta}}

The methodology of \cite{chen2021insta} is rather similar to that of \cite{alfarra2020data}. The biggest difference consists in the optimization of Expression~\ref{alfarra theoretical fcn}. 

Instead of stochastic gradient descent, they use more sophisticated version of grid search - so called \textit{multiple-start fast gradient sign search}. Simply speaking, this method first generates a set of pairs $(\sigma_0, s)_i, i \in \{1, \dots, K\}$ and then for each of the $i$ runs, it runs a $j$-while cycle, where in each step $j$, it increases $\sigma_j^2=\sigma_0^2+js$ to $\sigma_{j+1}^2=\sigma_0^2+(j+1)s$ and checks, whether the $\sigma$-dependent empirical certified radius in Expression~\ref{alfarra practical fcn} increases or not. If yes, they continue until $j$ is above some threshold $T$, if not, they break and report $\sigma_i$ as the $\sigma_j$ from the inner step where while cycle was broken. After obtaining $\sigma_i$ for $i \in \{1, \dots, K\}$, they choose $\sigma(x)$ to be the one, that maximizes Expression~\ref{alfarra practical fcn}. More concretely, their multiple-start fast gradient sign search algorithm looks as follows: 

\begin{algorithm}[t!] 
\caption{Instance-wise multiple-start FGSS \citep{chen2021insta}}
\begin{algorithmic}
\Function{OptimizeSigma}{$F, x, \sigma_0, M, K, T$}:
    \State $\text{generate} \hspace{1mm} (l_k, s_k), k \in \{1, \dots, K\}$
    \For{$k = 1, \dots, K$}
        \State $\text{sample} \hspace{1mm} \delta_1, \dots, \delta_M \sim \mathcal{N}(0, l_k \sigma_0^2 I)$
        \State $\phi(\sqrt{l_k} \sigma_0)=\frac{1}{M}\sum\limits_{i=1}^M F(x+\delta_i)$
        \State $\hat E_A(\sqrt{l_k} \sigma_0) = \max_C \phi(\sqrt{l_k} \sigma_0)_C$
        \State $R(\sqrt{l_k} \sigma_0)=\sqrt{l_k} \sigma_0\Phi^{-1}(\hat E_A(\sqrt{l_k} \sigma_0))$
        \State $m_k = R(\sqrt{l_k} \sigma_0)$
        \While{$l_k \in [1, T]$}
            \State $\text{sample} \hspace{1mm} \delta_1, \dots, \delta_M \sim \mathcal{N}(0, (l_k+s_k) \sigma_0^2 I)$
            \State $\phi(\sqrt{l_k+s_k} \sigma_0)=\frac{1}{M}\sum\limits_{i=1}^M F(x+\delta_i)$
            \State $\hat E_A(\sqrt{l_k+s_k} \sigma_0) = \max_C \phi(\sqrt{l_k+s_k} \sigma_0)_C$
            \State $R(\sqrt{l_k+s_k} \sigma_0)=\sqrt{l_k+s_k} \sigma_0\Phi^{-1}(\hat E_A(\sqrt{l_k+s_k}\sigma_0))$
            \If{$R(\sqrt{l_k+s_k} \sigma_0) \ge R(\sqrt{l_k} \sigma_0)$}
                \State $l_k \gets l_k+s_k$
                \State $m_k = R(\sqrt{l_k+s_k} \sigma_0)$
            \Else
                \State $\text{break}$
            \EndIf
        \EndWhile
    \EndFor
    \State $\sigma(x) = \underset{k \in \{1, \dots, K\}}{\max}m_k$
    \State \Return $\sigma(x)$
\EndFunction
\end{algorithmic}
\label{chen algo}
\end{algorithm}

It is not entirely clear from the text of \cite{chen2021insta}, how exactly are $l_k, s_k$ sampled, but it is written there, that the interval for $l$ is $[1,16]$ and for $s$ it is $(-1,1)$. Moreover, the authors don't provide the code and from the text, it seems, that they don't use lower confidence bounds during the evaluation of certified radiuses, what we consider to be a serious mistake (if really the case). However, we add some comments to this method regardless of the lower confidence bounds. 

Generally, this method possesses most of the disadvantages mentioned in Section~\ref{sec: alfarra}. They use the same function for optimization, the Expression~\ref{alfarra theoretical fcn} and its empirical version~\ref{alfarra practical fcn}. This means, that the method suffers from having several local optima, having no global optimum in general (and in most cases with limit infinity). Similarly like before, here is also no control over the correctness of the prediction, i.e. many or all local optima might lead to misclassification. 

On the other hand, in this paper authors use $M=500$ (the effective batch size), which is definitely more reasonable than $M=1$ as in \cite{alfarra2020data}. Furthermore, they use multiple initializations, making the optimization more robust and improving the chances to obtain global, or at least very good local minimum. 

However, the main problem, the curse of dimensionality yielding invalid results is even more pronounced here. Unlike the ``continuous approach'' in \cite{alfarra2020data}, here authors for each $x_0$ sample just some discrete grid (more complex, since there are more initializations) of possible values of $\sigma(x)$. For instance, if $s=1$, then the smallest possible ratio between two consecutive $l$'s in the Algorithm~\ref{chen algo} is $\sqrt{15}/4 \sim 0.97$, making it impossible to certify some $x_1$ w.r.t. $x_0$ if for both $s=1$ and $l_0 \neq l_1$ on ImageNet and also for a lot of samples on CIFAR10. Of course, the fact that $s$ is randomly sampled from $(-1, 1)$ makes this counter-argumentation more difficult, but it is, again, highly unlikely that this highly stochastic method without control over $\sigma(x)$ would yield function with sufficiently small semi-elasticity. Therefore also the impressive results of \cite{chen2021insta} are, unfortunately, scientifically invalid.

\section{Implementation Details} \label{appD: implementation details}
Even though our algorithm is rather easy, there are some perks that should be discussed before one can safely use it in practice. First, we show the actual Algorithm~\ref{my method algo}.

\begin{algorithm}[t!]
\caption{Pseudocode for certification and prediction of my method based on \cite{cohen2019certified}}
\begin{algorithmic}
\State \textit{\# evaluate } $g$ \textit{at } $x_0$
\Function{PREDICT}{$f, \sigma_0, x_0, n, \alpha$}:
    \State $\texttt{counts} \xleftarrow[]{} \text{SampleUnderNoise(}f, x_0, n, \sigma_0\text{)}$
    \State $\hat c_A, \hat c_B \xleftarrow[]{} \text{two top indices in } \texttt{counts}$
    \State $n_A, n_B \xleftarrow[]{} \texttt{counts}(c_A), \texttt{counts}(\hat c_B)$
    \If{\text{BinomPValue}$(n_A, n_A+n_B, 0.5)\le \alpha$} \Return $\hat c_A$
    \Else \hspace{1mm} \Return \text{ABSTAIN} \EndIf
\EndFunction
\newline
\State \textit{\# certify the robustness of} $g$ \textit{around} $x_0$
\Function{CERTIFY}{$f, \sigma_0, x_0, n_0, n, \alpha$}:
    \State $\texttt{counts0} \xleftarrow[]{} \text{SampleUnderNoise(}f, x_0, n_0, \sigma_0\text{)}$
    \State $\hat c_A \xleftarrow[]{} \text{top index in } \texttt{counts0}$
    \State $\texttt{counts} \xleftarrow[]{} \text{SampleUnderNoise(}f, x_0, n, \sigma_0\text{)}$
    \State $\underline{p_A} \xleftarrow[]{} \text{LowerConfBound}(\texttt{counts}[\hat c_A], n, 1-\alpha)$
    \If{$\underline{p_A}>1/2$} \Return \text{prediction } $\hat c_A$ \text{and radius } \text{ComputeCertifiedRadius(}$\sigma_0, r, N, \underline{p_A}, \text{num\_steps}$\text{)}
    \Else \hspace{1mm} \Return \text{ABSTAIN} \EndIf
\EndFunction
\newline
\Function{ComputeCertifiedRadius}{$\sigma_0, r, N, \underline{p_A}, \text{num\_steps}$}
    \State $\texttt{radiuses} \xleftarrow[]{} \text{linspace(num\_space)}$
    \For{$R$ \text{in } \texttt{radiuses}} 
        \State $\sigma_{11} \xleftarrow[]{} \sigma_0 \exp(-rR)$ 
        \State $\sigma_{12} \xleftarrow[]{} \sigma_0 \exp(rR)$ 
        \State $\texttt{xi\_bigger} \gets \xi_>(R, \sigma_{11})$
        \State $\texttt{xi\_lower} \gets \xi_<(R, \sigma_{12})$
        \If{$\max \{\texttt{xi\_bigger}, \texttt{xi\_lower}\} > 0.5$} \text{BREAK} \EndIf
    \EndFor
    \State \Return $R$
\EndFunction
\end{algorithmic}
\label{my method algo}
\end{algorithm}

Note that the function \texttt{ComputeCertifiedRadius} is a bit more complicated than depicted in Algorithm~\ref{my method algo}. We don't use a simple for-loop, but rather quite an efficient search method. 

Theoretically speaking, this algorithm works perfectly. However, in practice, it is a bit problematic. The issue is, that since we use $\sigma_{11}$ and $\sigma_{12}$, which are extremely close to $\sigma_0$ for small tested radiuses $R$, the NCCHSQ functions will get extremely high inputs, making the results numerically instable. To prevent this, we use a simple trick. Since the more extreme $\sigma_1$ will we assume in evaluation at particular distance $R$, the worse for us, we can prevent numerical issues simply by putting $\sigma_t<\sigma_0$ and $\sigma_T>\sigma_0$ to be maximal and minimal used $\sigma$'s in our evaluation, i.e. the true $\sigma$ used will be $\min \{\sigma_t, \sigma_0 \exp(-rR)\}$ and $\max \{\sigma_T, \sigma_0 \exp(rR)\}$. This way, we avoid numerical issues, because we can put $\sigma_t, \sigma_T$ to be s.t. $\frac{1}{\sigma_0^2-\sigma_t^2}$ is not too big and in the same time maintainting the correct certification thanks to the Lemma~\ref{correctness of certification procedure}. The problems of this workarounds are first that it decreases the certification power, since it assumes $\sigma_1$'s that are even worse than the theoretically guaranteed worst-case possibilities and second, more importantly, that it requires some engineering to design the $\sigma_t, \sigma_T$ designs. It is submoptimal to put one constant value for these thresholds, because the numerical problems occur at different ratio thresholds of $\sigma_1/\sigma_0$ for different class probabilities $\underline{p_A}$ \textit{and} the dimension $N$. This requires to design a specific $\sigma_t(\underline{p_A})$ and $\sigma_T(\underline{p_A})$ functions for each dimension $N$ which we want to apply. For instance, we use $\sigma_t(\underline{p_A})=0.9993+0.001\log_{10}(\overline{p_B})$, and $\sigma_T(\underline{p_A})=1/\sigma_t(\underline{p_A})$ for CIFAR10, while for MNIST we use $\sigma_t(\underline{p_A})=0.9988+0.001\log_{10}(\overline{p_B})$, and $\sigma_T(\underline{p_A})=1/\sigma_t(\underline{p_A})$. To design such functions, one needs to plot \texttt{plot\_real\_probability\_of\_a\_ball\_with\_fixed\_variances\_as\_fcn\_of\_dist} function, which computes the $\xi$ functions, for particular $N$ and several values of $\underline{p_A}$ and look, whether it computes correctly. As an example, we provide such a plots for well and ill working setups on Figure~\ref{implementation well and ill working real prob funct}. 

\begin{figure}[t!]
    \centering
    \begin{minipage}[b]{0.48\linewidth}
        \includegraphics[width=\textwidth]{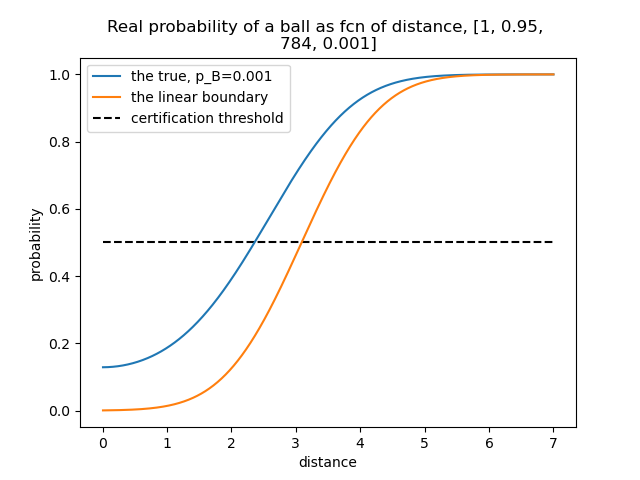}
    \end{minipage}
    \begin{minipage}[b]{0.48\linewidth}
        \includegraphics[width=\textwidth]{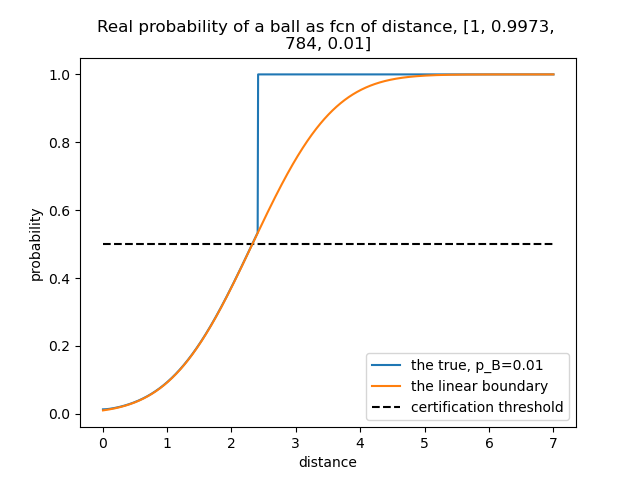}
    \end{minipage}
    \caption{Well and ill working function \newline \texttt{plot\_real\_probability\_of\_a\_ball\_with\_fixed\_variances\_as\_fcn\_of\_dist}, \newline which computes the $\xi$ functions. The coding is $[\sigma_0, \sigma_1, N, p_A]$.} 
    \label{implementation well and ill working real prob funct}
\end{figure}

Another performance trick is to not evaluate $\xi$ for each $R_i$, where $R_i$ is $i$-th grid point of evaluation, but rather evaluate sequentially $R_{i^2}$, i.e. just every $i^2$-th point, until we reach value $>0.5$ and then to search just the interval $[(i-1)^2, i^2]$, where $i$ is the first iteration for which $\xi_>(R_{i^2}, \sigma_1)\ge 0.5$.

\section{More to Experiments and Ablations} \label{appE: experiments and ablations}
Before we present our further results, we must emphasize that our certification procedure is barely any slower than that of \cite{cohen2019certified}. More specifically, given 100000 iterations of monte-carlo sampling, certification of one sample using \cite{cohen2019certified}'s algorithm on CIFAR10 takes $\sim 15$ seconds on our machine, while certification of a sample using our Algorithm~\ref{my method algo} takes $15-20$ seconds depending on the $\sigma_b, r$ setup. If at least one of $\sigma_b$ and $r$ is not small, then our method runs practically instantly. If both parameters are small, then one evaluation can take up to 5 seconds depending on the exact value of parameters and on the $\underline{p_A}$. Note, that this part of the certification is dimension-independent and therefore can run in the same time also on much higher-dimensional problems. 

Besides the actual certification, we have to compute $\sigma(x)$ for each of the test examples. This part of the algorithm is being executed before the actual certification and usually takes around 1 minute on our machine and on CIFAR10.

All in all, even in the really worst-case scenario, our method runs at most $1/3$-times longer than the old method on CIFAR10. On MNIST, the ratio between our run and the original run is higher, since MNIST is smaller-dimensional problem. However, since our part of evaluation is practically independent of the setup (except the values of $\sigma_b$ and $r$, which, however, can yield just some upper-bounded amount of slow-down), our algorithm does not bring any added asymptotic time complexity.

\subsection{How to Choose the Hyperparameters?} \label{ssec: hyperparameter selection}
Our design of $\sigma(x)$ function defined in Equation~\ref{eq: the sigma fcn} uses several hyperparameters. These are: $r$ for the rate, $m$ for the scaling, $\sigma_b$ for the base sigma and $k$ for the $k$-nearest neighbors. How do we choose these hyperparameters? 

The $m$ parameter depends on our goals. We can set it so that $\sigma(x)$ achieves lowest values at $\sigma_b$ by setting it so that it is roughly equal to the minimal distance from $k$ nearest neighbors across, for instance, training samples. Other possibility is to set it so that it is roughly equal to the average distance from $k$ nearest neighbors, to ensure that the average $\sigma(x)$ will roughly correspond to the $\sigma_b$. 

The $k$ parameter needs to be set with two objectives in mind. Firstly, it would be unwise to set it too small, because then the distance from $k$ nearest neighbors would be too noisy. On the other hand, we don't want it too high, because then it will not be changing fast enough with changing the position of $x$. The suitable value can be obtained by looking at histograms of average distances from $k$ nearest neighbors and choosing the $k$ for which the histogram is enough scattered, but it is not too small. 

The $r$ parameter needs to be chosen so that we can have some significant advantage over constant smoothing, but it cannot be too big, because otherwise the curse of dimensionality would apply. The value can be decided either by trial and error, or by plotting the certified radius given linear classifier, or from Theorem~\ref{main theorem}, setting the rate low-enough so that within the expected certified radius range, the ratio $\frac{\sigma_1}{\sigma_0}$ can't move anywhere near the theoretical thresholds implied by Theorem~\ref{main theorem}. 

The $\sigma_b$ is the base $\sigma$ and should be used according to the level of smoothing variance we want to use. More discussion on this can be found in \cite{cohen2019certified}. 

\subsection{Comparison with \citet{cohen2019certified}'s Methodology on CIFAR10 Datasets} \label{ssec: cifar10 cohen comparison}

Here, we compare \cite{cohen2019certified}'s evaluations for $\sigma= 0.12, 0.25, 0.50$ with our evaluations directly on models trained by \cite{cohen2019certified}, setting $\sigma_b=\sigma$, $r=0.005, 0.01$ and $0.015$ for $\sigma_b=\sigma=0.12$, $k=20$ and $m=5$. In this way, the levels of $\sigma(x)$ used in direct comparison will rise from the values roughly equal to \cite{cohen2019certified}'s constant $\sigma$ to higher values. The results are depicted in Figure~\ref{me vs. cohen second main comparison}.

\begin{figure}[t!]
    \centering
    \begin{minipage}[b]{0.32\linewidth}
        \includegraphics[width=\textwidth]{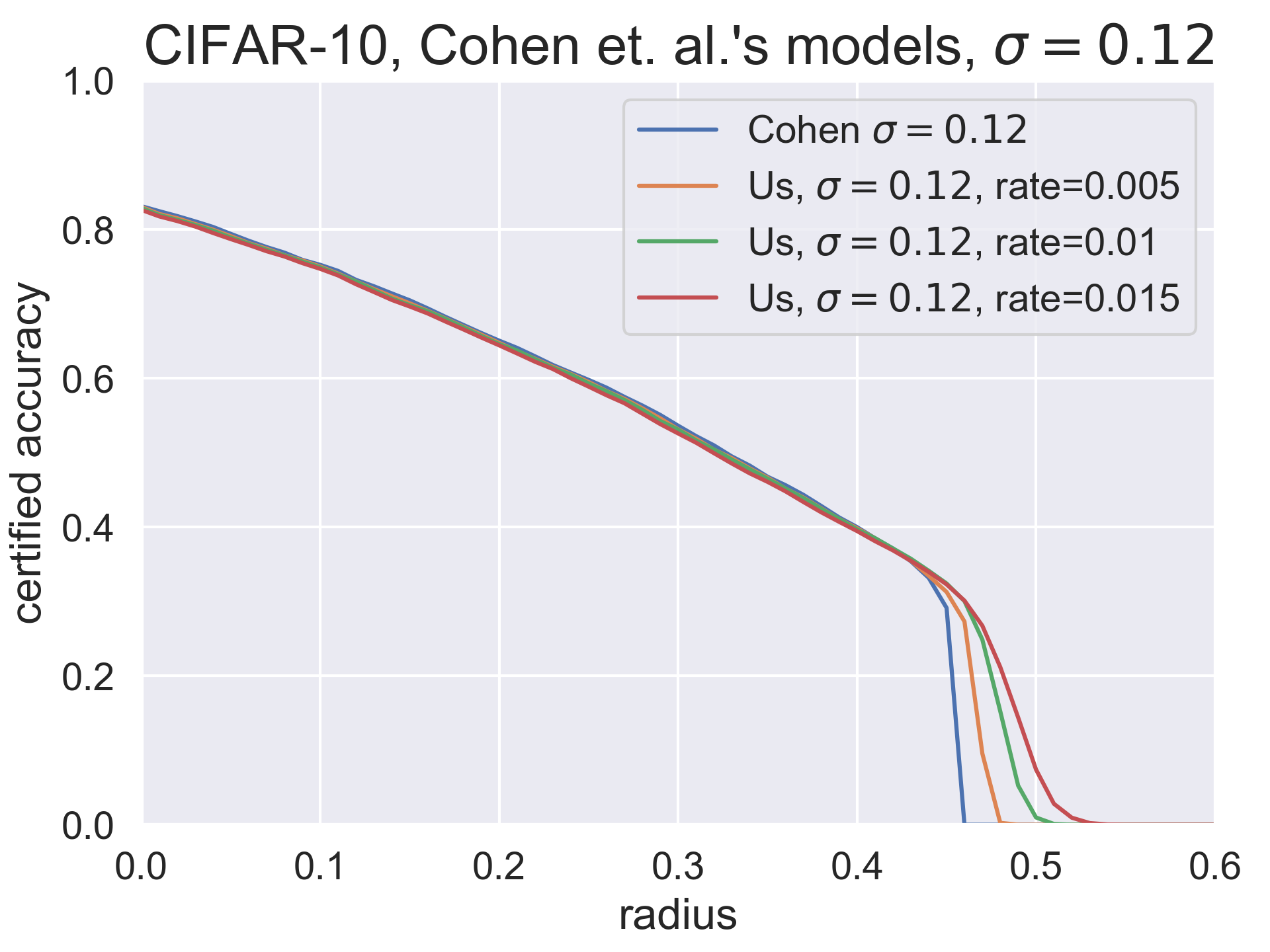}
    \end{minipage}
    \begin{minipage}[b]{0.32\linewidth}
        \includegraphics[width=\textwidth]{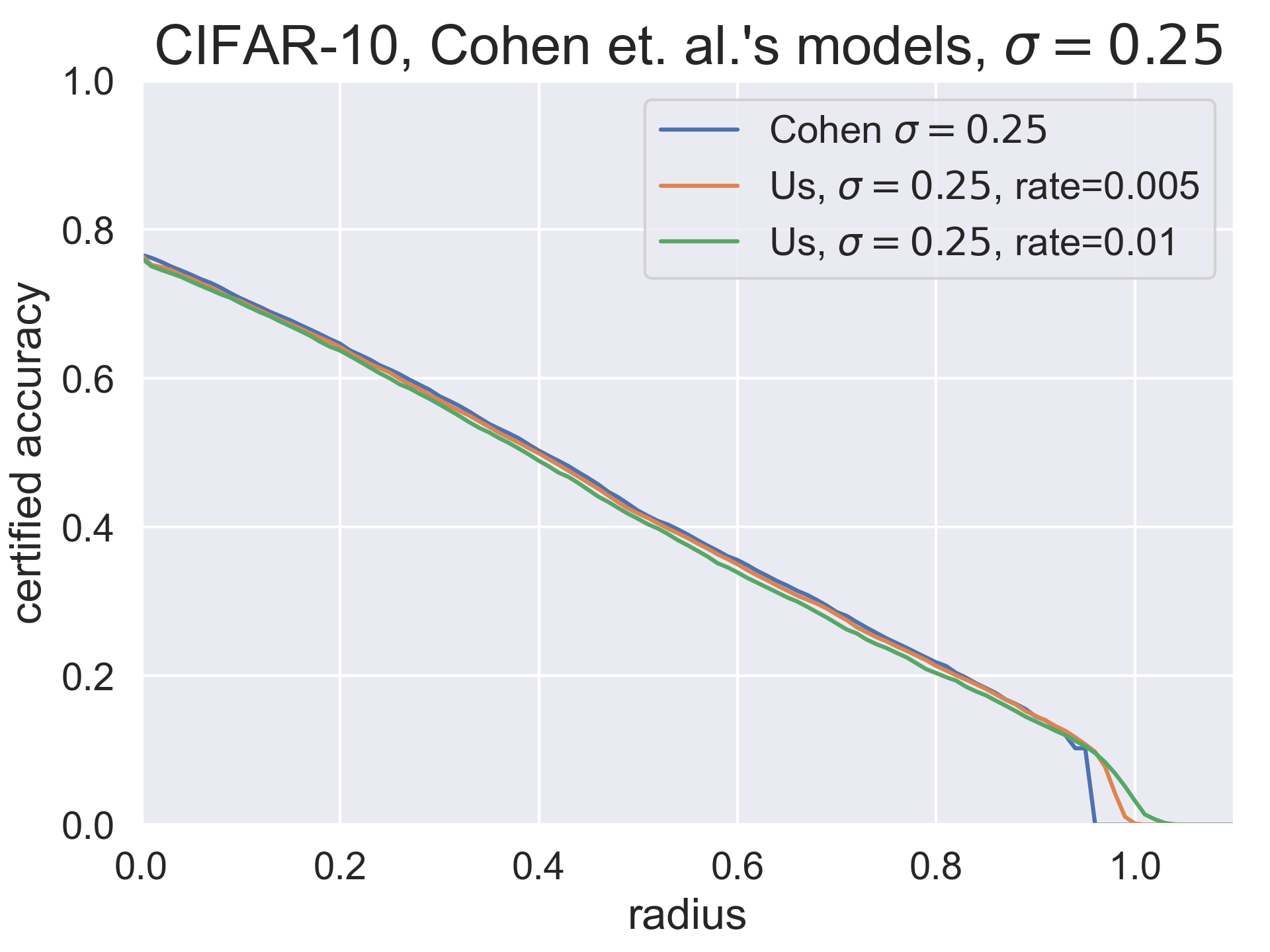}
    \end{minipage}
    \begin{minipage}[b]{0.32\linewidth}
        \includegraphics[width=\textwidth]{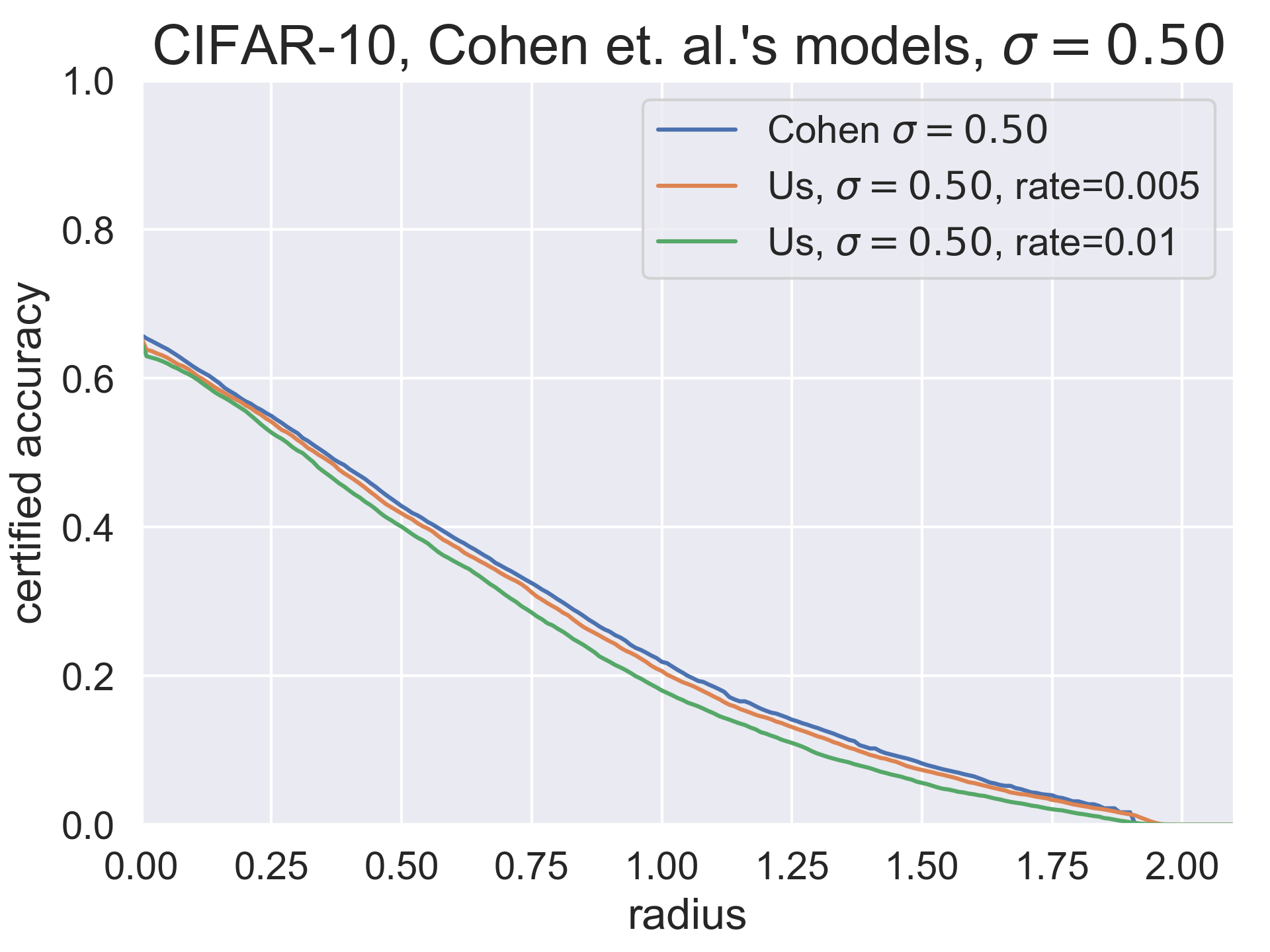}
    \end{minipage}
    \caption{Comparison of certified accuracy plots for \cite{cohen2019certified} and our work. For each plot, the same base model $f$ is used for evaluation.}
    \label{me vs. cohen second main comparison}
\end{figure}

Note, that this evaluation is being done on the models trained directly by \cite{cohen2019certified} and therefore the variance of Gaussian data augmentation is not entirely consistent with the optimal variance that should be used for non-constant $\sigma$, which should be either the same, $\sigma(x)$ or constant, but in average equal to $\sigma(x)$. The results are similar as in the Section~\ref{experiments}. Note, that for $\sigma_b=\sigma=0.50$, the curse of dimensionality becomes most pronounced, as explained in Appendix~\ref{appC: more on theory}. Further, we provide the Tables~\ref{accuracy table for Cohen's comparison}, \ref{standard devs of class-wise accuracies}, where the clean accuracies and class-wise standard deviations are displayed.

\begin{table}[t!]
\centering
\begin{tabular}{||c||c|c|c||} 
\hline\hline
 & $\sigma=0.12$ & $\sigma=0.25$ & $\sigma=0.50$ \\ 
\hline\hline
$r=0.00$ & 0.831 & 0.766 & 0.658 \\ 
\hline
$r=0.005$ & 0.830 & 0.766 & 0.654 \\
\hline
$r=0.01$ & 0.828 & 0.762 & 0.649 \\
\hline
$r=0.015$ & 0.826 & - & - \\
\hline\hline
\end{tabular}
\vspace{2mm}
\caption{Clean accuracies for Cohen's models and our non-constant $\sigma(x)$ models.}
\label{accuracy table for Cohen's comparison}
\end{table}

\begin{table}[t!]
\centering
\begin{tabular}{||c||c|c|c||} 
\hline\hline
 & $\sigma=0.12$ & $\sigma=0.25$ & $\sigma=0.50$ \\ 
\hline\hline
$r=0.00$ & 0.084 & 0.108 & 0.131 \\ 
\hline
$r=0.005$ & 0.086 & 0.112 & 0.135 \\
\hline
$r=0.01$ & 0.088 & 0.119 & 0.142 \\
\hline\hline
\end{tabular}
\vspace{2mm}
\caption{Standard deviations of class-wise accuracies for different levels of $\sigma$ and $r$.}
\label{standard devs of class-wise accuracies}
\end{table}

The results are, again, similar as in the section~\ref{experiments}. 

\subsection{Comparison with \citet{cohen2019certified}'s Methodology on MNIST Datasets}
Here, we present similar comparison as in Subsection~\ref{ssec: cifar10 cohen comparison}, but on MNIST and with models trained by us. Again, the setup is similar as in the Section~\ref{experiments}. We compare $\sigma=\sigma_b= 0.12, 0.25, 0.50$ with test-time rates $r = 0.005, 0.01, 0.02, 0.05$ \textit{and} train-time level of $\sigma$ again equal to $\sigma=\sigma_b$. It is important to note, that we use different normalization constant $m$ in the MNIST case. In CIFAR10, we set $m=5$, in MNIST, the suitable $m$ is $1.5$. This way we assure, that the smallest $\sigma(x)$ values in the test set will roughly equal the $\sigma_b=\sigma$. The certified accuracy plots are depicted on Figure~\ref{me vs. cohen main comparison MNIST}. We also add the clean accuracy table and class-wise clean accuracies standard deviation table (\ref{accuracy table for Cohen's comparison, MNIST}, \ref{standard devs of class-wise accuracies, MNIST}).

\begin{figure}[t!]
    \centering
    \begin{minipage}[b]{0.32\linewidth}
        \includegraphics[width=\textwidth]{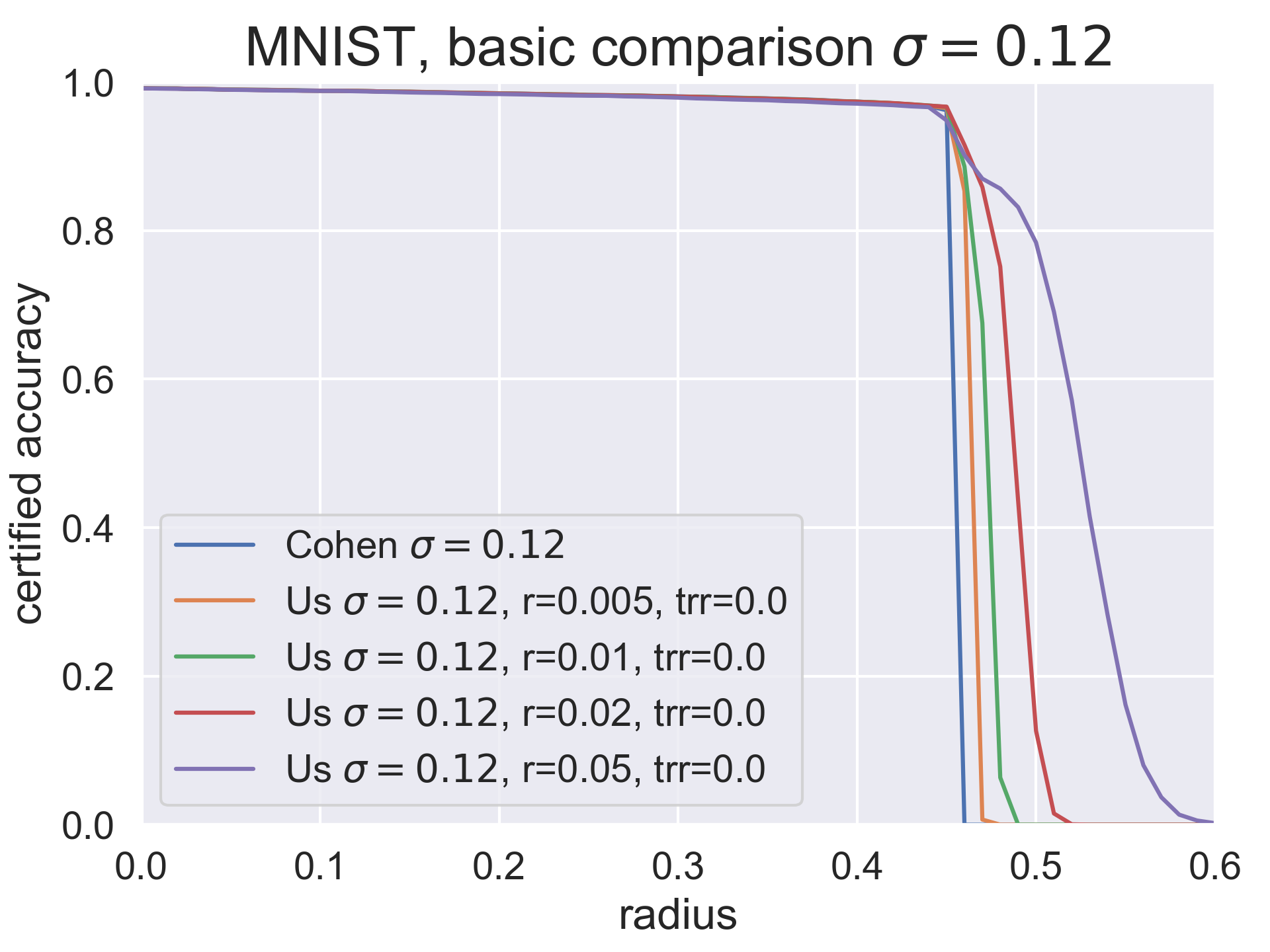}
    \end{minipage}
    \begin{minipage}[b]{0.32\linewidth}
        \includegraphics[width=\textwidth]{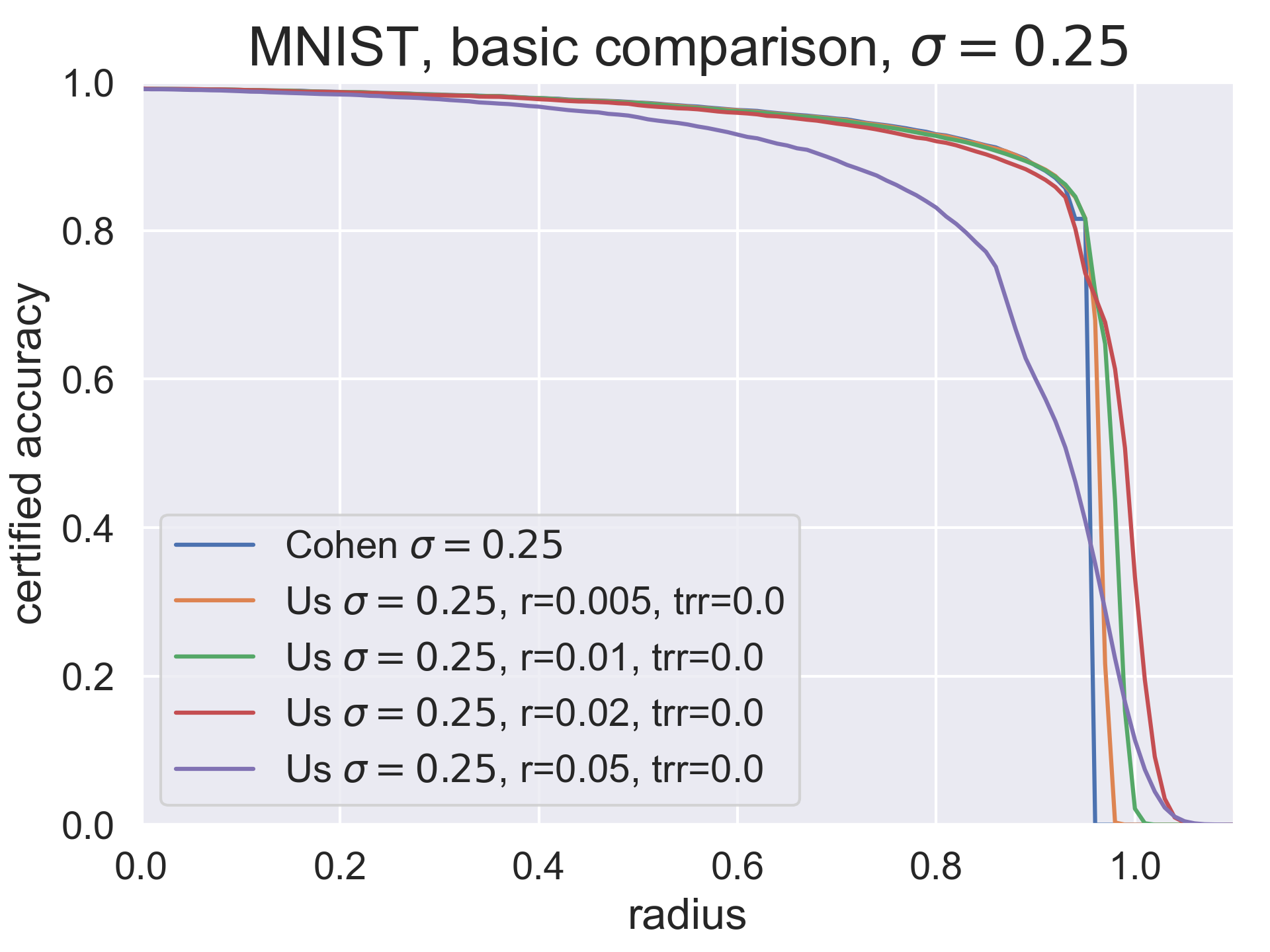}
    \end{minipage}
    \begin{minipage}[b]{0.32\linewidth}
        \includegraphics[width=\textwidth]{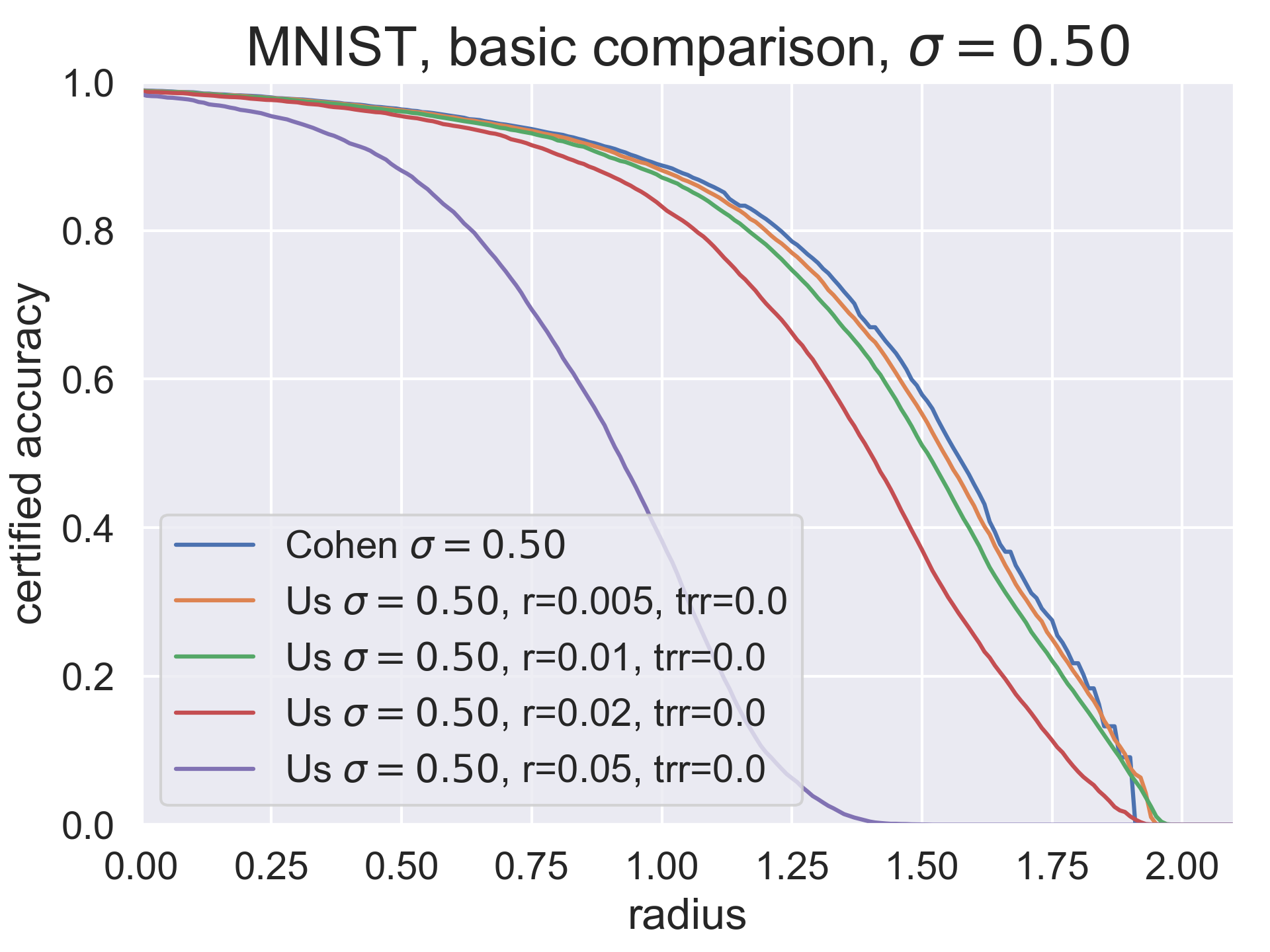}
    \end{minipage}
    \caption{Comparison of certified accuracy plots for \cite{cohen2019certified} and our work, MNIST. For each plot, the same base model $f$ is used for evaluation. The term $trr$ stands for \textit{train-time rate}, will be discussed later and can be ignored now.}
    \label{me vs. cohen main comparison MNIST}
\end{figure}

\begin{table}[t!]
\centering
\begin{tabular}{||c||c|c|c||} 
\hline\hline
 & $\sigma=0.12$ & $\sigma=0.25$ & $\sigma=0.50$ \\ 
\hline\hline
$r=0.00$ & 0.9913 & 0.9910 & 0.9888 \\ 
\hline
$r=0.005$ & 0.9914 & 0.9912 & 0.9885 \\
\hline
$r=0.01$ & 0.9914 & 0.9910 & 0.9887 \\
\hline
$r=0.02$ & 0.9914 & 0.9912 & 0.9876 \\
\hline
$r=0.05$ & 0.9914 & 0.9906 & 0.9836 \\
\hline\hline
\end{tabular}
\vspace{2mm}
\caption{Clean accuracies for Cohen's models and our non-constant $\sigma(x)$ models on MNIST.}
\label{accuracy table for Cohen's comparison, MNIST}
\end{table}

\begin{table}[t!]
\centering
\begin{tabular}{||c||c|c|c||} 
\hline\hline
 & $\sigma=0.12$ & $\sigma=0.25$ & $\sigma=0.50$ \\ 
\hline\hline
$r=0.00$ & 0.677 & 0.729 & 0.909 \\ 
\hline
$r=0.005$ & 0.659 & 0.735 & 0.905 \\
\hline
$r=0.01$ & 0.659 & 0.722 & 0.9318 \\
\hline
$r=0.02$ & 0.659 & 0.713 & 0.960 \\
\hline
$r=0.05$ & 0.715 & 0.796 & 1.159 \\
\hline\hline
\end{tabular}
\vspace{2mm}
\caption{Standard deviations of class-wise accuracies for different levels of $\sigma$ and $r$. The printed values are multiples of 100 of the real standard deviations.}
\label{standard devs of class-wise accuracies, MNIST}
\end{table}

All the results are, again, very similar to those presented in Section~\ref{experiments}, even though the gain in certified accuracies is marginally worse, since our evaluations run on models trained with in average smaller train-time data-augmentation standard deviation $\sigma_b=\sigma$. 

\subsection{Investigation of the Effect of Training with Input-dependent Gaussian Augmentation} \label{ssec: cifar10 input-dependent training effect}
It has been shown by many works, that apart from a good test-time certification method, also the appropriate training plays a very important role in the final robustness of our smoothed classifier $g$. Already \cite{cohen2019certified} realize this and propose to train with gaussian data augmentation with constant $\sigma$. They experiment with different levels of $\sigma$ during training and conclude that training with the same level of $\sigma$ that will be later used in the test time is usually the most suitable option. 

The question of best-possible training to boost the certified robustness didn't stay without the interest of different researchers. Both \cite{zhai2020macer} and \cite{salman2019provably} try to improve the way of training and propose two different, yet interesting and effective training methods. While
\cite{zhai2020macer} manage to incorporate the adversarial robustness into the training loss function, therefore training directly for the robustness, \cite{salman2019provably} propose to use adversarial training to achieve more robust classifiers. 

Both \cite{alfarra2020data} and \cite{chen2021insta} already propose to use training with input-dependent $\sigma$ as the variance of gaussian data augmentation. Both of them proceed similarly as during test time - to obtain training $\sigma(x)$, they optimize for such, that would maximize the certified accuracy of training samples.

In this section, we propose and test our own training method. We propose to use again gaussian data augmentation with input-dependent $\sigma(x)$, but we suggest to use the simple $\sigma(x)$ defined in Equation~\ref{eq: the sigma fcn}. In other words, we suggest using the same $\sigma(x)$ during training as during testing (up to parametrization, which might differ).
 
Note, that, unlike the certification, the training procedure does not require any mathematical analysis nor certification. It is totally up to us how we train the base classifier $f$ and the way of training does not influence the validity of subsequent certification guarantees during test time. However, it is good to have a reasonable training procedure, because otherwise, we would achieve a satisfactory model neither in terms of clean accuracy nor in terms of adversarial robustness. 

In the subsequent analysis, we evaluate and compare our certification procedures on models trained with different training parametrizations. For this particular section, we run the comparison only on the CIFAR10 dataset. For each test-time $\sigma_b, r$, we evaluate our method with these parameters on base models $f$ trained with the same $\sigma_b$, but different level of \textit{training rate} $trr$. The training rate $trr$ plays exactly the same role as the evaluation rate $r$ but is used exclusively during training. Note, that this makes our $\sigma(x)$ different during training and testing since it is parametrized with different rates. 

On the Figure~\ref{me, the trainings} we plot evaluations on CIFAR10 of our method for rate $0.01$, all levels of $\sigma_b = 0.12, 0.25, 0.50$ and each of these test-time setups is evaluated on 4 different levels of train-time rate \\ $trr = 0.0, 0.01, 0.04, 0.1$.

\begin{figure}[t!]
    \centering
    \begin{minipage}[b]{0.32\linewidth}
        \includegraphics[width=\textwidth]{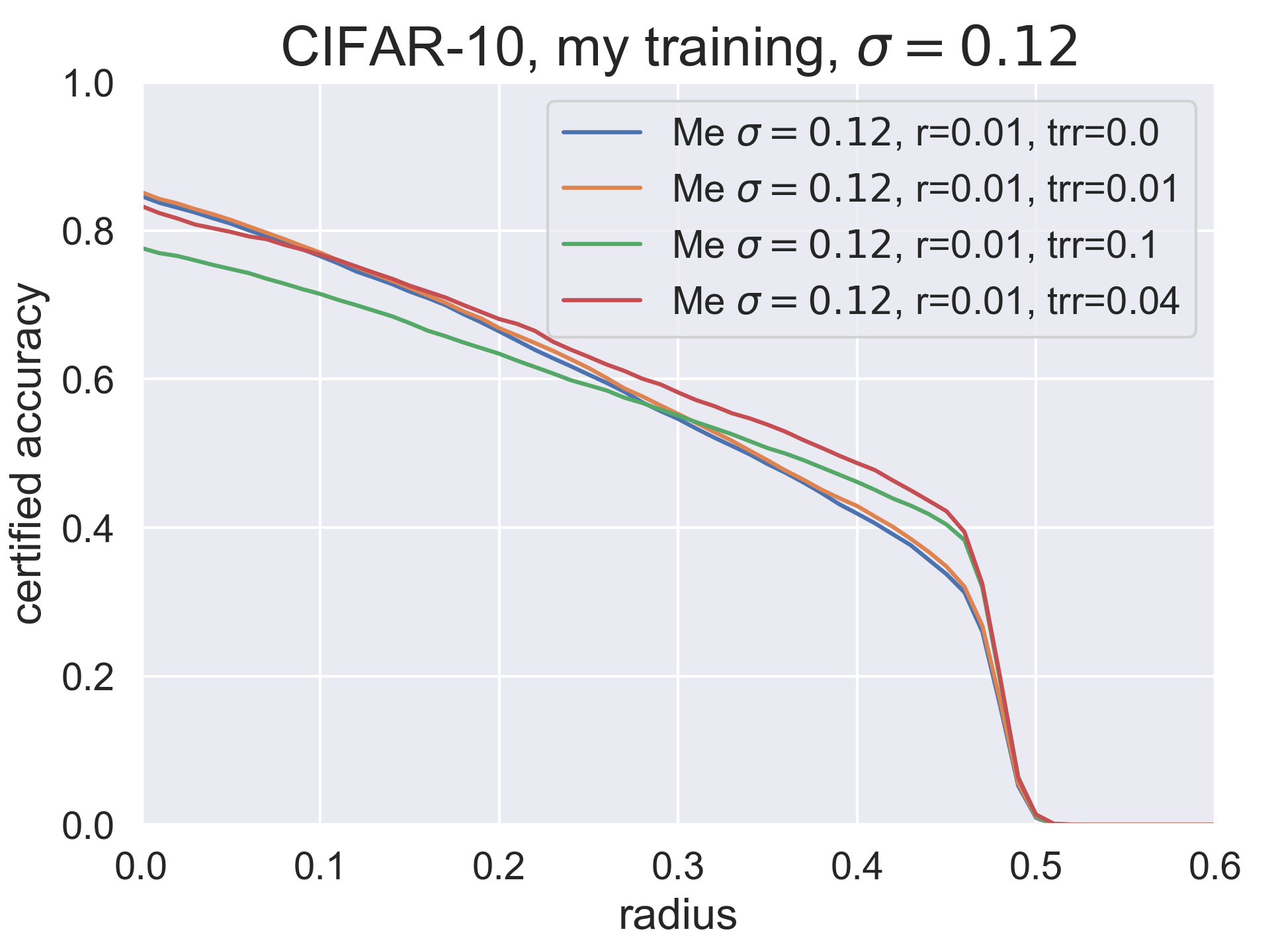}
    \end{minipage}
    \begin{minipage}[b]{0.32\linewidth}
        \includegraphics[width=\textwidth]{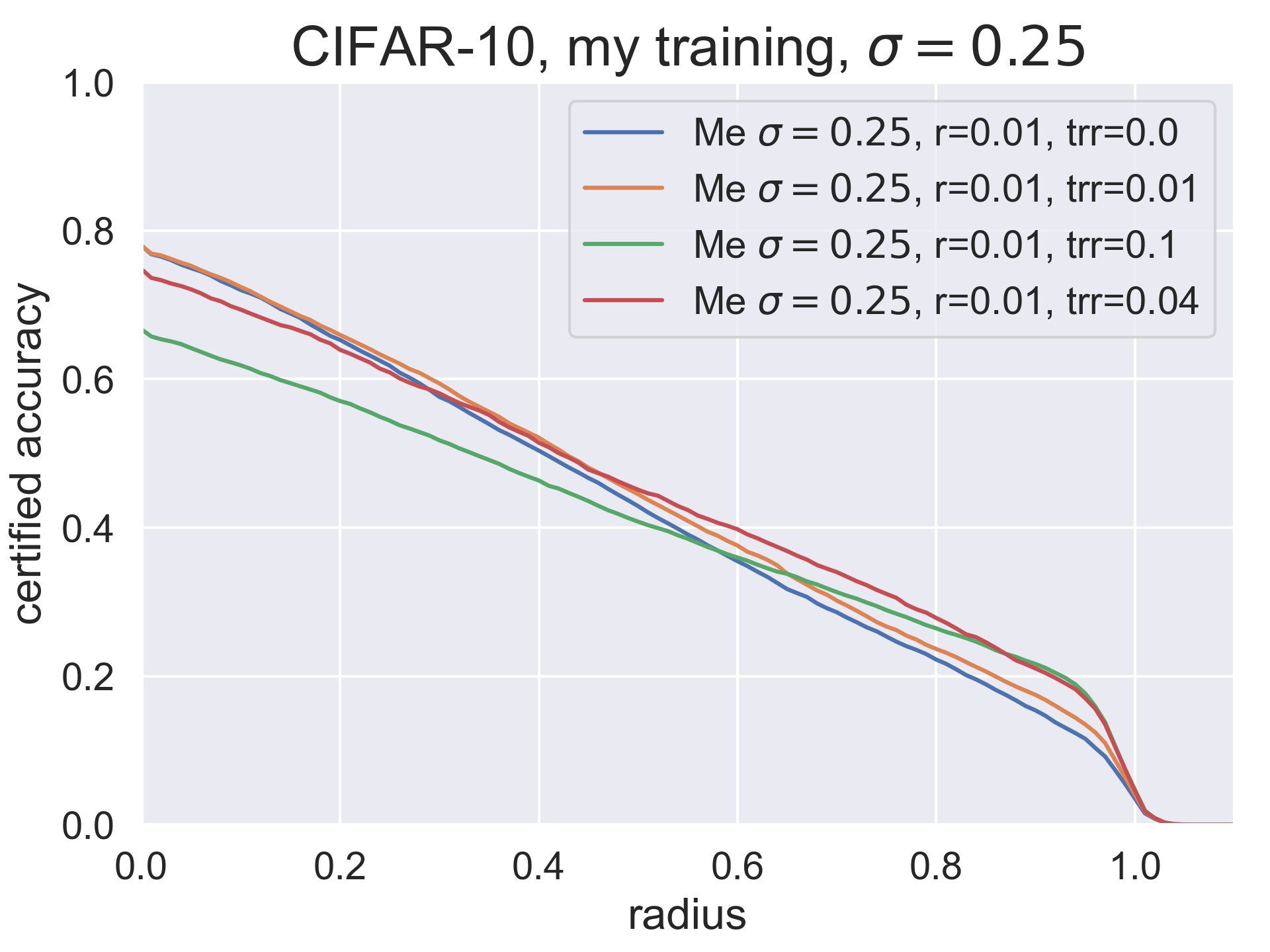}
    \end{minipage}
    \begin{minipage}[b]{0.32\linewidth}
        \includegraphics[width=\textwidth]{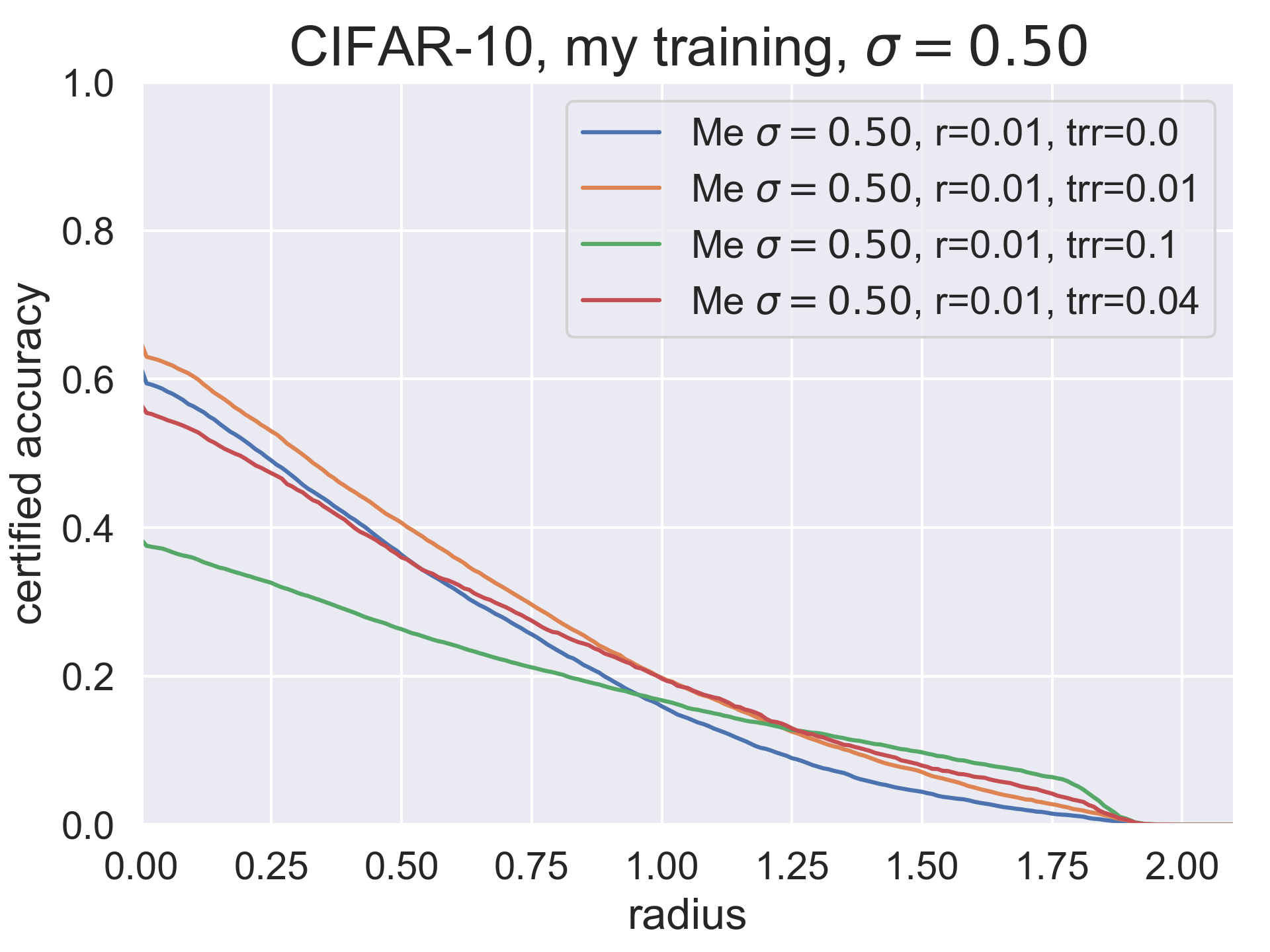}
    \end{minipage}
    \caption{The certified accuracies of our procedure on CIFAR10 for $\sigma_b= 0.12, 0.25, 0.50$, rate $r=0.01$ and training rate $trr= 0.0, 0.01, 0.04, 0.1$.} 
    \label{me, the trainings}
\end{figure}

From the results, we judge, that our training procedure works satisfactorily well. It can generally outperform the constant $\sigma$ training, yet the standard accuracy vs.\ robustness trade-off is present in some cases. If we train with small train-time rate, the improvement of the certified accuracies is not pronounced (the case for $\sigma_b=\sigma=0.50$ is slightly misleading, since such a configuration is just a result of the variance of clean accuracy w.r.t different traning runs) enough, but we also don't lose almost any clean accuracy. Increasing the rate to $trr=0.04$ results in much more pronounced improvements in high certified accuracies, yet also comes at a prize of clean accuracy drop, especially for large $\sigma$ levels. Even bigger training rate, such as $trr=0.1$ seems to be too big and does not bring almost any improvement over the rate $trr=0.04$, yet loses a large amount of clean accuracy. 

These results suggest, that the input-dependent training with a carefully chosen training rate for $\sigma(x)$ can lead to significant improvements in certifiable robustness. However, it is important to note, that the optimal $trr$ seems to be dependent on the $\sigma_b$, therefore for each value of $\sigma_b$, some effort has to be invested to find the optimal hyperparameters. 

Besides, we were also interested, whether using an input-dependent $\sigma(x)$ during training influences the class-wise accuracy balance. In Table~\ref{training class-wise std devs} we report the standard deviations of class-wise accuracies. 

\begin{table}[t!]
\centering
\begin{tabular}{||c||c|c|c||} 
\hline\hline
 & $\sigma=0.12$ & $\sigma=0.25$ & $\sigma=0.50$ \\ 
\hline\hline
$trr=0.00$ & 0.084 & 0.107 & 0.153 \\ 
\hline
$trr=0.01$ & 0.078 & 0.099 & 0.126 \\
\hline
$trr=0.04$ & 0.068 & 0.081 & 0.117 \\
\hline
$trr=0.1$ & 0.088 & 0.099 & 0.230 \\
\hline\hline
\end{tabular}
\vspace{2mm}
\caption{Standard deviations of class-wise accuracies for different levels of $\sigma$ and $trr$, under constant rate $r=0.01$.}
\label{training class-wise std devs}
\end{table}

We can observe, that unlike the pure input-dependent evaluation, the input-dependent training is partially capable of mitigating the effects of the shrinking. For instance, the $trr=0.04$ for $\sigma_b=0.12$ provides obvious improvement in establishing class-wise balance. Similarly successful are trainings with $trr=0.01$ for $\sigma_b=0.12$ and both $trr=0.01, 0.04$ for $\sigma_b=0.25$. Also for $\sigma=0.50$ the mitigation is present for small-enough training rates. However, we must emphasize, that if we use too big training rate, the disbalance between class accuracies will be re-established and in some cases even magnified. Therefore, we must be careful to choose the appropriate training rate for the $\sigma_b, r$. 

\subsection{Why Do We Not Compare with the Current State-of-the-art?} %Briefly speaking -- we could, but we don't consider it necessary.
Since we claim one type of improvement over \citet{cohen2019certified}'s model (experiment-wise) and don't develop new state-of-the-art training method, we don't find it necessary to compare with methods of \citet{salman2019provably} and \citet{zhai2020macer}. It is obvious that we would outperform these methods in the question of \textit{certified accuracy waterfalls}, since these methods focus on the training phase. Since we do not outperform RS by \citet{cohen2019certified} neither in terms of the clean accuracies nor in terms of class-wise accuracies, it is not our belief that we would outperform the two modern methods in these metrics. Moreover, we find the comparison with \citet{cohen2019certified}'s work most appropriate, since we extend the theory built by them. 

\subsection{Ablations}

Even though our results so far might look impressive, we can't claim that it is fully due to our particular method until we exclude the possibility, that some different effects play an essential role in the improvement over \cite{cohen2019certified}'s work. 

To investigate, whether our particular method dominates the contribution to the performance boost, we conduct several ablation studies - first, we study the variance of our evaluations and trainings, second, we study the effect of input-dependent test-time randomized smoothing, and third, we study the effect of input-dependent train-time data augmentation. 

\subsubsection{Variance of the Evaluation} \label{ssec: ablation eval variance}

To find out, whether there is a significant variance in the evaluation of certified radiuses, we conduct a simple experiment - we train a single model on CIFAR10 and evaluate our method on this model for the very same setup of parameters multiple times. This way, the only present stochasticity is in the Monte-Carlo sampling, which influences the evaluation of certified radiuses. We pick the parameters as follows: $\sigma_b=0.50, r=0.01, trr=0.0$, since the $\sigma_b=0.50$ turns out to have biggest variance in the training. The results are depicted in Figure~\ref{ablation eval var}.

\begin{figure}[t!]
    \centering
    \includegraphics[width=0.60\textwidth]{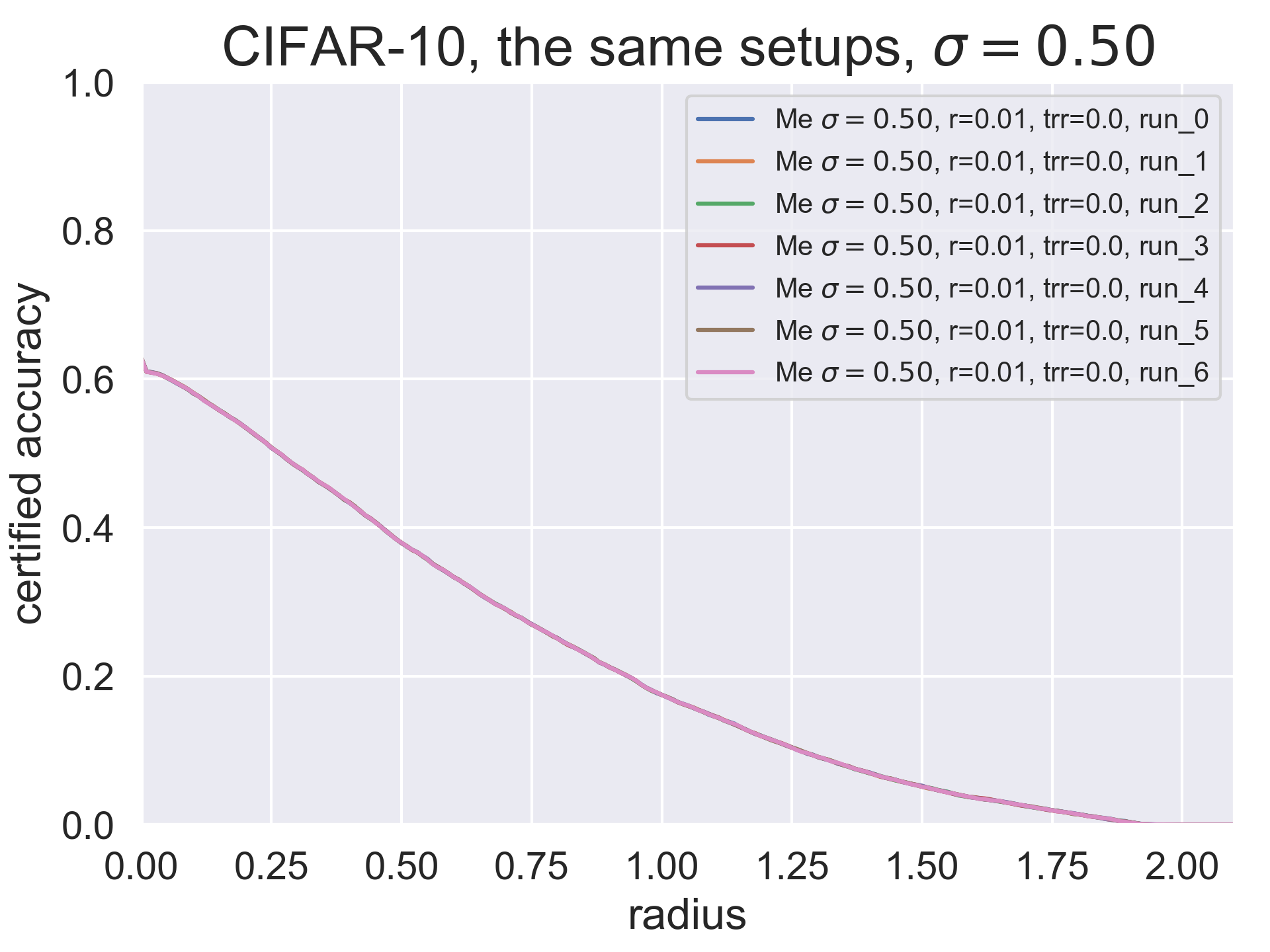}
    \caption{The variance of evaluation. Parameters are $\sigma_b=0.50, r=0.01, trr=0.0$, the evaluated model is the same for all runs. There are 7 runs on CIFAR10.}
    \label{ablation eval var}
\end{figure}

From the results, it is obvious that the variance in the evaluation phase is absolutely negotiable. Therefore, there is no need to run the same evaluation setup more times. 
\subsubsection{Variance of the Training} \label{ssec: ablation train variance}

To estimate the variance of the training, we train several models for one specific training setup and evaluate them with the same evaluation setup (knowing, that there is no variance in the evaluation phase, this is equivalent to measuring directly the training variance). We pick our classical non-constant $\sigma(x)$ for the evaluation, but we train with constant variance data augmentation. The concrete parameters we work with are: $\sigma_b \in \{0.12, 0.25, 0.50\}, r=0.01, trr=0.0$ and we run 9 trainings for each of these parameter configurations. Then we run full certification to not only see the variance in clean accuracy, but also the variance in the certified radiuses. The results are depicted in Figure~\ref{ablation train var}.

\begin{figure}[t!]
    \centering
    \begin{minipage}[b]{0.32\linewidth}
        \includegraphics[width=\textwidth]{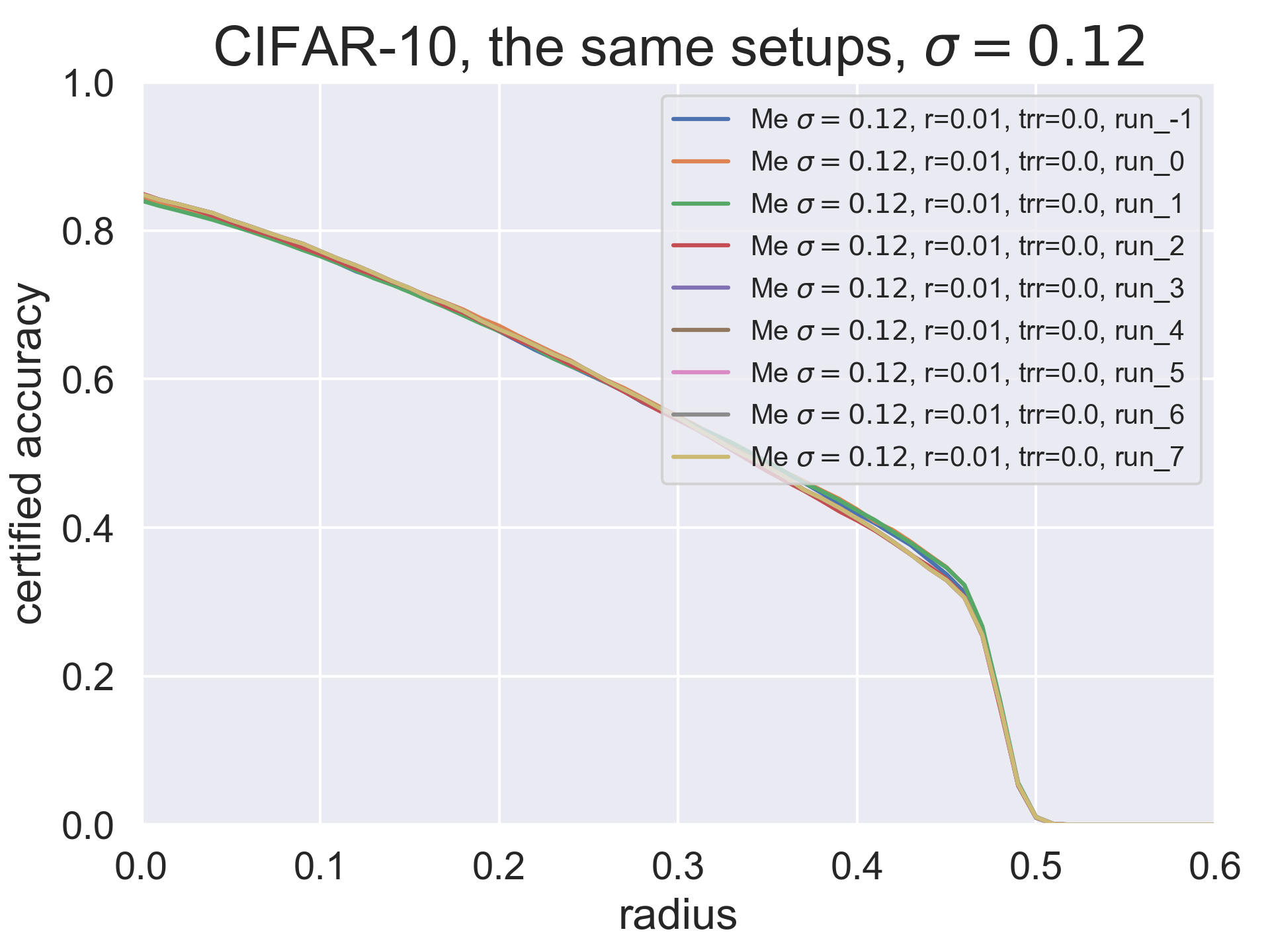}
    \end{minipage}
    \begin{minipage}[b]{0.32\linewidth}
        \includegraphics[width=\textwidth]{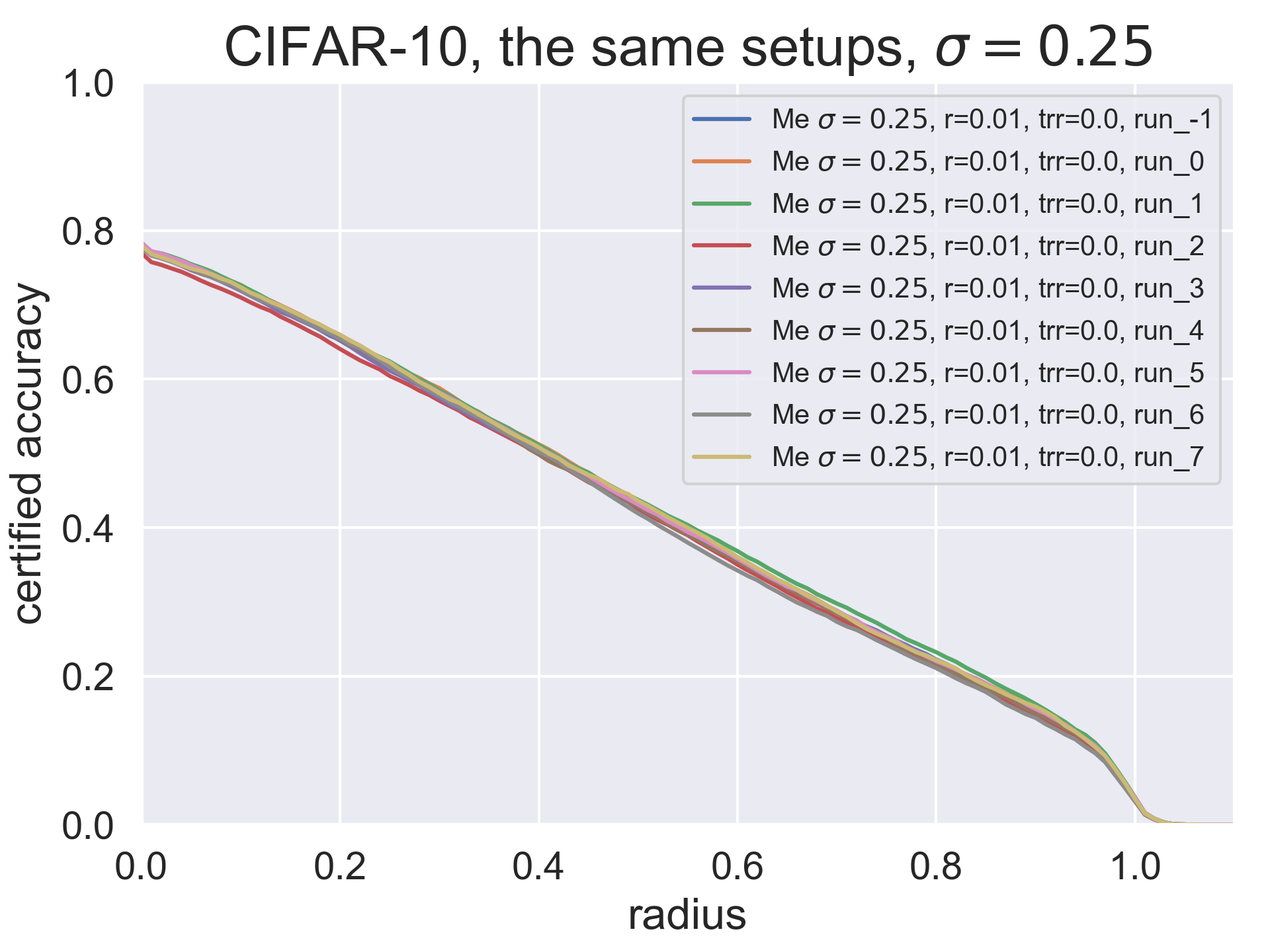}
    \end{minipage}
    \begin{minipage}[b]{0.32\linewidth}
        \includegraphics[width=\textwidth]{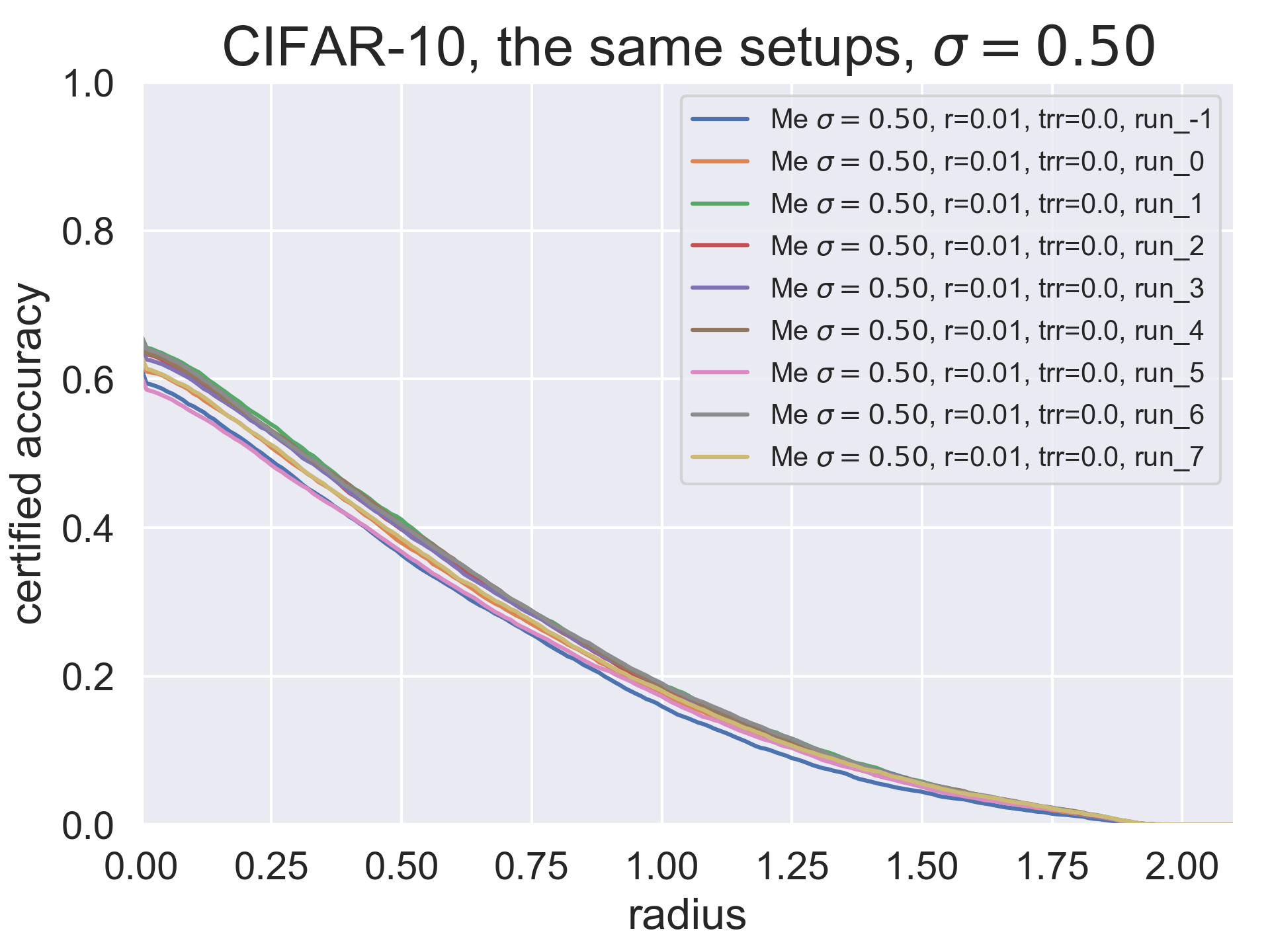}
    \end{minipage}
    \caption{The certified accuracies of our procedure on CIFAR10 for $\sigma_b= 0.12, 0.25, 0.50$, rate $r=0.01$ and training rate $trr= 0.0$ evaluated on 9 different trained models for each of the setups.} 
    \label{ablation train var}
\end{figure}

From the figures we see, that the variance of the training is strongly $\sigma_b$-dependent. Most volatile clean accuracy is present for the case $\sigma_b=0.50$. However, fortunately, the biggest variability is present for the clean accuracy and the curves seem to be less scattered in the areas of high certified radiuses. The concrete standard deviations of clean accuracies are in Table~\ref{ablation training var acc table}. The standard deviations of clean accuracies for MNIST dataset and the same parameters are in Table~\ref{ablation training var acc table MNIST}. 

\begin{table}[t!]
\centering
\begin{tabular}{||c||c|c|c||} 
\hline\hline
 & $\sigma=0.12$ & $\sigma=0.25$ & $\sigma=0.50$ \\ 
\hline\hline
accuracy & 0.61\% & 0.40\% & 1.86\% \\ 
\hline
abstention rate & 0.17\% & 0.34\% & 0.59\% \\
\hline
misclassification rate & 0.60\% & 0.24\% & 1.48\% \\
\hline\hline
\end{tabular}
\vspace{2mm}
\caption{Standard deviations of clean accuracies, abstention rates and misclassification rates for 9 runs of each parameter configuration on CIFAR10.}
\label{ablation training var acc table}
\end{table}

\begin{table}[t!]
\centering
\begin{tabular}{||c||c|c|c||} 
\hline\hline
 & $\sigma=0.12$ & $\sigma=0.25$ & $\sigma=0.50$ \\ 
\hline\hline
accuracy & 0.036\% & 0.042\% & 0.044\% \\ 
\hline
abstention rate & 0.037\% & 0.027\% & 0.058\% \\
\hline
misclassification rate & 0.043\% & 0.029\% & 0.021\% \\
\hline\hline
\end{tabular}
\vspace{2mm}
\caption{Standard deviations of clean accuracies, abstention rates and misclassification rates for 8 runs of each parameter configuration on MNIST.}
\label{ablation training var acc table MNIST}
\end{table}

Since the differences in accuracies of different methods are very subtle, it is hard to obtain statistically trustworthy results. For instance, given, that the standard deviation $0.4\%$ is the true standard deviation of the $\sigma_b=0.25$ runs, we would need 16 runs to decrease it to a standard deviation of $0.1\%$, which might be considered to be precise-enough. To do the same in the case of $\sigma_b=0.50$ on CIFAR10, we would roughly need 400 runs to decrease the standard deviation below $0.1\%$. Therefore, the results we provide in the subsequent subsections, being the average of ``just'' 8 runs, have to be taken just modulo variance in the results, which might still be considerable. 

\subsubsection{Effect of Input-dependent Evaluation} \label{ssec: ablation eval}

In this ablation study, we compare the certification method for particular $\sigma_b, r=0.01, trr=0.0$ with the constant-$\sigma$ certification method with $C\sigma_{b}, r=0.0, trr=0.0$, where $C$ is an appropriate constant. The motivation behind such an experiment is, that our $\sigma(x)$ is generally bigger than $\sigma_b$, but originally, we compare this method to constant $\sigma=\sigma_b$ evaluation. Therefore, in average, samples in our method enjoy bigger values of $\sigma(x)$. Natural question is, whether we cannot obtain the same performance boost using just the constant $\sigma$ method with $C\sigma_{b}>\sigma_b$ set to such value, which roughly corresponds to the average of $\sigma(x_i)$ for $x_i, i \in \{1, \dots, T\}$ being the test set. The problem of using bigger $C\sigma_{b}$ is, that we encounter performance drop and more severe case of shrinking, but we need to check, to what extent is the performance drop present in the input-dependent $\sigma(x)$ method. Comparing the performance drops of larger constant $C\sigma_{b}$ and input-dependent $\sigma(x)$, which is in average larger (but in average the same as the $C\sigma_{b}$), we will be able to answer, to what degree is the usage of input-dependent $\sigma(x)$ really justified. If we remind ourselves, that $$\sigma(x)=\sigma_b \exp\left(r \left(\frac{1}{k} \left( \sum\limits_{x_i \in \mathcal{N}_k(x)} \norm{x-x_i}\right) -m\right)\right),$$ then we see, that the constant $C$ we are searching for is the average (or rather median) value of $$\exp\left(r \left(\frac{1}{k} \left( \sum\limits_{x_i \in \mathcal{N}_k(x)} \norm{x-x_i}\right) -m\right)\right).$$

Fortunately, empirically, the mean and median of the above expression are roughly equal for both CIFAR10 and MNIST, so we are not forced to choose between them. For $r=0.01, m=5$, we choose the rounded value of $C=\exp(0.05)$ on CIFAR10. For $r=0.01, m=1.5$ as in MNIST, the constant is set to $C=1.035$. In the end, the values of $C\sigma$ used in this experiment are $C\sigma= 0.126, 0.263, 0.53$ for CIFAR10 and $C\sigma= 0.124, 0.258, 0.517$ for MNIST. To obtain a fair comparison, though, we evaluate the input-dependent $\sigma(x)$ evaluation strategy on models trained with constant $C\sigma_b$ standard deviation of gaussian augmentation. This is because this level of $\sigma$ is equal to the mean value of the $\sigma(x)$ and we believe, that such a training data augmentation standard deviation is more consistent with our $\sigma(x)$ function. We provide the plots of single-run evaluations of certified accuracies for CIFAR10 in Figure~\ref{ablation non-constant sigma effect} and for MNIST in Figure~\ref{ablation non-constant sigma effect MNIST}. The models on which we evaluate differ because for the increased constant $\sigma$ evaluations we needed to also use an increased level of the training data augmentation variance. 

\begin{figure}[t!]
    \centering
    \begin{minipage}[b]{0.32\linewidth}
        \includegraphics[width=\textwidth]{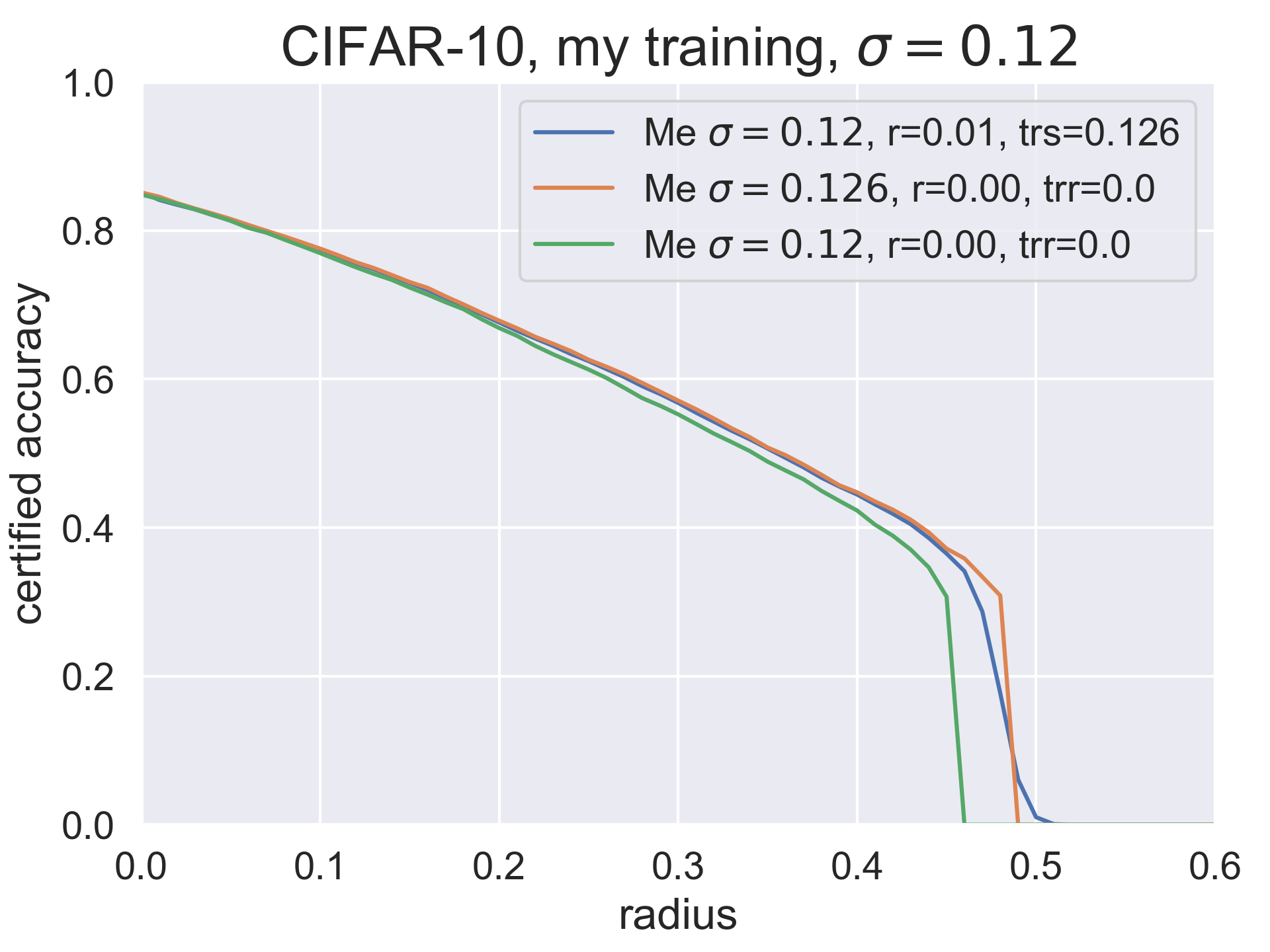}
    \end{minipage}
    \begin{minipage}[b]{0.32\linewidth}
        \includegraphics[width=\textwidth]{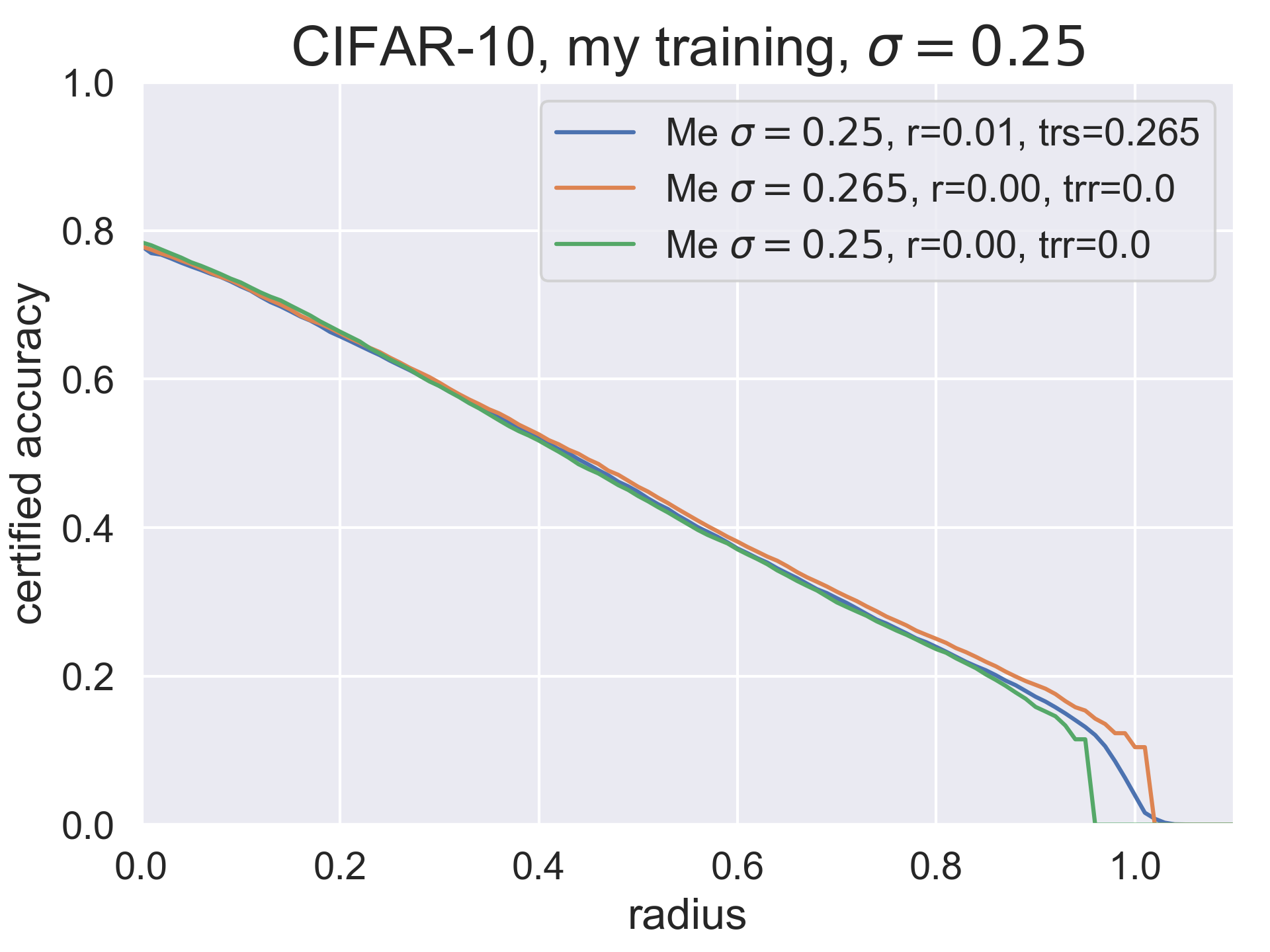}
    \end{minipage}
    \begin{minipage}[b]{0.32\linewidth}
        \includegraphics[width=\textwidth]{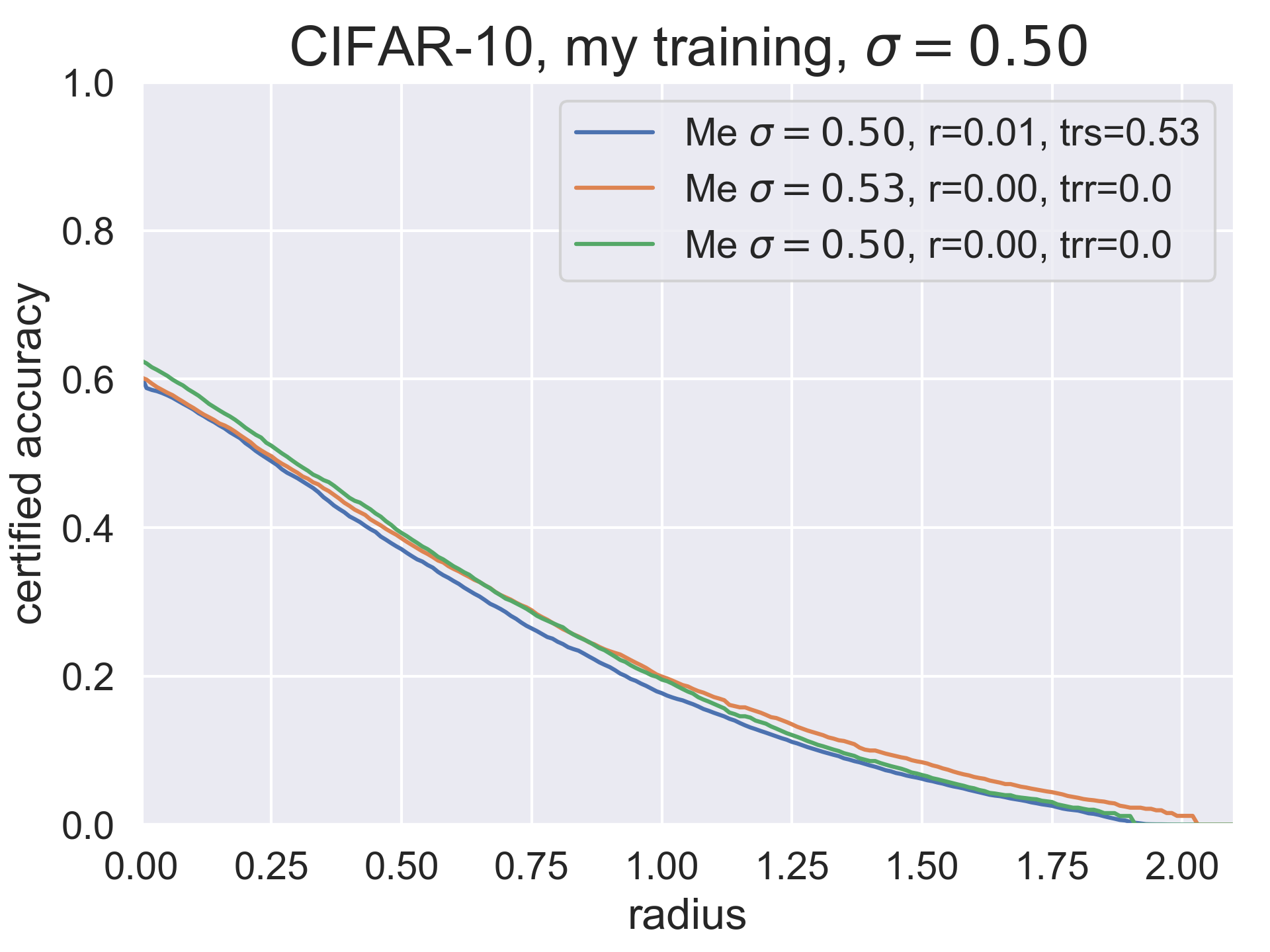}
    \end{minipage}
    \caption{The certified accuracies of our procedure on CIFAR10 for $\sigma_b= 0.12, 0.25, 0.50$, rate $r=0.01$ and constant, yet increased $C\sigma_b$ training variance, compared to certified accuracies of the constant $\sigma$ method for $\sigma=\sigma_b= 0.12, 0.25, 0.50$ and also $\sigma=C\sigma_b= 0.126, 0.265, 0.53$. Evaluated on a single training.} 
    \label{ablation non-constant sigma effect}
\end{figure}

\begin{figure}[t!]
    \centering
    \begin{minipage}[b]{0.32\linewidth}
        \includegraphics[width=\textwidth]{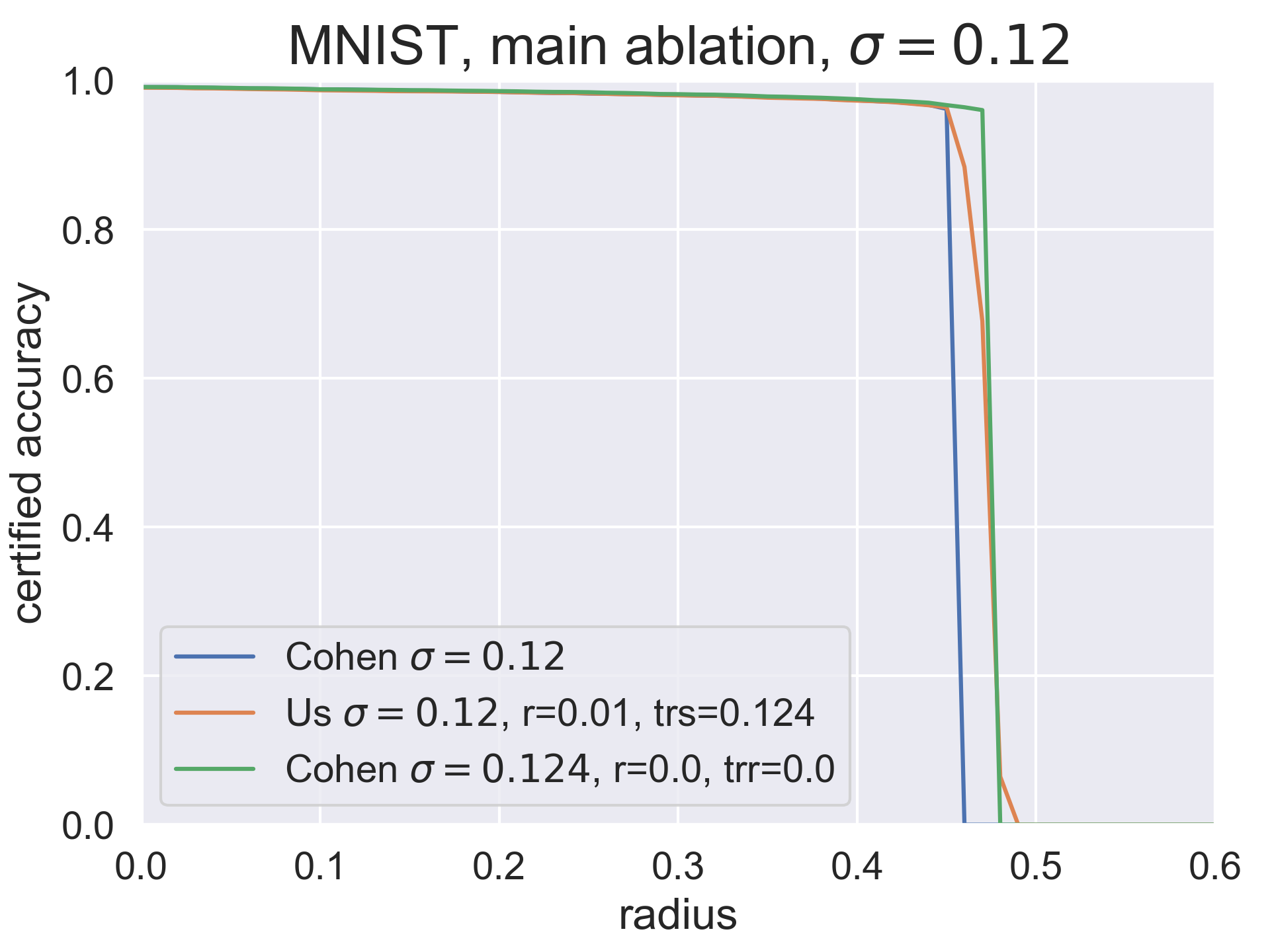}
    \end{minipage}
    \begin{minipage}[b]{0.32\linewidth}
        \includegraphics[width=\textwidth]{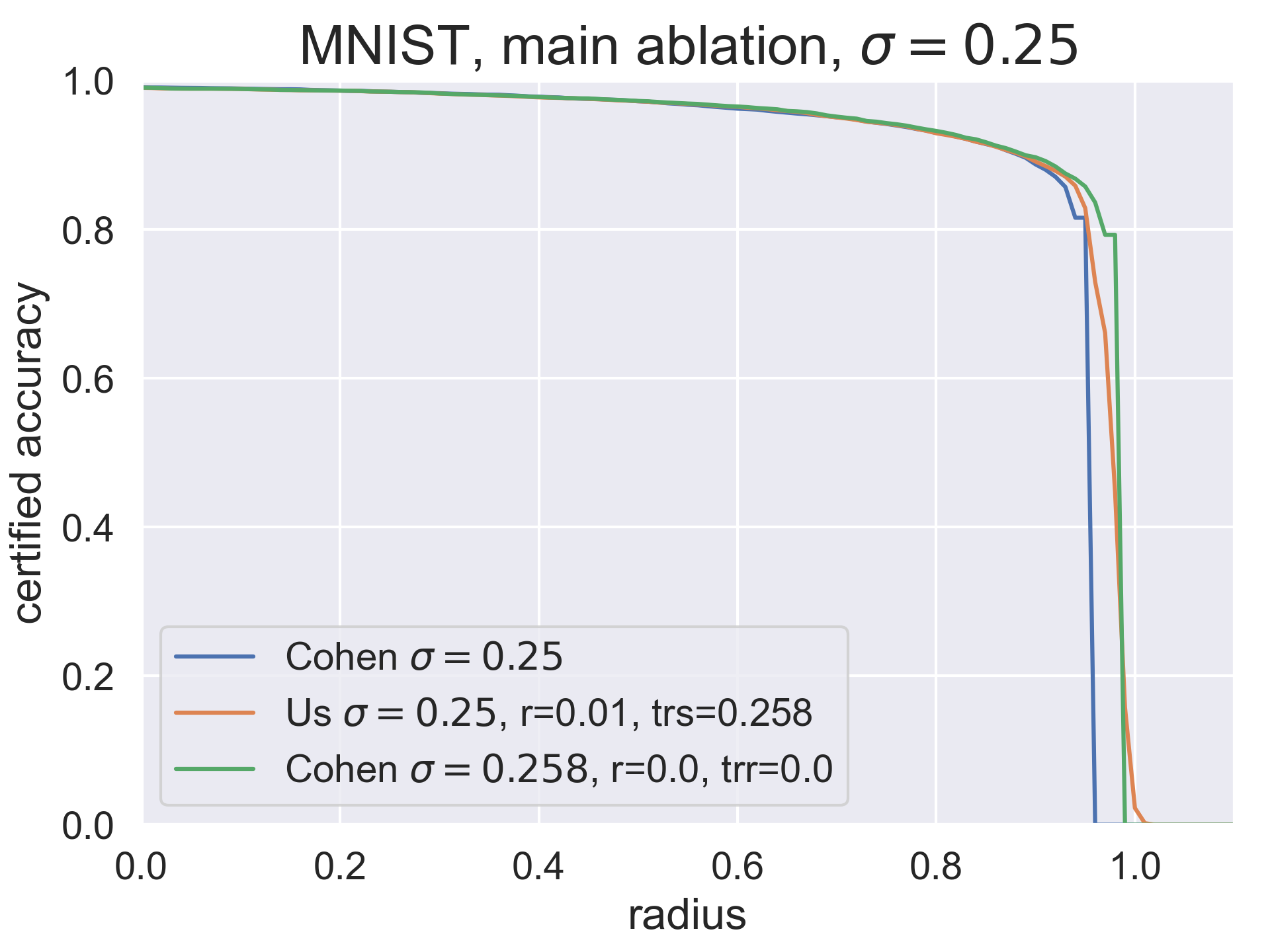}
    \end{minipage}
    \begin{minipage}[b]{0.32\linewidth}
        \includegraphics[width=\textwidth]{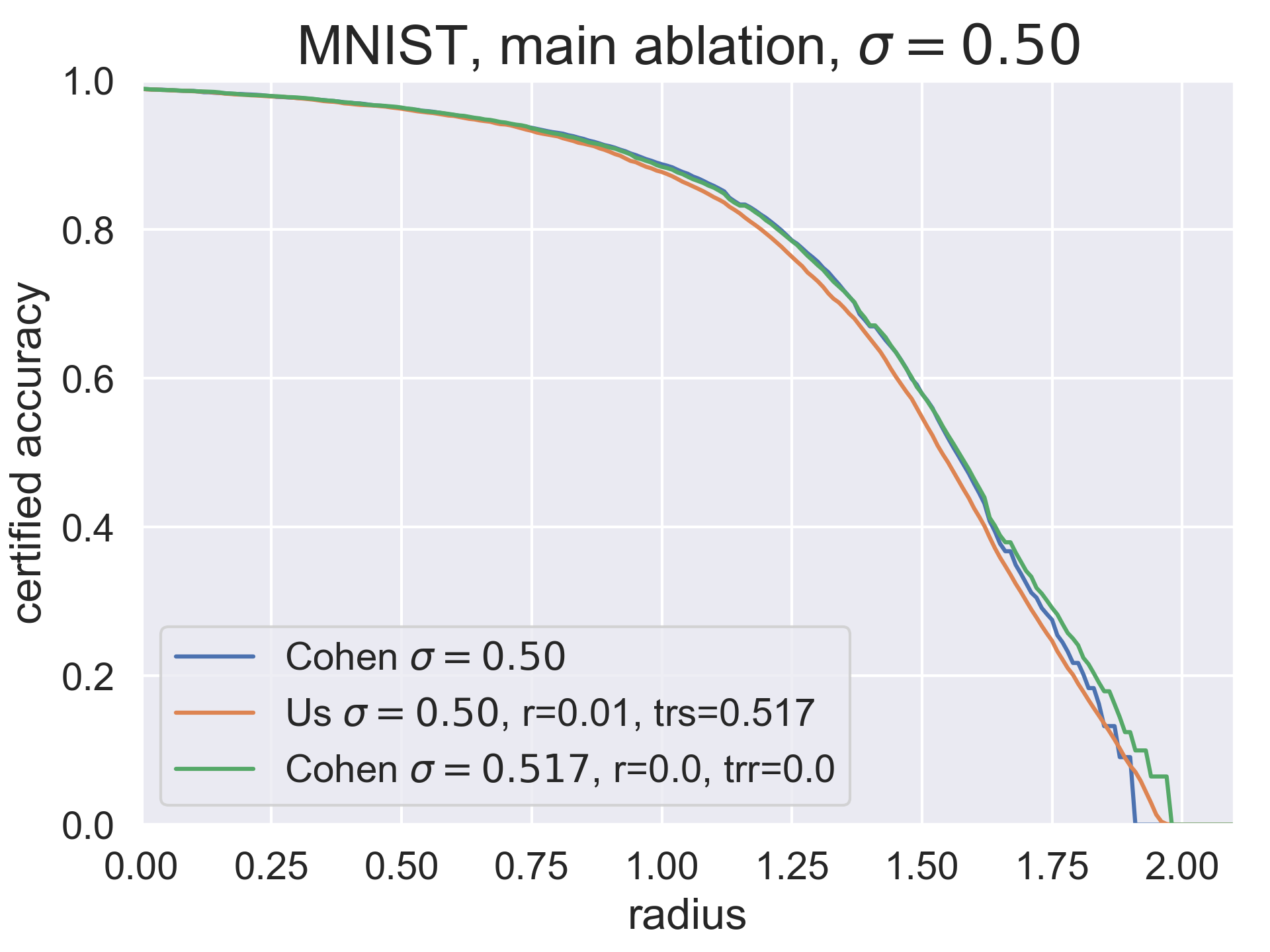}
    \end{minipage}
    \caption{The certified accuracies of our procedure on MNIST for $\sigma_b= 0.12, 0.25, 0.50$, rate $r=0.01$ and constant, yet increased $C\sigma_b$ training variance, compared to certified accuracies of the constant $\sigma$ method for $\sigma=\sigma_b= 0.12, 0.25, 0.50$ and also $\sigma=C\sigma_b= 0.124, 0.258, 0.517$. Evaluated on a single training.} 
    \label{ablation non-constant sigma effect MNIST}
\end{figure}

From the figures, it is obvious, that our method is not able to outperform the constant $\sigma$ method using the same mean $\sigma$ in terms of certified accuracy, not even for our strongest $\sigma_b=0.12$. This fact might not be in general bad news, if we demonstrated, that our method suffers from less pronounced accuracy drop or less pronounced disbalance in class-wise accuracies. To find out, we measure average accuracies of the evaluation strategies from 8 runs for each, as well as average class-wise accuracy standard deviations from 8 runs. The results are provided in Tables~\ref{ablation eval accuracy table} and \ref{ablation eval class-wise accuracy var table} for CIFAR10 and Tables~\ref{ablation eval accuracy table MNIST} and \ref{ablation eval class-wise accuracy var table MNIST} for MNIST.

\begin{table}[t!]
\centering
\begin{tabular}{||c||c|c|c||} 
\hline\hline
 & $\sigma=0.12$ & $\sigma=0.25$ & $\sigma=0.50$ \\ 
\hline\hline
$r=0.01, trs$ increased & 0.852 & 0.780 & 0.673 \\ 
\hline
$r=0.00$ classical & 0.851 & 0.792 & 0.674 \\
\hline
$r=0.00$ increased & 0.853 & 0.780 & 0.673 \\
\hline\hline
\end{tabular}
\vspace{2mm}
\caption{Clean accuracies for both input-dependent and constant $\sigma$ evaluation strategies on CIFAR10.}
\label{ablation eval accuracy table}
\end{table}

\begin{table}[t!]
\centering
\begin{tabular}{||c||c|c|c||} 
\hline\hline
 & $\sigma=0.12$ & $\sigma=0.25$ & $\sigma=0.50$ \\ 
\hline\hline
$r=0.01, trs$ increased & 0.076 & 0.099 & 0.120 \\ 
\hline
$r=0.00$ classical & 0.076 & 0.097 & 0.122 \\
\hline
$r=0.00$ increased & 0.076 & 0.101 & 0.123 \\
\hline\hline
\end{tabular}
\vspace{2mm}
\caption{Class-wise accuracy standard deviations for both input-dependent and constant $\sigma$ evaluation strategies on CIFAR10.}
\label{ablation eval class-wise accuracy var table}
\end{table}

\begin{table}[t!]
\centering
\begin{tabular}{||c||c|c|c||} 
\hline\hline
 & $\sigma=0.12$ & $\sigma=0.25$ & $\sigma=0.50$ \\ 
\hline\hline
$r=0.01, trs$ increased & 0.9913 & 0.9905 & 0.9885 \\ 
\hline
$r=0.00$ classical & 0.9914 & 0.9907 & 0.9886 \\
\hline
$r=0.00$ increased & 0.9914 & 0.9904 & 0.9885 \\
\hline\hline
\end{tabular}
\vspace{2mm}
\caption{Clean accuracies for both input-dependent and constant $\sigma$ evaluation strategies on MNIST.}
\label{ablation eval accuracy table MNIST}
\end{table}

\begin{table}[t!]
\centering
\begin{tabular}{||c||c|c|c||} 
\hline\hline
 & $\sigma=0.12$ & $\sigma=0.25$ & $\sigma=0.50$ \\ 
\hline\hline
$r=0.01, trs$ increased & 0.00757 & 0.00798 & 0.00929 \\ 
\hline
$r=0.00$ classical & 0.00751 & 0.00778 & 0.00934 \\
\hline
$r=0.00$ increased & 0.00750 & 0.00798 & 0.00925 \\
\hline\hline
\end{tabular}
\vspace{2mm}
\caption{Class-wise accuracy standard deviations for both input-dependent and constant $\sigma$ evaluation strategies on MNIST. Printed are multiples of 100 of the real values.}
\label{ablation eval class-wise accuracy var table MNIST}
\end{table}

As for CIFAR10, except for $\sigma=\sigma_b=0.25$, the differences in accuracies between different evaluation strategies are so small, that we cannot consider them to be statistically significant. Even though the difference for $\sigma=\sigma_b=0.25$ is high, it is still not possible to draw some definite conclusions, especially for the difference between the input-dependent $\sigma(x)$ and the increased constant $C\sigma_b$ evaluations. In general, it is not easy to judge, whether our method possesses some advantage (or disadvantage) over the increased $C\sigma_b$ method in terms of clean accuracy. Similar conclusions can be drawn in the context of the shrinking phenomenon. Here, the differences are also very small, but unlike in the comparison with \cite{cohen2019certified} models, where we evaluate our input-dependent $\sigma(x)$ method on classifiers trained with inconsistent data-augmentation variance, here we observe the general trend, that our method is able to outperform the increased constant $C\sigma_b$ evaluation. This is good news and it confirms our suspicion, that the bad results from Subsection~\ref{ssec: cifar10 cohen comparison} could come from the train-test $\sigma$ inconsistency. 

The results on MNIST suggest similar conclusions for the accuracy vs.\ robustness tradeoff. Similarly, the $\sigma=\sigma_b=0.12, 0.50$ are not telling much, and for $\sigma=\sigma_b=0.25$, the differences are still rather small (yet the standard deviation of the results should be $\sim 0.0001$, so it is rather on the edge). The conclusions for the shrinking phenomenon are a bit more pesimistic than in the case of CIFAR10. Here we don't see any improvement over the constant $\sigma$, not even the one with increased $\sigma$ level. 

\subsubsection{Effect of Input-dependent Training} \label{ssec: ablation train}

In this last ablation study, we compare our input-dependent data augmentation for particular $\sigma_b, r$ and particular training rate $trr$ with the constant $C\sigma_b$ data augmentation, where the training rate $trr$ is set to 0. The strategy for choosing the constant $C$ is exactly the same as in the first experiment. Particularly, we evaluate our method with $r=0.01$ and $\sigma_b=0.12, 0.25, 0.50$, trained with the same level of $\sigma_b$ and training rate $trr=0.01$ with the evaluations using $r=0.01$ and $\sigma_b=0.12, 0.25, 0.50$ during test time, while during train time using training rate $trr=0.0$, but using the constant $\sigma= 0.126, 0.263, 0.53$ for CIFAR10 and $\sigma= 0.124, 0.258, 0.517$ for MNIST. This way, we compensate for the ``increased levels of $\sigma(x)$'' with respect to $\sigma_b$. We present our comparisons in the Figure~\ref{ablation input-dependent training effect} for CIFAR10 and \ref{ablation input-dependent training effect MNIST} for MNIST, providing the evaluations with $r=0.01$, $\sigma_b=0.12, 0.25, 0.50$ and the same $\sigma_b$ and $trr=0.0$ during train time as a reference. 

\begin{figure}[t!]
    \centering
    \begin{minipage}[b]{0.32\linewidth}
        \includegraphics[width=\textwidth]{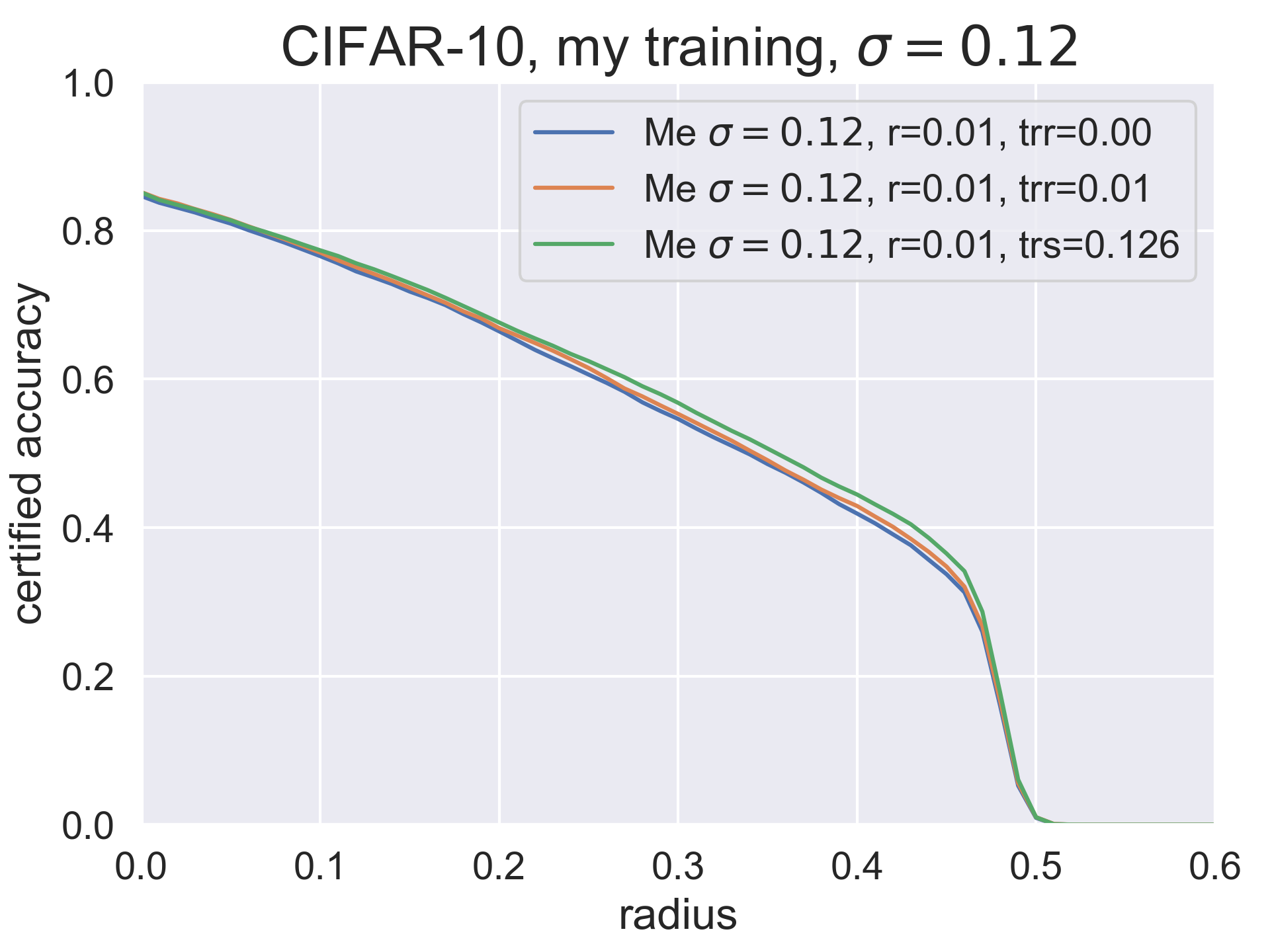}
    \end{minipage}
    \begin{minipage}[b]{0.32\linewidth}
        \includegraphics[width=\textwidth]{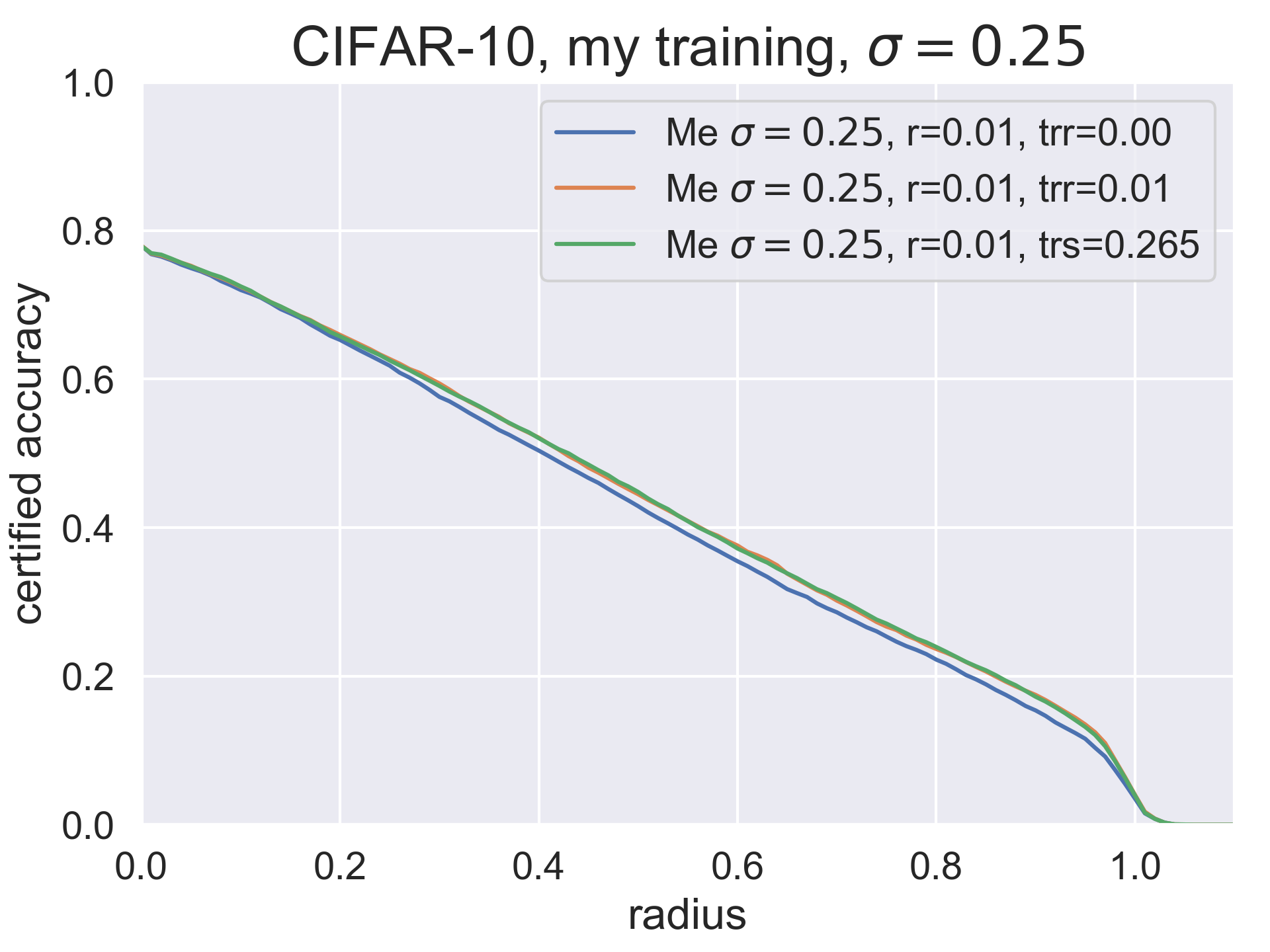}
    \end{minipage}
    \begin{minipage}[b]{0.32\linewidth}
        \includegraphics[width=\textwidth]{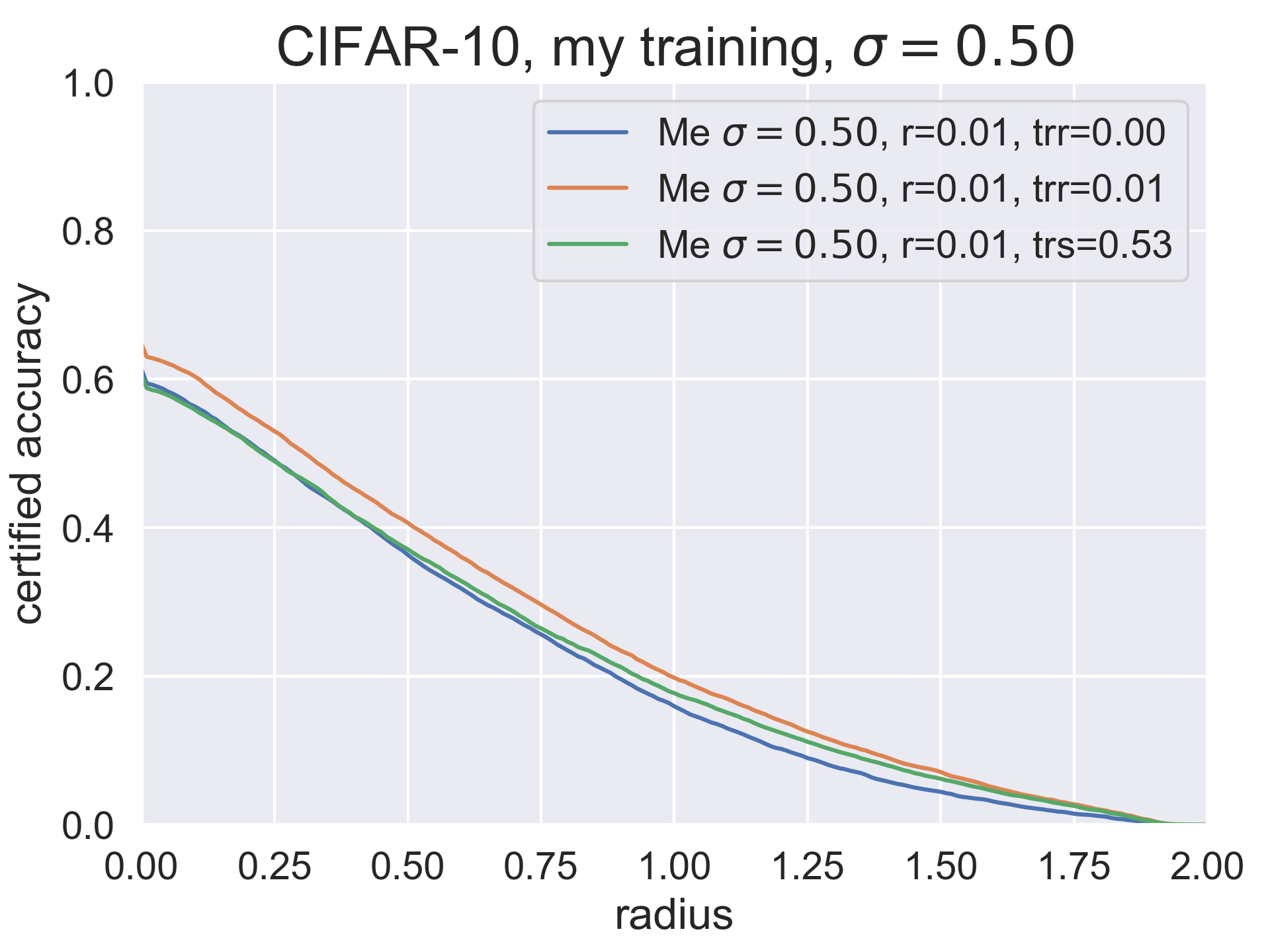}
    \end{minipage}
    \caption{The certified radiuses on CIFAR10 of the non-constant $\sigma(x)$ method with rate $r=0.01$, but different training strategies. Used training strategies are input-dependent training with the same $\sigma(x)$ function and constant-$\sigma$ training with either $\sigma_b$ or $C\sigma_b$ variance level. Evaluations are being done from single run.} 
    \label{ablation input-dependent training effect}
\end{figure}

\begin{figure}[t!]
    \centering
    \begin{minipage}[b]{0.32\linewidth}
        \includegraphics[width=\textwidth]{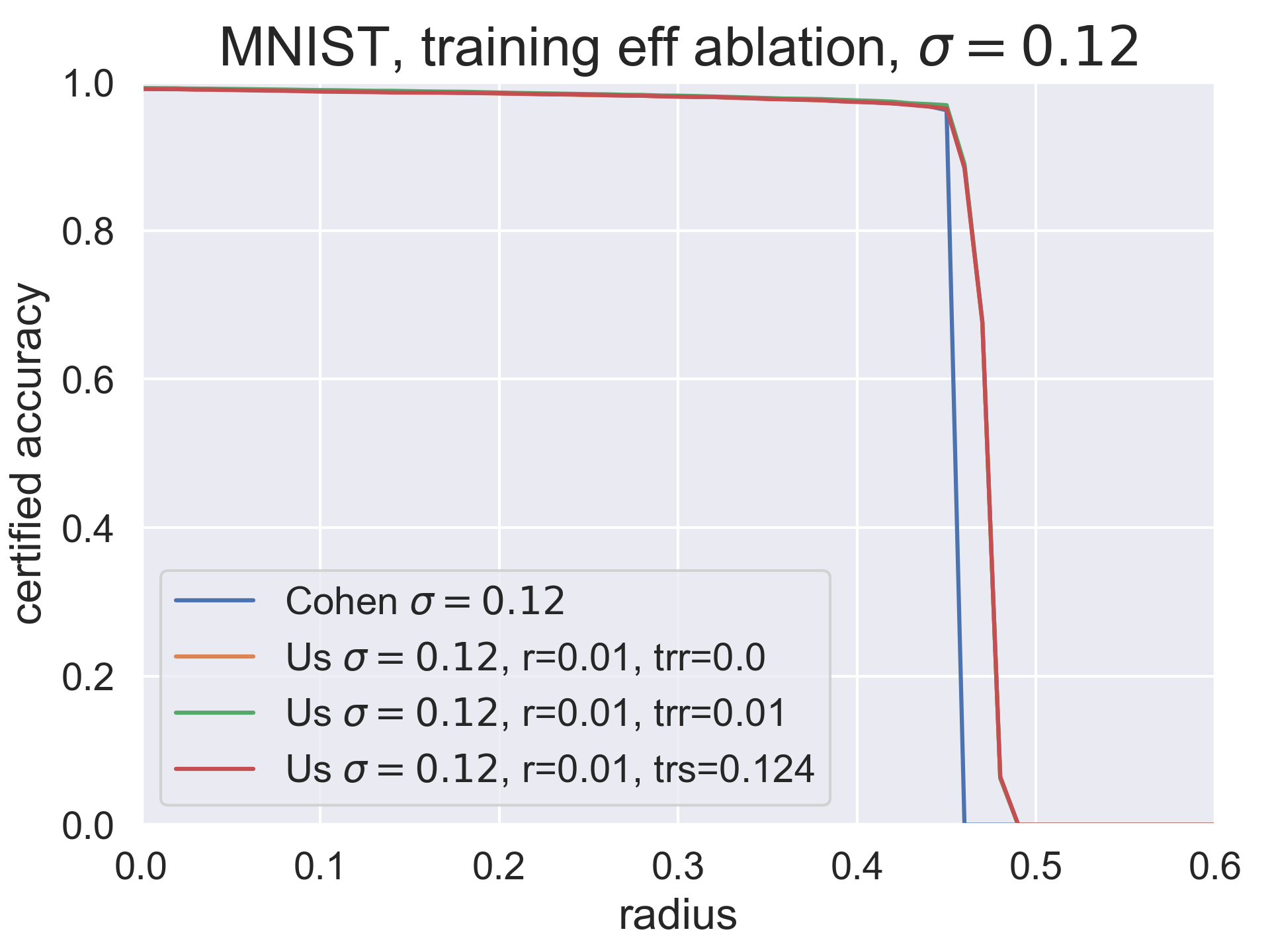}
    \end{minipage}
    \begin{minipage}[b]{0.32\linewidth}
        \includegraphics[width=\textwidth]{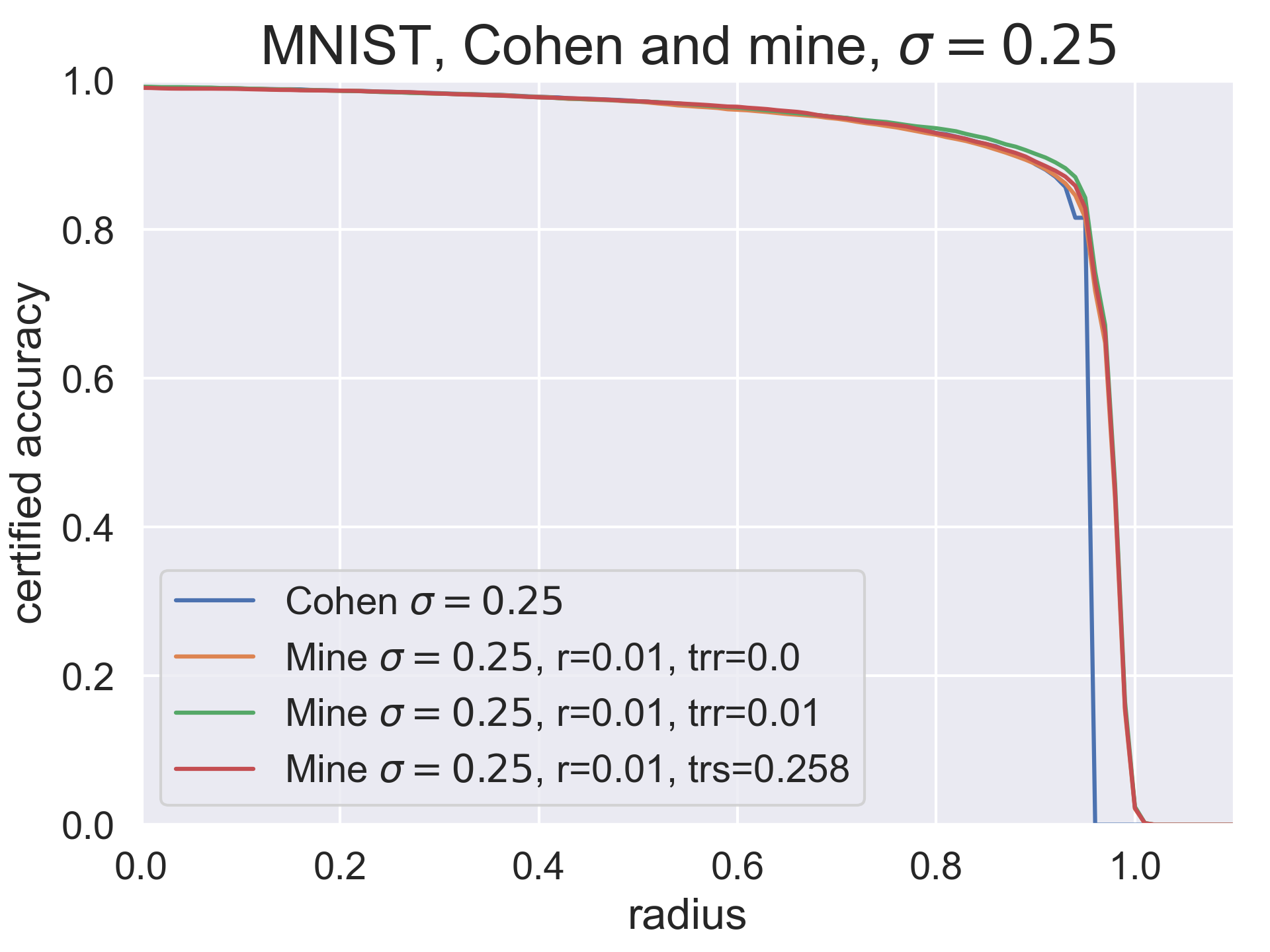}
    \end{minipage}
    \begin{minipage}[b]{0.32\linewidth}
        \includegraphics[width=\textwidth]{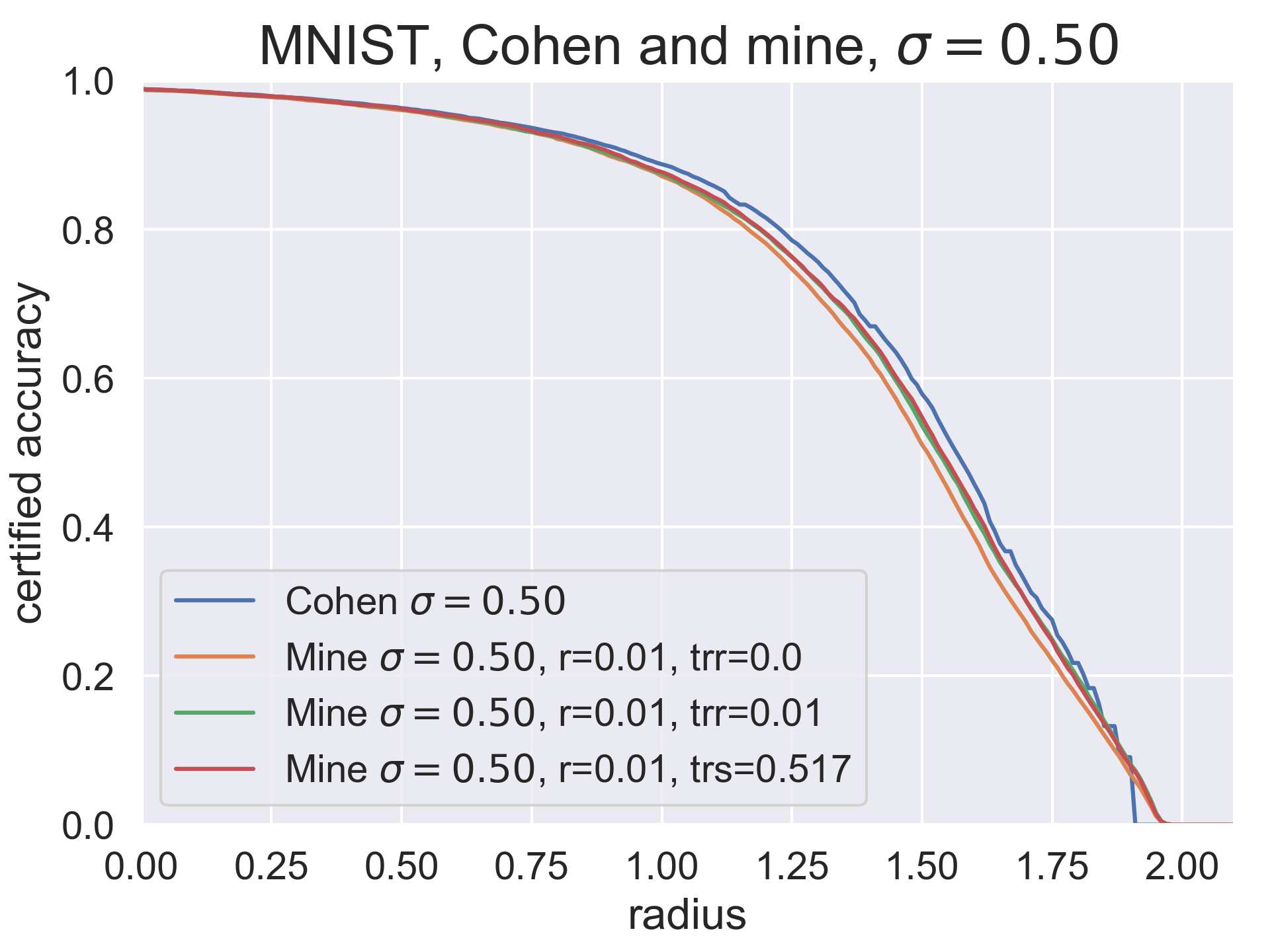}
    \end{minipage}
    \caption{The certified radiuses on MNIST of the non-constant $\sigma(x)$ method with rate $r=0.01$, but different training strategies. Used training strategies are input-dependent training with the same $\sigma(x)$ function and constant-$\sigma$ training with either $\sigma_b$ or $C\sigma_b$ variance level. Evaluations are being done from single run.} 
    \label{ablation input-dependent training effect MNIST}
\end{figure}

The certified accuracy results for the CIFAR10 and the MNIST differ a bit. For CIFAR10 training with rate $r=0.01$ is does not overperform the constant $C\sigma_b$ training. For $\sigma_b=0.12$, the constant $C\sigma_b$ training clearly outperforms the input-dependent training. For $\sigma_b=0.25$, these two training strategies seem to have almost identical performances. For $\sigma_b=0.50$, the input-dependent $\sigma(x)$ strategy outperforms the constant $\sigma$ ones, but now we know, that it is purely due to the variance in the training. On the other hand on MNIST, we either have very similar performance or even slightly outperform the constant $\sigma$ training. 

\begin{table}[t!]
\centering
\begin{tabular}{||c||c|c|c||} 
\hline\hline
 & $\sigma=0.12$ & $\sigma=0.25$ & $\sigma=0.50$ \\ 
\hline\hline
$trr=0.01$ & 0.843 & 0.780 & 0.671 \\ 
\hline
$trr=0.00$ classical & 0.849 & 0.790 & 0.670 \\
\hline
$trr=0.00$ increased & 0.852 & 0.780 & 0.673 \\
\hline\hline
\end{tabular}
\vspace{2mm}
\caption{Clean accuracies for both input-dependent and constant $\sigma$ training strategies on CIFAR10.}
\label{ablation train accuracy table}
\end{table}

\begin{table}[t!]
\centering
\begin{tabular}{||c||c|c|c||} 
\hline\hline
 & $\sigma=0.12$ & $\sigma=0.25$ & $\sigma=0.50$ \\ 
\hline\hline
$trr=0.01$ & 0.080 & 0.101 & 0.121 \\ 
\hline
$trr=0.00$ classical & 0.080 & 0.105 & 0.135 \\
\hline
$trr=0.00$ increased & 0.076 & 0.099 & 0.120 \\
\hline\hline
\end{tabular}
\vspace{2mm}
\caption{Class-wise accuracy standard deviations for both input-dependent and constant $\sigma$ training strategies on CIFAR10.}
\label{ablation train class-wise accuracy var table}
\end{table}

\begin{table}[t!]
\centering
\begin{tabular}{||c||c|c|c||} 
\hline\hline
 & $\sigma=0.12$ & $\sigma=0.25$ & $\sigma=0.50$ \\ 
\hline\hline
$trr=0.01$ & 0.9912 & 0.9910 & 0.9883 \\ 
\hline
$trr=0.00$ classical & 0.9914 & 0.9906 & 0.9884 \\
\hline
$trr=0.00$ increased & 0.9913 & 0.9905 & 0.9885 \\
\hline\hline
\end{tabular}
\vspace{2mm}
\caption{Clean accuracies for both input-dependent and constant $\sigma$ training strategies on MNIST.}
\label{ablation train accuracy table MNIST}
\end{table}

\begin{table}[t!]
\centering
\begin{tabular}{||c||c|c|c||} 
\hline\hline
 & $\sigma=0.12$ & $\sigma=0.25$ & $\sigma=0.50$ \\ 
\hline\hline
$trr=0.01$ & 0.00757 & 0.00800 & 0.00947 \\ 
\hline
$trr=0.00$ classical & 0.00743 & 0.00789 & 0.00929 \\
\hline
$trr=0.00$ increased & 0.00757 & 0.00798 & 0.00929 \\
\hline\hline
\end{tabular}
\vspace{2mm}
\caption{Class-wise accuracy standard deviations for both input-dependent and constant $\sigma$ training strategies on MNIST.}
\label{ablation train class-wise accuracy var table MNIST}
\end{table}

Looking at the accuracy and standard deviation Tables~\ref{ablation train accuracy table}, \ref{ablation train class-wise accuracy var table}, \ref{ablation train accuracy table MNIST} and \ref{ablation train class-wise accuracy var table MNIST}, we can deduce the following. In terms of clean accuracy, the input-dependent training strategy performs worst in most of the cases, even though the differences in performance might not be statistically significant. We see, that we would need far more evaluations to see some clear pattern. However, these results are definitely not good news for the use of input-dependent $\sigma(x)$ during training. 

In terms of the class-wise accuracy standard deviation, we again see countering results for CIFAR10 and MNIST datasets. For CIFAR10 the input-dependent $\sigma(x)$ clearly outperforms the smaller constant $\sigma_b$ training method, particularly for $\sigma_b= 0.50$. However, the constant $C\sigma_b$ method seem to outperform even the input-dependent $\sigma(x)$. For MNIST, the smaller constant $\sigma_b$ outperforms both other methods, while they are rather similar. 

Together with findings from previous sections, these results suggest, that usage of this particular design of input-dependent $\sigma(x)$ might not be worthy until a more precise evaluation is conducted. However, the combination of input-dependent test-time evaluation with constant, yet increased train-time augmentation is possibly the strongest combination that can be achieved using input-dependent sigma at all (especially for CIFAR10).

\section{Proofs} \label{appF: proofs}
\begin{lemma}[Neyman-Pearson] \label{Neyman-Pearson}
Let $X, Y$ be random vectors in $\mathbb{R}^N$ with densities $x, y$. Let $h: \mathbb{R}^N \xrightarrow[]{} \{0, 1\}$ be a random or deterministic function. Then, the following two implications hold: 
\begin{enumerate}
    \item If $S=\left\{z \in \mathbb{R}^N: \frac{y(z)}{x(z)} \le t\right\}$ for some $t>0$ and $\mathbb{P}(h(X)=1)\ge \mathbb{P}(X \in S)$, then $\mathbb{P}(h(Y)=1)\ge \mathbb{P}(Y \in S)$.
    \item If $S=\left\{z \in \mathbb{R}^N: \frac{y(z)}{x(z)} \ge t\right\}$ for some $t>0$ and $\mathbb{P}(h(X)=1)\le \mathbb{P}(X \in S)$, then $\mathbb{P}(h(Y)=1)\le \mathbb{P}(Y \in S)$.
\end{enumerate}
\end{lemma}
\begin{proof}
See \cite{cohen2019certified}.
\end{proof}

\textbf{Lemma~\ref{np lemma}:}
Out of all possible classifiers $f$ such that ${G_f(x_0)}_B \le p_B = 1-p_A$, the one, for which ${G_{f}(x_0+\delta)}_B$ is maximized is the one, which predicts class $B$ in a region determined by the likelihood ratio: \begin{displaymath} B=\left\{x \in \mathbb{R}^N: \frac{q_1(x)}{q_0(x)} \ge \frac{1}{r}\right\}, \end{displaymath} where $r$ is fixed, such that $\mathbb{P}_0(B)=p_B$. Note, that we use $B$ to denote both the class and the region of that class.
\begin{proof}
Let $f$ be arbitrary classifier. To invoke the Neyman-Pearson Lemma~\ref{Neyman-Pearson}, define $h \equiv f$ (with the only difference, that $h$ goes to $\{0, 1\}$ instead of $\{A, B\}$). Moreover, let $S \equiv B$ and $X \sim \mathcal{N}(x, \sigma_0^2 I), Y \sim \mathcal{N}(x+\delta, \sigma_1^2 I)$. Let also $f^*$ classify $S$ as $B$. Then obviously, $\mathbb{P}(X \in S)=\mathbb{P}_0(B)=p_B$. Since $G_f(x)_B \le p_B$, we have $\mathbb{P}(h(X)=1) \le p_B$. Using directly the second part of Neyman-Pearson Lemma~\ref{Neyman-Pearson}, this will yield $\mathbb{P}(Y \in S) \ge \mathbb{P}(h(Y)=1)$. Rewritten in the words of our setup, $G_{f^*}(x+\delta) \ge G_{f}(x+\delta)$. 
\end{proof}

\textbf{Theorem~\ref{lrt set}:}
If $\sigma_0 > \sigma_1$, then $B$ is a $N$-dimensional ball with the center at $S_>$ and radius $R_>$: $$S_>=x_0+\frac{\sigma_0^2}{\sigma_0^2 - \sigma_1^2}\delta, \hspace{1mm} R_>=\sqrt{\frac{\sigma_0^2 \sigma_1^2}{(\sigma_0^2-\sigma_1^2)^2}\norm{\delta}^2+2N\frac{\sigma_0^2 \sigma_1^2}{\sigma_0^2-\sigma_1^2}\log\left(\frac{\sigma_0}{\sigma_1}\right)+\frac{2\sigma_0^2 \sigma_1^2}{\sigma_0^2-\sigma_1^2}\log(r)}.$$
If $\sigma_0 < \sigma_1$, then $B$ is the complement of a $N$-dimensional ball with the center at $S_<$ and radius $R_<$: $$S_<=x_0-\frac{\sigma_0^2}{\sigma_1^2 - \sigma_0^2}\delta, \hspace{1mm} R_<=\sqrt{\frac{\sigma_0^2\sigma_1^2}{(\sigma_0^2 - \sigma_1^2)^2}\norm{\delta}^2+2N\frac{\sigma_0^2 \sigma_1^2}{\sigma_1^2-\sigma_0^2}\log\left(\frac{\sigma_1}{\sigma_0}\right)-\frac{2\sigma_0^2 \sigma_1^2}{\sigma_1^2-\sigma_0^2}\log(r)}.$$
\begin{proof}
From spherical symmetry of isotropic multivariate normal distribution, it follows, that without loss of generality we can take $\delta \equiv (a, 0, \dots, 0)$. With little abuse of notation, let $a$ refer to $(a, 0, \dots, 0)$ as well as $\norm{(a, 0, \dots, 0)}$. With this, the $B$ is a set of all $x$, for which: 
\begin{gather*}
\frac{q_1(x)}{q_0(x)} \ge \frac{1}{r} \iff \\
\frac{1}{(2\pi)^{N/2}\sigma_0^N}\exp\left(-\frac{1}{2\sigma_0^2}\sum\limits_{i=1}^N x_i^2\right) \le \frac{r}{(2\pi)^{N/2}\sigma_1^N}\exp\left(-\frac{1}{2\sigma_1^2}\left[(x_1-a)^2+\sum\limits_{i=2}^N x_i^2 \right]\right) \iff \\
\frac{1}{2\sigma_1^2}\left((x_1-a)^2+\sum\limits_{i=2}^N x_i^2\right)-\frac{1}{2\sigma_0^2}\sum\limits_{i=1}^N x_i^2 \le N\log\left(\frac{\sigma_0}{\sigma_1}\right)+\log(r) \iff \\
(\sigma_0^2-\sigma_1^2)\sum\limits_{i=2}^N x_i^2 + (\sigma_0^2-\sigma_1^2)x_1^2-2\sigma_0^2 x_1 a + \sigma_0^2 a^2 \le 2N\sigma_0^2\sigma_1^2\log\left(\frac{\sigma_0}{\sigma_1}\right)+2\sigma_0^2\sigma_1^2\log(r)
\end{gather*}
Now assume $\sigma_0>\sigma_1$ and continue: 
\begin{gather*}
(\sigma_0^2-\sigma_1^2)\sum\limits_{i=2}^N x_i^2 + (\sigma_0^2-\sigma_1^2)x_1^2-2\sigma_0^2 x_1 a + \sigma_0^2 a^2 \le 2N\sigma_0^2\sigma_1^2\log\left(\frac{\sigma_0}{\sigma_1}\right)+2\sigma_0^2\sigma_1^2\log(r) \iff \\
\sum\limits_{i=2}^N x_i^2 + x_1^2 - \frac{2\sigma_0^2}{\sigma_0^2 - \sigma_1^2}ax_1+\frac{a^2\sigma_0^2}{\sigma_0^2 - \sigma_1^2} \le 2N\frac{\sigma_0^2\sigma_1^2}{\sigma_0^2-\sigma_1^2}\log\left(\frac{\sigma_0}{\sigma_1}\right)+\frac{2\sigma_0^2\sigma_1^2}{\sigma_0^2-\sigma_1^2}\log(r) \iff \\
\left(x_1-\frac{\sigma_0^2}{\sigma_0^2-\sigma_1^2}a\right)^2+\sum\limits_{i=2}^N x_i^2 \le \frac{\sigma_0^2\sigma_1^2}{(\sigma_0^2 - \sigma_1^2)^2}a^2+2N\frac{\sigma_0^2\sigma_1^2}{\sigma_0^2-\sigma_1^2}\log\left(\frac{\sigma_0}{\sigma_1}\right)+\frac{2\sigma_0^2\sigma_1^2}{\sigma_0^2-\sigma_1^2}\log(r)
\end{gather*}

Such inequality defines exactly the ball from the statement of the theorem. On the other hand, if $\sigma_0<\sigma_1$:

\begin{gather*}
(\sigma_1^2-\sigma_0^2)\sum\limits_{i=2}^N x_i^2 + (\sigma_1^2-\sigma_0^2)x_1^2+2\sigma_0^2 x_1 a - \sigma_0^2 a^2 \ge 2N\sigma_0^2\sigma_1^2\log\left(\frac{\sigma_1}{\sigma_0}\right)-2\sigma_0^2\sigma_1^2\log(r) \iff \\
\sum\limits_{i=2}^N x_i^2 + x_1^2 + \frac{2\sigma_0^2}{\sigma_1^2 - \sigma_0^2}ax_1-\frac{a^2\sigma_0^2}{\sigma_1^2 - \sigma_0^2} \ge 2N\frac{\sigma_0^2\sigma_1^2}{\sigma_1^2-\sigma_0^2}\log\left(\frac{\sigma_1}{\sigma_0}\right)-\frac{2\sigma_0^2\sigma_1^2}{\sigma_1^2-\sigma_0^2}\log(r) \iff \\
\left(x_1+\frac{\sigma_0^2}{\sigma_1^2-\sigma_0^2}a\right)^2+\sum\limits_{i=2}^N x_i^2 \ge \frac{\sigma_0^2\sigma_1^2}{(\sigma_0^2 - \sigma_1^2)^2}a^2+2N\frac{\sigma_0^2\sigma_1^2}{\sigma_1^2-\sigma_0^2}\log\left(\frac{\sigma_1}{\sigma_0}\right)-\frac{2\sigma_0^2\sigma_1^2}{\sigma_1^2-\sigma_0^2}\log(r)
\end{gather*}

This is exactly the complement of a ball from the second part of the statement of the theorem. 
\end{proof}

\textbf{Theorem~\ref{thm ncchsq}:}
$$\mathbb{P}_0(B)=\chi^2_N\left(\frac{\sigma_0^2}{(\sigma_0^2 - \sigma_1^2)^2}\norm{\delta}^2, \frac{R_{<,>}^2}{\sigma_0^2}\right), \mathbb{P}_1(B)=\chi^2_N\left(\frac{\sigma_1^2}{(\sigma_0^2 - \sigma_1^2)^2}\norm{\delta}^2, \frac{R_{<,>}^2}{\sigma_1^2}\right),$$
where the sign $<$ or $>$ is choosed according to the inequality between $\sigma_0$ and $\sigma_1$. 
\begin{proof}
Assume first $\sigma_0>\sigma_1$. Let us shift the coordinates, such that $x+\frac{\sigma_0^2}{\sigma_0^2 - \sigma_1^2}\delta \xleftarrow[]{} 0$. Now, the $x$ will have coordinates $-\frac{\sigma_0^2}{\sigma_0^2 - \sigma_1^2}\delta$. Assume $X \sim \mathcal{N}(-\frac{\sigma_0^2}{\sigma_0^2 - \sigma_1^2}\delta, \sigma_0^2 I)$. To obtain $$\mathbb{P}_0(B)=\mathbb{P}(X \in B)=\mathbb{P}\left(\norm{X}^2 < \frac{\sigma_0^2 \sigma_1^2}{(\sigma_0^2-\sigma_1^2)^2}\norm{\delta}^2+2N\frac{\sigma_0^2 \sigma_1^2}{\sigma_0^2-\sigma_1^2}\log\left(\frac{\sigma_0}{\sigma_1}\right)+\frac{2\sigma_0^2 \sigma_1^2}{\sigma_0^2-\sigma_1^2}\log(r)\right),$$ we could almost use NCCHSQ, but we don't have correct scaling of variance of $X$. However, for any regular square matrix $Q$, it follows that $\mathbb{P}(X \in B)=\mathbb{P}(QX \in QB)$, where $QB$ is interpreted as set projection. Therefore, if we choose $Q \equiv \frac{1}{\sigma_0}I$, we will get $$\mathbb{P}_0(B)=\mathbb{P}\left(\norm{X/\sigma_0}^2 < \frac{ \sigma_1^2}{(\sigma_0^2-\sigma_1^2)^2}\norm{\delta}^2+2N\frac{ \sigma_1^2}{\sigma_0^2-\sigma_1^2}\log\left(\frac{\sigma_0}{\sigma_1}\right)+\frac{2\sigma_1^2}{\sigma_0^2-\sigma_1^2}\log(r)\right).$$ Now, since $X/\sigma_0 \sim \mathcal{N}(-\frac{\sigma_0}{\sigma_0^2 - \sigma_1^2}\delta,I),$ we can use the definition of NCCHSQ to obtain the final: $$\mathbb{P}_0(B)=\chi^2_N\left(\frac{\sigma_0^2}{(\sigma_0^2 - \sigma_1^2)^2}\norm{\delta}^2, \frac{\sigma_1^2}{(\sigma_0^2-\sigma_1^2)^2}\norm{\delta}^2+2N\frac{\sigma_1^2}{\sigma_0^2-\sigma_1^2}\log\left(\frac{\sigma_0}{\sigma_1}\right)+\frac{2\sigma_1^2}{\sigma_0^2-\sigma_1^2}\log(r)\right).$$
To obtain $\mathbb{P}_1(B)$, we will do similar calculation, yet we need to compute the offset: 
$$x+\frac{\sigma_0^2}{\sigma_0^2 - \sigma_1^2}\delta-x-\delta=\frac{\sigma_1^2}{\sigma_0^2-\sigma_2^2}a.$$
Thus, after shifting coordinates in the same way, our alternative $X$ will be distributed like \\ $X \sim \mathcal{N}(-\frac{\sigma_1^2}{\sigma_0^2-\sigma_2^2}a, \sigma_1 I)$. Now, the same idea as before will yield the required formula. 

In the case of $\sigma_0<\sigma_1$, we do practically the same thing, yet now, we have to keep in mind, that $B$ will not be a ball, but its complement, therefore we will obtain ``$1-$'' in the formulas. 
\end{proof}

\begin{lemma} \label{continuity of xi}
Functions $\xi_>(a), \xi_<(a)$ are continuous on the whole $\mathbb{R}_+$. Particularly, they are continuous at 0. 
\end{lemma}

\begin{proof}
Assume for simplicity $\sigma_0>\sigma_1$ and fix $x_0$, whose position is irrelevant and fix $x_1$ such that $\norm{x_0-x_1}=a_m$, where $a_m$ is the point, where we prove the continuity. Note, that $\chi^2_N(\lambda, x)$ can be interpreted (as we have seen) as a probability of an offset ball with radius $\sqrt{x}$ and offset $\sqrt{\lambda}$. Assume we have a sequence $\left\{a_i\right\}_{i=1}^\infty, a_i \xrightarrow[]{} a_m$. Define $x_i$ to be a point lying on the line defined by $x_0, x_1$ s.t. $\norm{x_0-x_i}=a_i$. define $B_{a_i}$ to be the worst-case ball corresponding to $a_i, x_i$. Now, without loss of generality, we can assume, that all $B$'s are open. We have already seen from Lemma~\ref{lrt set}, that centers of $B_i$ converge to the center of $B_m$. Define $X_i=\mathbf{1}(B_i), X_m=\mathbf{1}(B_m)$.

First we need to prove, that $r_i$, radiuses of the balls converge. Assume for contradiction, that $r_i$ do not converge. Without loss of gererality, let $r_s=\underset{i \xrightarrow[]{} \infty}{\limsup} \hspace{1mm} r_i > r_m$. If $r_s=\infty$, it is trivial to see, that $P_0(B_i) \neq p_B$ for some $i$ for which $r_i$ is too big. If $r_s<\infty$, consider $\{a_{i_k}\}_{k=0}^\infty$ to be the subsequence for which $r_s$ is monotonically attained. Define $B_s$ the ball with center $x_1$ and radius $r_s$ and $X_s=\mathbf{1}(B_s)$. Then, $X_{i_k} \overset{a.s.}{\xrightarrow[]{}} X_s$ for $k \xrightarrow[]{} \infty$ and from dominated convergence theorem, $\mathbb{P}_0(B_{i_k}) \xrightarrow[]{} \mathbb{P}_0(B_s)$. However, $\mathbb{P}_0(B_s) > \mathbb{P}_0(B_m) = p_B$, what is contradiction, since obviously $\mathbb{P}_0(B_{i_k}) \neq p_B$ for some $k$. 

Since $\sigma_1$ is fixed, the $\mathbb{P}_{1_i}$, probability measures corresponding to $\mathcal{N}(x_i, \sigma_1 I)$ are actually the same probability measure up to a shift. Therefore, $\mathbb{P}_{1_i}(B_i)$ can be treated as $\mathbb{P}_2(\bar B_i)$, where $\mathbb{P}_2$ is simply measure corresponding to $\mathcal{N}(0, \sigma_1 I)$ and $\bar B_i$ is simply $B_i$ shifted accordingly s.t. $\mathbb{P}_{1_i}(B_i)=\mathbb{P}_2(\bar B_i)$ (and assume, without loss of generality, that for each $i$, they are shifted such that their centers lie on a fixed line). Now, since we know, that ``position'' of both the centers of the balls and the $x_1$ is continuous w.r.t $a$, as can be seen from Lemma~\ref{lrt set} (and the radiuses are still $r_i$ and converge), we see, that even $\mathbf{1}(\bar B_i) \xrightarrow[]{} \mathbf{1}(\bar B_m)$ almost surely. Now, we can simply use dominated convergence theorem using $\mathbb{P}_2$ to obtain $\mathbb{P}_2(\bar B_i) \xrightarrow[]{} \mathbb{P}_2(\bar B_m)$ and thus $\mathbb{P}_{1_i}(B_i) \xrightarrow[]{} \mathbb{P}_{1_m}(B_m)$, what we wanted to prove.

Note that the proof for the case $\sigma_0<\sigma_1$ is fully analogous, yet instead of class $B$ probabilities, we work with class $A$ probabilities to still work with balls and not with less convenient complements.
\end{proof}

\begin{lemma} \label{monotonicity of chi wrt noncentrality}
If $\lambda_1>\lambda_2$, then $\chi^2_N(\lambda_1^2, x^2) \le \chi^2_N(\lambda_2^2, x^2)$.
\end{lemma}

\begin{proof}
Let us fix $\mathcal{N}(0, I)$ and respective measure $\mathbb{P}$ and respective density $f$. From symmetry, the NCCHCSQ defined as distribution of $\norm{X}^2$ for an offset normal distribution can be as well defined as $\norm{X-s}^2$ under centralized normal distribution. Define $B_1$ a ball with center at $(\lambda_1, 0, \dots, 0)$ and radius $x$ and $B_2$ a ball with center at $(\lambda_2, 0, \dots, 0)$ and radius $x$. Denote $C(B)$ as the center of a ball $B$. From definition of NCCHSQ it now follows, that $\mathbb{P}(B_i)=\chi^2_N(\lambda_i^2, x^2), i \in \{1, 2\}.$ Therefore, it suffices to show $\mathbb{P}(B_1) \le \mathbb{P}(B_2)$. 

Define $D_1 = B_1 \backslash B_2$ and $D_2 = B_2 \backslash B_1$. Then we know: 
\begin{gather*}
\mathbb{P}(B_1)=\int\limits_{B_1} f(z)dz = \int\limits_{B_1 \cap B_2} f(z)dz + \int\limits_{B_1 \backslash B_2} f(z)dz = \int\limits_{B_1 \cap B_2} f(z)dz + \int\limits_{D_1} f(z)dz \\ 
\mathbb{P}(B_2)=\int\limits_{B_2} f(z)dz = \int\limits_{B_2 \cap B_1} f(z)dz + \int\limits_{B_2 \backslash B_1} f(z)dz = \int\limits_{B_2 \cap B_1} f(z)dz + \int\limits_{D_2} f(z)dz.
\end{gather*}
Thus, $$\mathbb{P}(B_1) \le \mathbb{P}(B_2) \iff \int\limits_{D_1} f(z)dz \le \int\limits_{D_2} f(z)dz.$$
Let $S=\frac{C(B_1)+C(B_2)}{2}$. Define a central symmetry $M$ with center $S$. Let $z_1 \in D_1$. Then $z_1$ can be decomposed as $z_1=C(B_1)+d, \norm{d}\le x$. Then, $z_2 := M(z_1) = C(B_2)-d$ from symmetry. This way, we see, that $D_1 = M(D_2)$ and $D_2 = M(D_1)$ under a bijection $M$ which does not distort the geometry and distances of the euclidean space. Therefore, it suffices to show: 
$$\forall \hspace{1mm} z \in D_2: f(z)\ge f(M(z)), M(z) \in D_1.$$

From the monotonicity of $f(y)$ w.r.t $\norm{y}$ it actually suffices to show $\norm{z}\le \norm{M(z)} \hspace{1mm} \forall \hspace{1mm} z \in D_2$. Fix some $z \in D_2$. By the fact that $M$ is central symmetry and $z \xrightarrow[]{} M(z)$, it is obvious, that $z=S+p$, $M(z)=S-p$, where $p$ is some vector. Now, using law of cosine, we can write: $$\norm{z}^2=\norm{S}^2+\norm{p}^2-2\norm{S}\norm{p}\cos(\alpha),$$ where $\alpha$ is angle between $S$ and $-p$. On the other hand:
$$\norm{M(z)}^2=\norm{S}^2+\norm{-p}^2-2\norm{S}\norm{-p}\cos(\pi-\alpha).$$
It is obvious from these equations, that $$\norm{z} \le \norm{M(z)} \iff \alpha \le \pi/2 \iff p^TS \le 0.$$
Here, the crucial observation is, that $D_1$ and $D_2$ are separated by a hyperplane perpendicular to $S$ (vector), such that $S$ (point) is in this hyperplane. From this it follows: $$y \in D_2 \implies y^TS\le\norm{S}^2, \hspace{2mm} y \in D_1 \implies y^TS\ge\norm{S}^2.$$ 

Now, since $z=S+p$ and $z \in D_2$, this implies $\norm{S}^2 \ge z^TS = \norm{S}^2+p^TS$ and thus $p^TS \le 0$. 
\end{proof}

\begin{lemma} \label{monotonicity of xi}
Functions $\xi_>(a), \xi_<(a)$ are non-decreasing in $a$. 
\end{lemma}
\begin{proof}
First assume $\sigma_0>\sigma_1$ and analyse $\xi_>(a)$. From Lemma~\ref{continuity of xi}, we can without loss of generality assume, that $a>0$, since $\xi_>(0)$ is simply the limit for $a \xrightarrow[]{} 0$ and cannot change the monotonicity status. 

Now, fix $a>0$ and define $x_0, x_1$ s.t. $\norm{x_0-x_1}=a$. Denote $B_a$ to be the worst-case ball corresponding to $x_0, x_1$ and $\mathbb{P}_1$ as usual. Choose $\epsilon<\frac{\sigma_1^2}{\sigma_0^2-\sigma_1^2}a$, which is the distance between $x_1$ and $C(B_a)$, as can be seen from Lemma~\ref{lrt set}.

Now, assume $x_2$ lies on line defined by $x_0, x_1$ with $\norm{x_0-x_2}=a+\epsilon$. Let $B_{a+\epsilon}$ be the corresponding worst-case ball and $\mathbb{P}_2$ as usual. First observe, that $\mathbb{P}_2(B_a) \ge \mathbb{P}_1(B_a)$, since $\norm{x_1-C(B_a)}>\norm{x_2-C(B_a)}$. Here we use $\chi^2_N(\lambda_1, x) \le \chi^2_N(\lambda_2, x)$ if $\lambda_1 > \lambda_2$, what is proved in Lemma~\ref{monotonicity of chi wrt noncentrality}. Second, note, that since $B_{a+\epsilon}$ is the worst-case ball for $x_2$, it follows $\mathbb{P}_2(B_{a+\epsilon}) \ge \mathbb{P}_2(B_a)$. Thus, $\mathbb{P}_1(B_a) \le \mathbb{P}_2(B_{a+\epsilon})$, but that is exactly $\xi_>(a) \le \xi_>(a+\epsilon)$. 

To prove $\xi_>(a) \le \xi_>(a+\epsilon)$ also for $\epsilon>\frac{\sigma_1^2}{\sigma_0^2-\sigma_1^2}a$, it suffices to consider finite sequence of points $a_i$ starting at $a$ and ending at $a+\epsilon$ that are ``close enough to each other'' such that the respective $\epsilon_i$ that codes the shift $a_i \xrightarrow[]{} a_{i+1}$ satisfy $\epsilon_i<\frac{\sigma_1^2}{\sigma_0^2-\sigma_1^2}a_i$. 

Now, assume $\sigma_0<\sigma_1$ and analysie $\xi_<(a)$. The proof is similar, but we have to be a bit careful about some details. Again, fix $a>0$, define $x_0, x_1$ accordingly and all other objects as before, except now, let us denote $A_a=B_a^C$ and $A_{a+\epsilon}=B_a^C$ to be the class $A$ balls which are complements to the anti-balls $B$. Again, $\mathbb{P}_2(A_a) \le \mathbb{P}_1(A_a)$, since $\norm{x_1-C(A_a)}<\norm{x_2-C(A_a)}$. We again used the monotonicity from Lemma~\ref{monotonicity of chi wrt noncentrality}. Moreover, $\mathbb{P}_2(A_{a+\epsilon}) \le \mathbb{P}_2(A_a)$, since $B_{a+\epsilon}$ is the worst-case set for $x_2$. Therefore, $\mathbb{P}_1(A_a) \ge \mathbb{P}_2(A_{a+\epsilon})$, but after reverting to $B$'s, it follows $\xi_>(a) \le \xi_>(a+\epsilon)$. 

Here, we don't even need to care about $\norm{\epsilon}$, since the centers of $A$'s are on the opposite half-lines from $x_0$ than $x_1$ and $x_2$. 
\end{proof}

To prove the main theorem, we need a simple bound on a median of central chi-squared distribution, shown in \cite{robert1990some} in a more general way. 

\begin{lemma} \label{bounds on quantile of ncchsq}
For all $c \ge 0$, $$N-1+c \le \chi^2_{N, qf}(c, 0.5) \le \chi^2_{N, qf}(0.5)+c.$$
\end{lemma}
\begin{proof}
See \cite{robert1990some}.
\end{proof}

\textbf{Theorem~\ref{main theorem} (the curse of dimensionality):}
Let $x_0, x_1, p_A, \sigma_0, \sigma_1, N$ be as usual. Then, the following two implications hold: 
\begin{enumerate}
    \item If $\sigma_0>\sigma_1$ and $$\log\left(\frac{\sigma_1^2}{\sigma_0^2}\right)+1-\frac{\sigma_1^2}{\sigma_0^2} < \frac{2\log(1-p_A)}{N},$$
    then $x_1$ is not certified w.r.t. $x_0$. 
    \item If $\sigma_0<\sigma_1$ and 
    $$\log\left(\frac{\sigma_1^2}{\sigma_0^2}\frac{N-1}{N}\right)+1-\frac{\sigma_1^2}{\sigma_0^2}\frac{N-1}{N} < \frac{2\log(1-p_A)}{N},$$
    then $x_1$ is not certified w.r.t. $x_0$. 
\end{enumerate}
\begin{proof}
We will first prove first statement, thus let us assume $\sigma_0>\sigma_1$. Then $\mathbb{P}_1(B)=\xi_>(\norm{x_0-x_1})$. From monotonicity of $\xi$ showed in Lemma~\ref{monotonicity of xi}, we know $\xi_>(\norm{x_0-x_1})\ge \xi_>(0)$. We will show $\xi_>(0)>0.5$. We have, using definition of $\xi_>$ plugging in $a=0$: 
$$\xi_>(0)=\chi^2_N\left(\frac{\sigma_0^2}{\sigma_1^2}\chi^2_{N, qf}(1-p_A)\right).$$
Note, that here, we work with central chi-square cdf and quantile function. In order to show $\xi_>(0)>0.5$, it suffices to show $$\frac{\sigma_0^2}{\sigma_1^2}\chi^2_{N, qf}(1-p_A) \ge N,$$
because it is well-known, that median of central chi-square distribution is smaller than mean, which is $N$, i.e. from strict monotonicity of cdf, we will get $\chi^2_N(N)>0.5$. To show the above inequality, we will use Chernoff bound on chi-squared, which states the following: If $0<z<1$, then $\chi^2_N(zN)\le (z\exp(1-z))^{N/2}$. Putting $z \equiv \frac{\sigma_1^2}{\sigma_0^2}$, using chernoff bound we get: 
$$\chi^2_N\left(\frac{\sigma_1^2}{\sigma_0^2}N\right) \le \left[\frac{\sigma_1^2}{\sigma_0^2}\exp\left(1-\frac{\sigma_1^2}{\sigma_0^2}\right)\right]^{\frac{N}{2}} \overset{!}{<} 1-p_A.$$
The last inequality is required to hold. If it holds, then necessarily $\chi^2_{N, qf}(1-p_A)>\frac{\sigma_1^2}{\sigma_0^2}N$ and thus $\frac{\sigma_0^2}{\sigma_1^2}\chi^2_{N, qf}(1-p_A) > N.$
Manipulating the required inequality, we will get exactly $$\log\left(\frac{\sigma_1^2}{\sigma_0^2}\right)+1-\frac{\sigma_1^2}{\sigma_0^2} < \frac{2\log(1-p_A)}{N},$$ what is the assumption of 1. 

Similarly we will also prove the statement 2. Assume $\sigma_0>\sigma_1$. Like in part 1, we will just prove $\xi_<(0)>0.5$ by using chernoff bound. This time, however, we have: $$\xi_<(0)=1-\chi^2_N\left(\frac{\sigma_0^2}{\sigma_1^2}\chi^2_{N, qf}(p_A)\right),$$ i.e. we need to prove $$\chi^2_N\left(\frac{\sigma_0^2}{\sigma_1^2}\chi^2_{N, qf}(p_A)\right) < \frac{1}{2}.$$
The second part of Chernoff bound states: If $1<z$, then $\chi^2_N(zN)\ge 1-(z\exp(1-z))^{N/2}$.
Let us choose $z \equiv \frac{\sigma_1^2}{\sigma_0^2}\frac{N-1}{N}$. Then Chernoff bound yields: 
$$\chi^2_N\left(\frac{\sigma_1^2}{\sigma_0^2}\frac{N-1}{N} N\right)\ge 1-\left[\frac{\sigma_1^2}{\sigma_0^2}\frac{N-1}{N}\exp\left(1-\frac{\sigma_1^2}{\sigma_0^2}\frac{N-1}{N}\right)\right] \overset{!}{>} p_A.$$
If this holds, then $$\frac{\sigma_1^2}{\sigma_0^2}\frac{N-1}{N} N > \chi^2_{N, qf}(p_A) \iff \frac{\sigma_0^2}{\sigma_1^2}\chi^2_{N, qf}(p_A)<N-1.$$
Now, using Lemma~\ref{bounds on quantile of ncchsq} for the easy case of central chi-squared, we see: $\chi^2_N(N-1)<0.5$ and thus $$\chi^2_N\left(\frac{\sigma_0^2}{\sigma_1^2}\chi^2_{N, qf}(p_A)\right) < \frac{1}{2},$$
what we wanted to prove. 
\end{proof}

\textbf{Corollary~\ref{main thm corollary} (one-sided simpler bound):}
Let $x_0, x_1, p_A, \sigma_0, \sigma_1, N$ be as usual and assume now $\sigma_0>\sigma_1$. Then, if
$$\frac{\sigma_1}{\sigma_0}<\sqrt{1-2\sqrt{\frac{-\log(1-p_A)}{N}}},$$ then $x_1$ is not certified w.r.t $x_0$. 
\begin{proof}
We will simply prove $$\frac{\sigma_1}{\sigma_0}<\sqrt{1-2\sqrt{\frac{-\log(1-p_A)}{N}}} \implies \log\left(\frac{\sigma_1^2}{\sigma_0^2}\right)+1-\frac{\sigma_1^2}{\sigma_0^2} < \frac{2\log(1-p_A)}{N}.$$

Assume expression $\log(1-y)+y; 1>y>0$. From Taylor series, it is apparent, that $\log(1-y)+y < -\frac{y^2}{2}$. Therefore, if $-\frac{y^2}{2} < \frac{2\log(1-p_A)}{N}$, then also $\log(1-y)+y < \frac{2\log(1-p_A)}{N}$. Solving for $y$ in the first inequality, we get sufficient condition $y>2\sqrt{\frac{-\log(1-p_A)}{N}}$. Plugging $1-\frac{\sigma_1^2}{\sigma_0^2}$ into $y$ we get: $$1-\frac{\sigma_1^2}{\sigma_0^2}>2\sqrt{\frac{-\log(1-p_A)}{N}},$$ which is very easily manipulated to the inequality from theorem statement. 
\end{proof}

\textbf{Theorem~\ref{correctness of certification procedure}:}
Let $x_0, x_1, p_A, \sigma_0$ be as usual and let $\norm{x_0-x_1}=R$. Then, the following two statements hold:  
\begin{enumerate}
    \item Let $\sigma_1 \le \sigma_0$. Then, for all $\sigma_2: \sigma_1 \le \sigma_2 \le \sigma_0$, if $\xi_>(R, \sigma_2)>0.5$, then $\xi_>(R, \sigma_1)>0.5$.
    \item Let $\sigma_1 \ge \sigma_0$. Then, for all $\sigma_2: \sigma_1 \ge \sigma_2 \ge \sigma_0$, if $\xi_<(R, \sigma_2)>0.5$, then $\xi_>(R, \sigma_1)>0.5$.
\end{enumerate}
\begin{proof}
We will first prove the first statement. Denote, as usual in the proofs $B_i$ the worst-case ball for $\sigma_i$, $\mathbb{P}_i$ the probability associated to $\mathcal{N}(x_1, \sigma_i^2 I)$. Since $\xi_>(R, \sigma_2)>0.5$ and since it is essentially $\mathbb{P}_2(B_2)$, we see, that the probability of a ball under normal distribution is bigger than half. This is obviously possible just if $x_1 \in B_2$. From the fact, that $B_1$ is the worst-case ball for $\sigma_1$ we see $\xi_>(R, \sigma_1)=\mathbb{P}_1(B_1) \ge \mathbb{P}_1(B_2)$. It suffices to show $\mathbb{P}_1(B_2) \ge \mathbb{P}_2(B_2)$. 

This follows, since $\sigma_1 \le \sigma_2$ and $x_1 \in B_2$. We know, that we can rescale the space such that $$\mathbb{P}_1(B_2) = \mathbb{P}_2\left(\frac{\sigma_2}{\sigma_1}(B_2-x_1)+x_1\right),$$ using the fact that $\sigma$ just scales the normal distribution. The set $\frac{\sigma_2}{\sigma_1}(B_2-x_1)+x_1$ is just an image of $B_2$ via homothety with center $x_1$ and rate $\frac{\sigma_2}{\sigma_1}$. So it suffices to prove $$\mathbb{P}_2\left(\frac{\sigma_2}{\sigma_1}(B_2-x_1)+x_1\right) \ge \mathbb{P}_2(B_2).$$
However, obviously $\frac{\sigma_2}{\sigma_1}(B_2-x_1)+x_1 \supset B_2$ from convexity of a ball. If, namely, $x_1+z \in B_2$, then from convexity also $x_1+\frac{\sigma_1}{\sigma_2} z \in B_2$ and this maps back to $x_1+z$, thus $x_1+z$ is in an image. Applying monotonicity of $\mathbb{P}$, we obtain the result.

Now we will prove the second statement and as usual, let $A_1, A_2, \mathbb{P}_1, \mathbb{P}_2$ be as usual ($A$ is now the ball connected to class $A$). As always, $\mathbb{P}_1(A_1) \le \mathbb{P}_1(A_2)$, so it suffices to show $\mathbb{P}_2(A_2)<0.5 \implies \mathbb{P}_1(A_2)<0.5$. Now, we need to distinguish two cases. 
If $x_1 \in A_2$, the proof is completely analogical to the first part, but now reasoning on $A$'s rather than $B$'s. In this case, we will even get stronger $\mathbb{P}_1(A_2) \le \mathbb{P}_2(A_2)$ just like in the first part. If $x_1 \not\in A_2$, then it is easy to see, that indeed $\mathbb{P}_2(A_2)<0.5$, yet it is also obvious to see that $\mathbb{P}_1(A_2)<0.5$. This finishes the proof of the lemma. 
\end{proof}

\textbf{Theorem~\ref{thm the concrete method}:}
Let $\sigma(x)$ be $r$-semi-elastic function and $x_0, p_A, N, \sigma_0$ as usual. Then, the certified radius at $x_0$ guaranteed by our method is $$CR(x_0)= \max \left\{0, \sup \left\{R \ge 0; \hspace{1mm} \xi_>(R, \sigma_0 \exp(-rR))<0.5 \hspace{2mm} \text{and} \hspace{2mm} \xi_<(R, \sigma_0 \exp(rR))<0.5\right\}\right\}.$$
\begin{proof}
This follows easily from Theorem~\ref{correctness of certification procedure}
\end{proof}

\begin{lemma} \label{lipschitz continuity of many lc functions}
Let us have $ f(x)= \sum\limits_{i=1}^M f_i(x) \mathbf{1}(x \in R_i)$, where $\{R_i\}_{i=1}^M$ is finite set of regions that divide the $\mathbb{R}^N$ and $\{f_i\}_{i=1}^M, f_i: R_i \xrightarrow[]{} \mathbb{R}$ is finite set of 1-Lipschitz continuous (1-LC) functions. Moreover assume, that $f(x)$ is continuous. Then, $f(x)$ is 1-LC.
\end{lemma}
\begin{proof}
We do not assume any nice behaviour from our decision regions, what can make the situation quite ugly. For instance, regions might not be measurable. However, it will not be a problem for us. 

Fix $x_1, x_2$. Consider line segment $S = x_1+\alpha(x_2-x_1), \alpha \in [0,1]$. Let us instead of points in $S$ work with numbers in $[0,1]$ via the $\alpha$ encoding. Consider the following coloring $C$ of $[0, 1]$: Each point $a \in [0, 1]$ will be assigned one of $M$ colors according to which region the $x_1+a(x_2-x_1)$ belongs. Define $d_1 \equiv 0$ and $d_2 = \sup \{z \in [0, 1], C(z)=C(d_1)\}.$ Thus, $d_2$ is the supremum of all numbers colored the same color as $0$. Then, $$|f(x_1+d_2(x_2-x_1))-f(x_1+d_1(x_2-x_1))| \le (d_2-d_1)\norm{x_2-x_1}.$$

Why? Let $\{z_j\}_{j=1}^\infty$ be a non-decreasing sequence s.t. $C(z_j)=C(0)$ and $z_j \xrightarrow[]{} d_2$. Since $f$ is continuous, obviously $f(x_1+d_2(x_2-x_1))= \underset{j \xrightarrow[]{} \infty}{\lim} f(x_1+z_j(x_2-x_1))$. Now, since norm and absolute value are both continuous functions and since from 1-LC of $f_{C(0)}$ on $R_{C(0)}$ we have $$\forall j \in \mathbb{N}: \frac{|f(x_1+z_j(x_2-x_1))-f(x_1+d_1(x_2-x_1))|}{(z_j-d_1)\norm{x_2-x_1}}\le 1,$$ we also necessarily have $$\frac{|f(x_1+d_2(x_2-x_1))-f(x_1+d_1(x_2-x_1))|}{(d_2-d_1)\norm{x_2-x_1}}\le 1.$$

If $d_2=1$, we finish the construction. If not, distinguish two cases. First assume $C(d_2)=C(d_1)$. In this case, take some color $C$ s.t. $$\exists \{z_j\}_{j=1}^\infty: z_{j+1}\le z_{j} \hspace{1mm} \forall j \in \mathbb{N} \hspace{2mm} \text{and} \hspace{2mm} z_j \xrightarrow[]{} d_2 \hspace{2mm} \text{and} \hspace{2mm} z_j>d_2 \hspace{1mm}\forall j \in \mathbb{N},$$ and fix one such $\{z_j\}_{j=1}^\infty$. Obviously, $C \neq C(0),$ since $d_2$ is upper-bound on points of color $C(0)$. Then, define $d_3 = \sup \{z \in [0, 1], C(z)=C\}$ and also define $\{\overline{z_j}\}_{j=1}^\infty: \overline{z_{j+1}} \ge \overline{z_j} \hspace{1mm} \forall j \in \mathbb{N} \hspace{2mm} \text{and} \hspace{2mm} \overline{z_j} \xrightarrow[]{} d_3$.  From continuity of $f$, we again have $f(x_1+d_2(x_2-x_1))= \underset{j \xrightarrow[]{} \infty}{\lim} f(x_1+z_j(x_2-x_1))$ and similarly $f(x_1+d_3(x_2-x_1))= \underset{j \xrightarrow[]{} \infty}{\lim} f(x_1+\overline{z_j}(x_2-x_1))$. Again from continuity of absolute value and norm and 1-LC of all the partial functions we have: 
$$\forall j \in \mathbb{N}: \frac{|f(x_1+\overline{z_j}(x_2-x_1))-f(x_1+z_j(x_2-x_1))|}{(\overline{z_j}-z_j)\norm{x_2-x_1}}\le 1$$ and 
$$\frac{|f(x_1+d_3(x_2-x_1))-f(x_1+d_2(x_2-x_1))|}{(d_3-d_2)\norm{x_2-x_1}}\le 1.$$

Now assume $C(d_2)\neq C(d_1)$. Then, we can take as $C$ directly $C(d_2)$ and do the same as in the last paragraph (note, that this case could have implicitly come up in the previous construction too, but we would need to not take $C=C(0)$ and we find this case distinction to be more elegant). 

If $d_3=1$, we finish the construction. If not, we continue in exactly the same manner as before. Since the number of colors $M$ is finite, we will run out of colors in finite number of steps and thus, eventually there will be $l \le M$ s.t. $d_l=1$. The final 1-LC is now trivially obtained as follows: 
\begin{gather*}
|f(x_1+1(x_2-x_1))-f(x_1+0(x_2-x_1))|=\bigg|\sum\limits_{i=1}^{l-1} f(x_1+d_{i+1}(x_2-x_1))-f(x_1+d_i(x_2-x_1))\bigg| \\ \le \sum\limits_{i=1}^{l-1} \Big|f(x_1+d_{i+1}(x_2-x_1))-f(x_1+d_i(x_2-x_1))\Big| \le \sum\limits_{i=1}^{l-1} \big|(d_{i+1}-d_i)\norm{x_2-x_1}\big| \\ = \norm{x_2-x_1} \sum\limits_{i=1}^{l-1} (d_{i+1}-d_i) = \norm{x_2-x_1}
\end{gather*}
\end{proof}

\textbf{Theorem~\ref{r semi elasticity of sigma(x)}:}
The $\sigma(x)$ defined in Equation~\ref{eq: the sigma fcn} is $r$-semi-elastic. 
\begin{proof}
Our aim is to prove, that $$\log(\sigma(x))= \log(\sigma_b)+r \left(\frac{1}{k} \left( \sum\limits_{x_i \in \mathcal{N}_k(x)} \norm{x-x_i}\right) -m\right)$$ 
is $r$ lipschitz continuous. Obviously, this does not depend neither on $\log(\sigma_b)$, nor on $-rm$, so we will focus just on $\frac{r}{k}\sum\limits_{x_i \in \mathcal{N}_k(x)} \norm{x-x_i}$. Obviously, this function is $r$ lipschitz continuous if and only if $\frac{1}{k}\sum\limits_{x_i \in \mathcal{N}_k(x)} \norm{x-x_i}$ is 1-LC. 

Let us fix $y \in \mathbb{R}^N$. We will first prove $\norm{x-y}$ is 1-LC. Let us fix $x_1, x_2$. From triangle inequality we have $$\big|\norm{x_1-y}-\norm{x_2-y}\big| \le \norm{x_1-x_2},$$
what is exactly what we wanted to prove. 

Now fix $y_1, y_2, \dots, y_k$ and $x_1, x_2$. Then 
\begin{gather*}
\bigg|\frac{1}{k}\sum\limits_{i=1}^k \norm{x_1-y_i}-\frac{1}{k}\sum\limits_{i=1}^k \norm{x_2-y_i}\bigg|=\frac{1}{k}\bigg|\sum\limits_{i=1}^k \norm{x_1-y_i} - \norm{x_2-y_i}\bigg| \\ \le \frac{1}{k}\sum\limits_{i=1}^k \Big|\norm{x_1-y_i} - \norm{x_2-y_i}\Big| \le \frac{1}{k}\sum\limits_{i=1}^k \norm{x_1-x_2}=1
\end{gather*}

Finally note, that using the $k$ nearest neighbors out of finite training dataset will divide $\mathbb{R}^N$ in a finite number of regions, where each region is defined by the set of $k$ nearest neighbors for $x$ in that region. Note, that the average distance from $k$ nearest neighbors is obviously continuous. Then, using Lemma~\ref{lipschitz continuity of many lc functions}, the claim follows.
\end{proof}

%\fi
%%%%%%%%%%%%%%%%%%%%%%%%%%%%%%%%%%%%%%%%%%%%%%%%%%%%%%%%%%%%%%%%%%%%%%%%%%%%%%%
%%%%%%%%%%%%%%%%%%%%%%%%%%%%%%%%%%%%%%%%%%%%%%%%%%%%%%%%%%%%%%%%%%%%%%%%%%%%%%%

\end{document}